\pdfoutput=1
\documentclass{article}

\usepackage{mathpazo}
\usepackage[round,authoryear]{natbib}
\usepackage[a4paper,left=0.83in,right=0.83in,top=0.94in,bottom=0.84in]{geometry}

\usepackage{amsmath,amsfonts,bm}

\def\eqref#1{equation~\ref{#1}}

\def\1{\bm{1}}

\DeclareMathAlphabet{\mathsfit}{\encodingdefault}{\sfdefault}{m}{sl}
\SetMathAlphabet{\mathsfit}{bold}{\encodingdefault}{\sfdefault}{bx}{n}

\usepackage{amsmath}
\usepackage{amsfonts}
\usepackage{amssymb}
\usepackage{amsthm}
\usepackage{mathtools}
\usepackage{bm}
\usepackage{nicefrac}
\usepackage{tcolorbox}

\usepackage{booktabs}
\usepackage{amsmath}
\usepackage{amsfonts}
\usepackage{amssymb}
\usepackage{amsthm}
\usepackage{nicefrac}
\usepackage{microtype}
\usepackage{graphicx}
\usepackage[table]{xcolor}
\usepackage{multirow}
\usepackage{tabularx, array}
\usepackage{arydshln}
\usepackage{pgfplots}
\usepgfplotslibrary{groupplots}
\pgfplotsset{compat=1.18}
\usepackage{url}
\usepackage{xspace}

\usepackage{hyphenat}
\usepackage{wrapfig}

\usepackage{kdmlab-manuscript}

\renewcommand{\KDMlabmark}{%
  \includegraphics[height=0.75in]{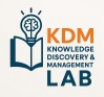}%
}
\renewcommand{\LCSmark}{%
  \includegraphics[width=1.39in]{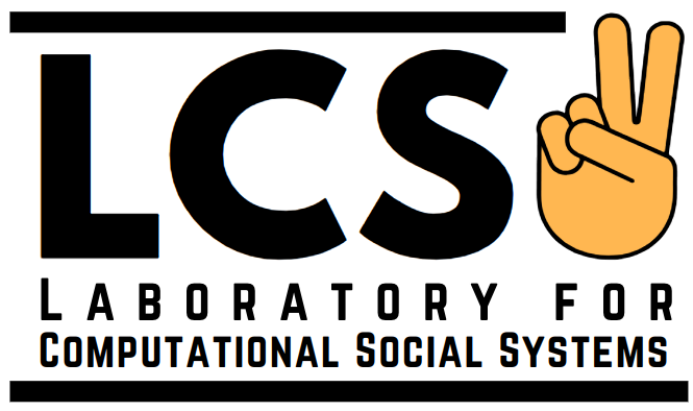}%
}
\renewcommand{\KDMidentitycard}[3]{%
  \begin{minipage}[c][1.18in][c]{\linewidth}
    \centering
    #2\par\vspace{0.10in}%
    {\sffamily\footnotesize\color{KDMGray}#1\par}%
  \end{minipage}%
}

\usepackage{hyperref}
\hypersetup{%
  colorlinks=true,linkcolor=KDMTeal,citecolor=KDMTeal,urlcolor=KDMTeal,
  pdftitle={Not All Is Lost: Repairing Lossy User Preference States of Personalization Encoders},
  pdfauthor={Parthiv Chatterjee, Dhiraj Golhar, Ummesalma Diwan, Sourish Dasgupta, Manjunath Joshi, Tanmoy Chakraborty}%
}

\DeclareRobustCommand{\repair}{\texttt{REPAIR}\xspace}
\theoremstyle{definition}
\newtheorem{definition}{Definition}[section]

\theoremstyle{plain}
\newtheorem{corollary}{Corollary}[section]

\newtheorem{theorem}{Theorem}

\newtheorem{proposition}{Proposition}

\title{Not All Is Lost: Repairing Lossy User Preference States of Personalization Encoders}

\author{%
  Parthiv Chatterjee\textsuperscript{1}\quad
  Dhiraj Golhar\textsuperscript{2}\quad
  Ummesalma Diwan\textsuperscript{1}\\[0.18em]
  Sourish Dasgupta\textsuperscript{1}\quad
  Manjunath Joshi\textsuperscript{2}\quad
  Tanmoy Chakraborty\textsuperscript{3}%
}

\begin{document}

\addtocontents{toc}{\protect\setcounter{tocdepth}{-1}}

\maketitle

\begin{tcolorbox}[
  colback=KDMWash, colframe=KDMRule, boxrule=0.35pt,
  arc=0pt, left=8pt, right=8pt, top=5pt, bottom=5pt
]
\footnotesize\textbf{Version note.} This author-prepared archival manuscript presents
work accepted to NeurIPS 2026. The discussion and interpretation are expanded
in this version; the research content, tables, and reported results are
preserved from the original accepted manuscript source.
\end{tcolorbox}
\vspace{0.05in}

\begin{abstract}
Personalization encoders compress evolving interaction histories into preference states used to rank items or condition text generation. A task head operating only on this state can miss useful evidence that remains in the frozen encoder's cached representations for individual timesteps. We study this \emph{recoverability gap} and propose \textbf{\repair{}}, which compares cached representations with the current preference state in a compact learned coordinate space. It resolves corrective evidence over extended history, recent interactions, and localized bursts. It then selects which patterns at which timesteps contribute and adds their aggregate correction to the state before the task head. Encoder-host repair reuses representations from the existing forward computation without re-encoding the history. Across MovieLens, PENS, MIND, and Amazon Reviews 2023, training only \repair{} improves MRR and nDCG@10 for all twelve representative recommendation hosts while both encoder and task head remain frozen. Head-only finetuning of the same hosts yields smaller gains. For example, Mamba4Rec on MovieLens gains $3.96$ MRR points, compared with $0.19$ from head-only finetuning. Rank and temporal diagnostics support a compact, host-dependent corrective structure. In personalized generation, IMPerSumm improves the two reported weighted PerSEval variants, which assess responsiveness to user preference, by up to $25.23\%$. These results support post-compression state correction and distinguish the availability of preference evidence from its downstream use.
\end{abstract}

\section{Introduction}\label{sec:introduction}

Personalization systems encode user interaction histories through recurrent, self-attentive, time-aware, state-space, or multi-interest mechanisms and expose a task-facing preference state to a downstream task head \citep{gru4rec-2016,nrms,mamba4rec-arxiv-2024,tisasrec-wsdm-2020,comirec-kdd-2020}. For example, a sequential recommender contextualizes item-level representations, compresses them into a user representation, and passes that state to a ranking head. A head operating only on that state has no direct access to evidence retained elsewhere in the encoder computation.

In this paper, we argue that this need not imply irreversible loss of every task-relevant trace. We therefore investigate whether cached per-timestep representations from the same frozen encoder computation can still support correction of the exposed state. We call this under-expressed but still accessible evidence the \textit{recoverability gap}. The practical goal, therefore, is to \textit{improve deployed models through state correction while keeping their encoders frozen}. Existing work focuses on encoder state quality by modifying internal representations \citep{reft-neurips-2024,stylevector-acl-2025} and editing frozen behavior \citep{serac-icml-2022}, or suppressing noisy and target-irrelevant interactions \citep{recdenoiser-recsys-2022,dicycle-sigir-2022,SSDRec-ICDE-2024}. These approaches leave open whether cached evidence from the same frozen computation supports selective correction of an already formed preference state before an unchanged task head.

To address this, we propose \textbf{\repair{}} (Recovering Evolving Preference State via Adaptive Interaction-aware Representation Correction), which compares the compressed encoder state exposed to the task head with the cached timestep representations, learns corrective evidence across time, selectively retains the most relevant parts, and adds them back to the state before it is used by the task head. \repair{} maps the mismatch to $K$ learned corrective coordinates and resolves their temporal variation through three views. Long-support, short-support, and episodic-support views emphasize extended history, recent evidence, and localized bursts. These names describe temporal basis functions rather than observed categories of user intent. The encoder itself is left unchanged and does not process the raw history again. The original state remains the base to which the correction is added.

We evaluate \repair{} on four diverse personalization datasets: MovieLens (15K movie-recommendation instances), MIND and PENS (20K news instances each), and Amazon Reviews 2023 (2K multimodal product instances), with PENS additionally supporting 20K personalized headline-generation examples. We first test whether the \textit{recoverability gap} supports useful state correction by freezing both the trained encoder and its original task head and training only \repair{} between them. We compare this condition with head-only finetuning on the unchanged encoder state to distinguish the two sources of improvement. We find that \repair{} improves all twelve representative recommendation hosts on MRR and nDCG@10. For example, \textsc{Mamba4Rec} on MovieLens, \textsc{EBNR} on PENS, and \textsc{OpenCLIP} on Amazon gain $+3.96$, $+1.42$, and $+2.84$ MRR, compared with $+0.19$, $+0.28$, and $+0.62$ from head-only finetuning.

We then investigate whether correction requires a high-dimensional representation or admits a \textit{compact, temporally structured} signal. We vary $K$, the number of learned corrective coordinates, and find that $K=128$ gives the best MRR and nDCG@10 in all five tested settings. Increasing $K$ to $192$ lowers both metrics. Temporal structure also affects the correction. Combining long-, short-, and episodic-support views is strongest in 11 of 12 controlled settings, while removing temporal resolution can even make \repair{} worse than the original frozen host. The type of temporal evidence that \repair{} finds most useful for correction does not always match the temporal pattern originally present in the user history. This suggests that different encoders can compress the same history differently, leaving different kinds of temporal evidence available for later correction.

We finally ask how the recovered preference evidence should be used downstream. For recommendation, we test whether correcting the preference state before the task head, as \repair{} does, is more effective than correcting the output after the head. Under the same frozen-host protocol, \repair{} outperforms two evaluated post-head alternatives in all eight recommendation settings. For generation, personalization measured by PerSEval can improve substantially even when standard generation-quality scores such as ROUGE and METEOR change little. Thus, generic generation quality alone does not capture the personalization response.

Our contributions are threefold. First, we identify a \textit{recoverability gap} and find compact, host-dependent temporal structure in the corrective representations. Second, we propose \textbf{\repair{}}, which recovers this evidence and adds a selective correction to the state without changing or rerunning the encoder. Third, we establish gains with both encoder and task head frozen in recommendation, and distinguish personalization gains from generic generation quality in personalized generation.

\section{Related Work}\label{sec:related-works}

\textbf{Preference Encoders And State Compression.} Personalization models encode user histories through recurrent, attentive, bidirectional, time-aware, state-space, and large-scale sequential architectures \citep{gru4rec-2016,okura2017embedding,naml,nrms,sasrec-2018,bert4rec-cikm-2019,tisasrec-wsdm-2020,mamba4rec-arxiv-2024,hstu-arxiv-2024}. These methods improve how historical evidence is contextualized and compressed into a task-facing state. \repair{} studies whether cached representations can improve that state after the trained encoder has been frozen.

\textbf{Temporal And Multi-Interest Preference Modeling.} Prior work models long- and short-term preferences \citep{lstur-acl-2019,yu-ijcai-2019}, multiple latent interests \citep{comirec-kdd-2020,miner-2022-aclfindings}, and frequency-sensitive or state-space dynamics \citep{fearec-sigir-2023,mamba4rec-arxiv-2024,ss4rec_ldh-arxic-2025}. Multi-timescale modeling is therefore established. In \repair{}, temporal bases resolve corrective evidence after compression. They do not construct the host's history state. Our questions concern the compactness of this corrective representation and the temporal evidence that remains useful across frozen hosts.

\textbf{Personalized Generation And Downstream Utilization.} Personalized generation conditions decoding on evidence from history, profiles, retrieval, or prompts \citep{pens-acl-2021,gtp-tacl-2023,scape-www-2025,FPG-icdm-2023,personachat-acl-2018,recap-acl-2023,miracle-acl-2023,mpchat-2023-acl,personamem-naacl-2024,icopernicus-emnlp-2024}. LLM recommendation also uses enriched prompts and agentic tools~\citep{llm4rec-acl-2024,recmind-naacl-2024}. The downstream interface affects how this evidence reaches the output. We examine how downstream interfaces use a repaired preference state, separating user-specific personalization gains from changes in generic text quality.

\textbf{Representation Intervention, Model Editing, And Denoising.} Adapters, Prefix-Tuning, LoRA, ReFT, and StyleVector modify parameters, conditioning vectors, or hidden representations \citep{houlsby2019parameter,li2021prefix,hu2022lora,reft-neurips-2024,stylevector-acl-2025}. MEND, SERAC, ROME, and MEMIT alter behavior through editors, memory, or localized parameter updates \citep{mitchell2022fast,serac-icml-2022,meng2022locating,meng2023memit}. Sequential denoising methods suppress noisy or target-irrelevant interactions~\citep{recdenoiser-recsys-2022,hsd-cikm-2022,dicycle-sigir-2022,SSDRec-ICDE-2024}. DiffuRec and DreamRec instead generate recommendation representations through diffusion~\citep{diffurec-tois-2023,dreamrec-neurips-2023}. \repair{} specifies a correction target and evidence source within this broader intervention setting: it uses cached per-timestep representations to correct an already formed preference state before downstream use, with the host encoder fixed and without re-encoding raw history.

\section{Preliminaries: Preference State Compression \& Recoverability}\label{sec:problem-setup}

We first represent different kinds of \textit{evolving} user histories as a time-ordered sequence of user interactions. Following the User Interaction Graph (UIG) abstraction of \citet{peraugy-tmlr-2025}, each interaction is $b^{(t_i)}=\langle a^{(t_i)},v^{(t_i)}\rangle$, where $a^{(t_i)}\in\{\texttt{pos},\texttt{neg},\texttt{cmd}\}$ records what happened and $v^{(t_i)}$ records the associated content. For example, clicking an item is a positive (\texttt{pos}) interaction, skipping or rejecting one is negative (\texttt{neg}), and \texttt{cmd} specifies the personalization task command. The ordered history up to $t_i$ is $\tau_b^{t_{1:i}}$. Recommendation tasks use this history to predict or rank the next item, whereas generation tasks additionally receive the current task-side query or content $q^{(t_i)}$ and produce a personalized response\footnote{UIG is a common abstraction of organizing heterogeneous histories and is not a component of \repair{}.}. The personalization encoder produces cached per-timestep representations $\{e_b^{(t_j)}\}_{j=1}^{i}$, which the encoder compresses into the task-facing preference state $ z_b^{(t_i)}= f_{\mathrm{enc}}\!\left( e_b^{(t_1)},\ldots,e_b^{(t_i)} \right) $. \repair{} uses $\{e_b^{(t_j)}\}_{j=1}^{i}$ to correct $z_b^{(t_i)}$ before passing it downstream to the task head. Dataset-specific UIG construction (termed \textit{eventization}) is given in Appendix~\ref{app:uig-construction}.

\textbf{Recoverability gap.} Some useful preference evidence may be only weakly expressed in the task-facing compressed state $z_b^{(t_i)}$, even though it is still available in the cached $\{e_b^{(t_j)}\}_{j=1}^{i}$. We call this difference the \emph{recoverability gap}.

This does not mean that the information is completely missing from $z_b^{(t_i)}$. Instead, $\{e_b^{(t_j)}\}_{j=1}^{i}$ may still provide additional evidence that can be used to correct it while keeping the encoder frozen. The resulting performance gain also depends on how effectively the downstream head uses the corrected state.

\begin{figure*}[t]
    \centering
    \includegraphics[width=0.85\textwidth]{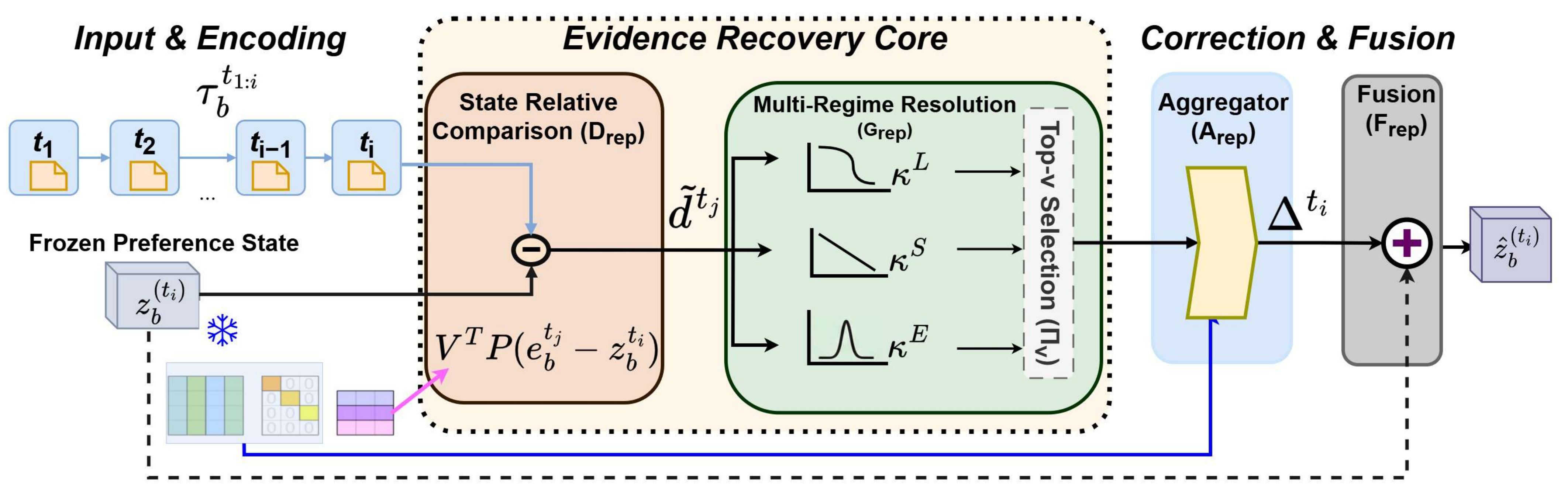}
    \caption{\textbf{Overview of \repair{}.} Given a frozen preference state and cached per-timestep representations from host computation, \repair{} constructs state-relative corrective evidence, resolves it through long-, short-, and episodic-support branches, selectively retains timestep-coordinate corrective components, and fuses the aggregated correction with the frozen state before the downstream interface consumes the repaired state.}
    \label{fig:repair_pipeline}
\end{figure*} 

\section{Methodology}\label{sec:methodology}

\repair{} is inserted between the encoder and downstream task head of a pre-trained personalization system, which we call its \emph{host}. For a user interaction prefix $\tau_b^{t_{1:i}}$ ending at the prediction time $t_i$, one frozen host-encoder forward pass already produces the compressed task-facing state $z_b^{(t_i)}$ and the per-timestep representations $\{e_b^{(t_j)}\}_{j=1}^{i}$ available before compression. \repair{} uses only this already-computed evidence to correct $z_b^{(t_i)}$, which remains the base state. Recovery neither reruns nor updates the host encoder. The cache is bounded by the host history window and costs $O(Ld_e)$ for maximum length $L$ (Figure~\ref{fig:repair_pipeline}, Appendix~\ref{app:method-details}).

\textbf{I. State-relative Corrective Evidence.} A past interaction should influence the correction only if its cached representation can still improve the compressed state. Since $e_b^{(t_j)}$ and $z_b^{(t_i)}$ may lie in different spaces, learned projections $P_e$ and $P_z$ map them to a common space, giving $\bar e_b^{(t_j)}$ and $\bar z_b^{(t_i)}$. Their difference $r^{(t_j;t_i)}=\bar e_b^{(t_j)}-\bar z_b^{(t_i)}$ is \repair{}'s state-relative corrective signal, not literal missing information. Rather than carrying this full-dimensional difference forward, \repair{} probes it using $K$ learned corrective patterns ($V_{\mathrm{rep}}^\top$) through an SVD-like rank-$K$ operator with read, scale, and write factors $J_{\mathrm{rep}}$:
\begin{equation}
\begin{split}
& J_{\mathrm{rep}} = \underbrace{U_{\mathrm{rep}}}_{\substack{\text{write}\\\text{(where patterns are written)}}} 
\cdot
\underbrace{\Lambda_{\mathrm{rep}}}_{\substack{\text{pattern scaling}\\\text{(how strongly each pattern should contribute \textit{in general})}}}
\cdot
\underbrace{V_{\mathrm{rep}}^\top}_{\substack{\text{read}\\\text{(how strongly patterns are present)}}}; \\
& \operatorname{rank}(J_{\mathrm{rep}})\le K; \quad d^{(t_j)} = V_{\mathrm{rep}}^\top r^{(t_j;t_i)}.
\end{split}
\label{eq:repair_drep_main}
\end{equation}
Each $d_k^{(t_j)}$ is the activation of corrective pattern $k$ at timestep $t_j$. $\Lambda_{\mathrm{rep}}$ stores the shared transmission strength of each pattern, and $U_{\mathrm{rep}}$ its write direction. The factorization is learned rather than computed by SVD, and its normalized columns need not be orthogonal. Because the projections and read markers are learned, $d^{(t_j)}$ is a learned corrective representation rather than a native host residual. We use $K=128$ by default and vary $K$ in RQ2. Smaller $K$ forces the correction into fewer patterns, directly testing compactness rather than assuming an intrinsic preference rank. The admissible-rank condition is stated in Appendix~\ref{app:k_choice}, and the conditional correction bound in Theorem~\ref{thm:repair-recoverability} (Appendix~\ref{app:recoverability}).

Before asking \emph{at what timestep} each pattern matters, \repair{} removes its prefix-wide average activation, $\tilde d^{(t_j)}=d^{(t_j)}-\frac{1}{i}\sum_{q=1}^{i}d^{(t_q)}$. A pattern activated equally throughout the prefix vanishes after centering, including the shared state offset. Subsequent gating, pooling, and fusion remain state-conditioned (Appendix~\ref{app:vector-implementation}).

\textbf{II. Temporal Resolution of Corrective Evidence.} The centered activations identify \emph{which} patterns vary at each timestep. The shared $\Lambda_{\mathrm{rep}}$ does not identify \emph{when} an activation is useful for correction. \repair{} therefore resolves timestep-wise activations over time before write-back. It uses \emph{long-support}, \emph{short-support}, and \emph{episodic-support} views for evidence extending farther back, recent evidence, and localized bursts, respectively. Because correction is made at the current endpoint $t_i$, an earlier timestep $t_j$ is measured by lag $\ell_j=i-j$. Larger lag means farther back in history. For pattern $k$,
\begin{equation}
\begin{aligned}
\kappa^{L}_{k}(\ell_j)&=\exp(-\beta_k\ell_j), &&\text{\emph{long: gradual decay}},\\
\kappa^{S}_{k}(\ell_j)&=\max\!\left(1-\frac{\ell_j}{i},0\right), &&\text{\emph{short: recent evidence}},\\
\kappa^{E}_{k}(\ell_j)&=\exp\!\left(\ell_j\log\xi_k-\xi_k-\log\Gamma(\ell_j+1)\right), &&\text{\emph{episodic: localized burst}}.
\end{aligned}
\label{eq:repair_kernels_main}
\end{equation}
These are three views of the same corrective activation, not labels that classify an interaction as long-, short-, or episodic. \repair{} learns their combined support,
\begin{equation}
\eta^{(t_j)}
=
\sum_{\rho\in\{L,S,E\}}
W_{\rho}\!\left(\kappa^{\rho}(\ell_j)\odot\tilde d^{(t_j)}\right),
\label{eq:repair_temporal_resolution_main}
\end{equation}
where $W_\rho$ learns how each view contributes. Because this reasoning occurs \emph{after} compression, different frozen encoders can leave different temporal evidence useful for correcting the same underlying history. For this temporal decomposition to be meaningful, the three views must be distinguishable. Theorem~\ref{thm:kernel-trace-recoverability} in Appendix~\ref{app:temporal_recoverability} formalizes when their individual contributions can be identified, while Corollary~\ref{cor:classical-timeseries-approximation} in the same appendix characterizes the finite-history patterns their combination can represent.

\textbf{III. Selective Write-Back and State Correction.} Some temporally supported activations remain weak, noisy, or redundant. \repair{} therefore retains selected timestep-pattern activations through $\Pi_{\nu}^{(t_i)}(\eta^{(t_j)})=\pi_j^{(t_i)}\odot\eta^{(t_j)}$, with $\pi_j^{(t_i)}\in\{0,1\}^{K}$. For each pattern, the adaptive selector ranks $|\eta_k^{(t_j)}|$ across the prefix and learns how much support to retain. A hard forward mask retains activations, while a straight-through surrogate trains the retention fraction (Appendix~\ref{app:selector}). The retained activations now enter the write side. Since $\Lambda_{\mathrm{rep}}$ and $U_{\mathrm{rep}}$ are shared across states, a state-conditioned gate modulates the retained pattern activations for the current compressed state $\gamma^{(t_i)}=\sigma(\mathrm{MLP}_{\gamma}(z_b^{(t_i)})), w^{(t_j)} = U_{\mathrm{rep}}\Lambda_{\mathrm{rep}} \left( \gamma^{(t_i)}\odot\Pi_{\nu}^{(t_i)}(\eta^{(t_j)}) \right).$ Thus, $w^{(t_j)}$ is the write candidate proposed by timestep $t_j$. The candidate set $\mathcal{W}^{(t_i)}=\{w^{(t_j)}\}_{j=1}^{i}$ is pooled with learned normalized weights $\alpha_j^{(t_i)}$ into one correction for the current state $\Delta_b^{(t_i)} = P_{\Delta} \left( \sum_{j=1}^{i} \alpha_j^{(t_i)}w^{(t_j)} \right).$ Here, $\alpha_j^{(t_i)}$ only controls the relative contribution of timestep $t_j$ during pooling. Its scoring and softmax normalization are given in Appendix~\ref{app:method-details}. Finally, $z_b^{(t_i)}$ remains explicitly as the base and the correction is added through residual fusion:
\begin{equation}
\tilde z_b^{(t_i)}
=
z_b^{(t_i)}
+
g_F^{(t_i)}
\odot
\mathrm{MLP}_{\mathrm{probe}}
\left(
[z_b^{(t_i)};\Delta_b^{(t_i)}]
\right),
\label{eq:repair_update_main}
\end{equation}
where $g_F^{(t_i)}=\mathrm{sigmoid}(\mathrm{MLP}_{F}([z_b^{(t_i)};\Delta_b^{(t_i)}]))$ controls how much correction is added. The repaired state is then passed to the original task head. Full aggregation and fusion details are in Appendix~\ref{app:aggregation-fusion}.

For prompt-conditioned LLM hosts, \repair{} follows the same core idea through a different evidence interface. Unlike encoder hosts, where one representation per interaction is already available in the same forward pass, an LLM processes the serialized history as tokens and therefore does not directly expose an equivalent interaction-level sequence. We thus encode each interaction separately as corrective evidence, while the frozen LLM provides the task-facing terminal state $z_b^{(t_i)}$. The repaired state $\tilde{z}_b^{(t_i)}$ is then mapped to a continuous prefix for the frozen decoder. Full details are given in Appendix~\ref{app:llm-hosts}.

\textbf{IV. Training Objective.} The correction should improve the task while preserving its timestep provenance and source interaction. We train
\begin{equation}
\mathcal{L}_{\mathrm{repair}}
=
\omega_{\mathrm{task}}\mathcal{L}_{\mathrm{task}}
+
\omega_{\mathrm{pos}}\mathcal{L}_{\mathrm{pos}}
+
\omega_{\mathrm{coh}}\mathcal{L}_{\mathrm{coh}},
\label{eq:repair_objective_main}
\end{equation}
where $\mathcal{L}_{\mathrm{task}}$ uses the downstream task loss (cross-entropy for recommendation or teacher-forced generation loss), $\mathcal{L}_{\mathrm{pos}}$ uses relative-position classification, and $\mathcal{L}_{\mathrm{coh}}$ uses classification and contrastive supervision to tie corrective activations to their source evidence. We fix $(\omega_{\mathrm{task}},\omega_{\mathrm{pos}},\omega_{\mathrm{coh}})=(0.5,0.2,0.3)$ across all hosts, avoiding host-specific coefficient search and testing one reusable correction objective. Exact losses, coefficients, and task-specific instantiations are in Appendix~\ref{app:repair-objective-implementation}.

\begin{table}[t]
\centering
\caption{Evaluation design. Recommendation ranks one target against $499$ negatives. Host architectures and citations are in Appendix~\ref{app:task-baselines}, with protocols in Appendix~\ref{app:evaluation-details}.}
\label{tab:evaluation_coverage}
\scalebox{0.83}{
\begin{tabular}{llll}
\toprule
\textbf{Task} & \textbf{Dataset} & \textbf{Host Coverage} & \textbf{Test Instances} \\
\midrule
Movie recommendation & MovieLens & 9 sequential hosts & 15K \\
News recommendation & MIND & 7 history encoders & 20K \\
News recommendation & PENS & History encoders and LLMs & 20K \\
Product recommendation & Amazon Reviews 2023 & 3 multimodal hosts & 2K \\
Headline generation & PENS & 11 encoder interfaces and 2 LLMs & 20K \\
\bottomrule
\end{tabular}}
\end{table}

\section{Evaluation}\label{sec:evaluation}

We evaluate whether cached evidence enables useful correction with the encoder and task head frozen (RQ1), whether the learned correction is compact and temporally structured across hosts (RQ2), and how the repaired state is used downstream in recommendation and generation (RQ3).

\textbf{I. Experimental Setup.}\label{subsec:experimental-setup} We evaluate recommendation and personalized generation on MovieLens~\citep{movielens-dataset}, MIND~\citep{dataset-mind-acl-2020}, PENS~\citep{pens-acl-2021}, and Amazon Reviews 2023~\citep{hou2024amazonreviews2023}. TG-ReDial~\citep{tgredial-acl-2020} is used separately as a discourse-derived stress test. Following \citet{peraugy-tmlr-2025}, we organize histories as UIG trajectories with user-disjoint train/test splits. Although test users are unseen during training, the host receives their observed history prefix $\tau_b^{t_{1:i}}$ at each prediction time. Each prefix is paired with its next supervised target ${r_u^\star}^{(t_{i+1})}$, while generation additionally receives the task-side query at $t_i$.

Recommendation ranks one ground-truth target against 499 sampled negatives and reports MRR, nDCG@5/10, and Hit@10. Paired conditions use the same held-out users, prefixes, targets, and candidate sets. Generation evaluates the same outputs using ROUGE~\citep{rouge}, METEOR~\citep{meteor}, BLEU~\citep{bleu}, BERTScore~\citep{bertscore}, weighted and paraphrase-aware ROUGE variants, and personalization-sensitive PerSEval~\citep{perseval-tmlr-2024}. Recommendation metrics use a $0$ to $100$ scale, while generation tables state their display scale and relative gains are percentages. Dataset construction, candidate sampling, and metric details are in Appendices~\ref{app:evaluation-details} and~\ref{app:metrics}.

Host selection varies history compression and downstream use (Table~\ref{tab:evaluation_coverage}). GRU4Rec, SASRec, and Mamba4Rec cover recurrent, self-attentive, and selective state-space sequence models, respectively~\citep{gru4rec-2016,sasrec-2018,mamba4rec-arxiv-2024}. NRMS uses self-attention for news histories, while LSTUR combines long- and short-term user representations~\citep{nrms,lstur-acl-2019}. SigLIP2 and OpenCLIP supply image and text features for multimodal product histories~\citep{tschannen2025siglip2,cherti2023openclip}. Walk2Pers and IMPerSumm test interaction-conditioned personalized generation~\citep{walk2pers-iclr-2026,impersumm-www-2026}. DeepSeek-R1-Distill-Qwen-32B and Qwen2.5-32B~\citep{deepseek-arxiv-2025,qwen2.5-2025} test the separate frozen-LLM interface, where repaired states enter through continuous prefixes (Appendix~\ref{app:llm-hosts}). These choices test correction across compression mechanisms and whether downstream interfaces use recovered preference evidence. The complete host inventory and task interfaces are in Appendix~\ref{app:task-baselines}.

\textbf{II. Frozen-Host Protocol.}\label{subsec:frozen-hosts} For each dataset and host pair, all RQ1 conditions start from the same task-trained host checkpoint. \emph{Original Host} uses it unchanged. \emph{Head-only FT} keeps the encoder and task-facing state fixed and trains only the original task head for five epochs, retaining the best validation checkpoint. \emph{\repair{} (Encoder+Head Frozen)} freezes both original host components and trains only \repair{} from cached per-timestep representations. This is the frozen-head test of post-compression recoverability. Starting from the trained repair checkpoint, \emph{\repair{} + Head Alignment} jointly updates \repair{} and the task head for two further epochs while keeping the encoder frozen (Appendix~\ref{app:stage3-alignment}). All trainable stages use validation-selected checkpoints. Optimization schedules are in Appendix~\ref{app:training-details}. The repair architecture and fixed coefficients are in Appendix Table~\ref{tab:repair_hyperparameters}.

\textbf{III. Corrective-Structure Diagnostics.}\label{subsec:mechanism-diagnostics} For RQ2, we vary $K\in\{32,64,128,192\}$, the number of learned corrective patterns, to test correction compactness. We also construct evaluation-only histories with long-range (L-Tr), recent (S-Tr), or episodic (E-Tr) evidence while keeping the supervised target fixed and without retraining. These controls test whether post-compression temporal evidence simply mirrors the temporal pattern placed in the history. We further ablate temporal resolution and selective retention. These diagnostics characterize the learned corrective representation and do not certify an intrinsic rank or semantic decomposition of the frozen host state.

\textbf{IV. Downstream-Utilization Diagnostics.}\label{subsec:utilization-diagnostics} For recommendation, we compare pre-head \repair{} with two adapted post-head alternatives: \emph{SERAC-Corr}, a memory-based correction inspired by SERAC~\citep{serac-icml-2022}, and \emph{Rec-Denoiser-Corr}, which adjusts output-side timestep influence following sequential denoising~\citep{recdenoiser-recsys-2022}. They use the same task-trained host checkpoint, data split, held-out instances, candidate sets, alignment budget, and frozen-host computation. Because their architectures also differ, this comparison evaluates the implemented correction settings rather than intervention locus alone. For personalized generation, we compare standard generation-quality and personalization-sensitive metrics on the same outputs to test whether recovery primarily improves generic text quality or user-specific personalization. Adaptation details are in Appendix~\ref{app:corrective-baseline-details}.

\textbf{V. Statistical Evaluation.}\label{subsec:statistical-scope} OpenCLIP and SigLIP2 on Amazon and EBNR on PENS have three-seed means and variability in Appendix Table~\ref{tab:rq1_rebuttal_stats}. Because a held-out user may contribute multiple prefixes, 95\% confidence intervals use 10,000 user-clustered bootstrap resamples. Extended settings without repeated-run uncertainty are reported as point estimates.

\section{Results And Analysis}\label{sec:results}

\begin{table*}[t]
\centering
\small
\caption{\textbf{RQ1: Post-compression recoverability under frozen-host controls.} Columns compare the original host, head-only continuation on the unchanged state, \repair{} with encoder and head frozen, and two-epoch joint \repair{}-head alignment with the host encoder frozen. Values are MRR / nDCG@10. Three-seed means, variability, and bootstrap intervals are available for OpenCLIP and SigLIP2 on Amazon and EBNR on PENS (Appendix Table~\ref{tab:rq1_rebuttal_stats}). The remaining rows are reported as point estimates. S/M/W denote strongest, median, and weakest original hosts.}
\label{tab:rq1_representative_hosts}
\scalebox{0.7}{
\begin{tabular}{llcccc}
\toprule
\textbf{Dataset} &
\textbf{Host} &
\multicolumn{4}{c}{\textbf{MRR / nDCG@10}} \\
\cmidrule(lr){3-6}
&
&
\textbf{Original Host} &
\textbf{Head-only FT} &
\textbf{\repair{} (Frozen Head)} &
\textbf{\repair{} + Head Alignment} \\
\midrule

\multirow{3}{*}{MovieLens}
& S: \texttt{TiSASRec}
& 35.51 / 41.14
& 35.57 / 41.19
& 36.38 / 42.07
& 36.78 / 42.91 \\

& M: \texttt{Mamba4Rec}
& 29.29 / 38.50
& 29.48 / 38.89
& 33.25 / 39.20
& 34.26 / 40.17 \\

& W: \texttt{SASRec}
& 26.43 / 33.82
& 27.05 / 34.37
& 29.57 / 40.08
& 28.49 / 35.18 \\

\midrule

\multirow{3}{*}{PENS}
& S: \texttt{LSTUR}
& 10.44 / 12.15
& 10.98 / 12.32
& 11.98 / 12.71
& 12.85 / 15.88 \\

& M: \texttt{EBNR}
& 2.65 / 2.45
& 2.93 / 2.86
& 4.07 / 3.97
& 11.97 / 14.28 \\

& W: \texttt{NAML}
& 1.29 / 0.81
& 2.75 / 1.46
& 5.20 / 5.21
& 12.38 / 15.13 \\

\midrule

\multirow{3}{*}{MIND}
& S: \texttt{LSTUR}
& 51.43 / 59.63
& 52.37 / 59.96
& 54.17 / 62.27
& 55.84 / 66.92 \\

& M: \texttt{EBNR}
& 47.08 / 51.63
& 47.63 / 52.31
& 49.43 / 55.17
& 53.22 / 63.18 \\

& W: \texttt{SCAPE-GRU}
& 42.53 / 41.19
& 43.21 / 41.82
& 43.31 / 44.25
& 45.32 / 47.36 \\

\midrule

\multirow{3}{*}{Amazon}
& S: \texttt{SigLIP2}
& 2.40 / 2.18
& 2.67 / 2.72
& 4.16 / 4.05
& 17.72 / 20.35 \\

& M: \texttt{OpenCLIP}
& 1.59 / 1.05
& 2.21 / 2.17
& 4.43 / 4.66
& 19.49 / 22.08 \\

& W: \texttt{DINOv2-E5}
& 1.19 / 0.72
& 2.35 / 2.42
& 3.16 / 2.90
& 16.15 / 18.19 \\

\bottomrule
\end{tabular}
}
\end{table*}

\subsection{RQ1: Post-Compression Recoverability With A Frozen Host}\label{subsec:rq1}

\textbf{Can correction help after compression?} With both the task-trained encoder and original task head frozen, training only \repair{} improves MRR and nDCG@10 for all twelve representative hosts in Table~\ref{tab:rq1_representative_hosts}. Average MRR/nDCG@10 gains are $+2.66/+2.63$, $+2.29/+2.16$, $+1.96/+3.08$, and $+2.19/+2.55$ on MovieLens, PENS, MIND, and Amazon Reviews 2023. Thus, even after compression into $z_b^{(t_i)}$, \emph{cached per-timestep representations still provide evidence that can improve the state seen by the unchanged task head}.

\textbf{Can head-only finetuning explain the gains?} \emph{Head-only FT} keeps the encoder and task-facing state fixed and trains only the original task head. Its average MRR gains are $+0.29$, $+0.76$, $+0.72$, and $+0.68$ across the four datasets, compared with $+2.66$, $+2.29$, $+1.96$, and $+2.19$ from frozen-head \repair{}. The same ordering holds in all twelve settings. Hence, \emph{further finetuning of the same task head does not reproduce the gains from correcting the state itself}.

\textbf{Does head alignment add further benefit?} Keeping the encoder frozen, jointly aligning \repair{} and the task head further improves MRR and nDCG@10 in 11/12 settings. We therefore use head-aligned \repair{} (Full) for further studies. Full aligned results are in Appendix Tables~\ref{tab:rq1_predictive_movielens_versioned} to~\ref{tab:rq1_predictive_amazon_compact}.

\subsection{RQ2: Structure Of The Learned Corrective Representation}\label{subsec:rq2}

\textbf{Does correction require many corrective patterns?} One concern is that evidence under-expressed after compression may require a large corrective representation. Across $K\in\{32,64,128,192\}$, however, $K=128$ gives the best MRR and nDCG@10 in all five settings, while $K=192$ lowers both metrics in all five (Table~\ref{tab:rq2_rank_sweep}). The mean learned read-coordinate energy ratio at $K=128$ is $96.46\%$ across these settings (Appendix~\ref{app:k_choice}, Equation~\ref{eq:learned-read-energy}). This ratio is distinct from a leading-eigenspace energy fraction. The separate conditional spectral approximation bound is stated in Theorem~\ref{thm:k_admissible} (Appendix~\ref{app:k_choice}). Thus, $K=128$ is an empirically effective repair capacity.

\begin{table*}[t]
\centering
\small
\caption{\textbf{RQ2: Repair-space rank sensitivity after compression.} MRR / nDCG@10 for two-epoch downstream alignment across $K$. Bold marks the strongest repair rank for each host.}
\label{tab:rq2_rank_sweep}

\scalebox{0.78}{
\begin{tabular}{lccccc}
\toprule
\textbf{\boldmath Host / Dataset} &
\textbf{\boldmath Original Host} &
\multicolumn{4}{c}{\textbf{\boldmath \repair{} Rank $K$}} \\
\cmidrule(lr){3-6}
&
&
\textbf{\boldmath 32} &
\textbf{\boldmath 64} &
\textbf{\boldmath 128} &
\textbf{\boldmath 192} \\
\midrule

\texttt{OpenCLIP / Amazon}
& 1.59 / 1.05
& 9.31 / 9.78
& 11.17 / 12.22
& \textbf{19.49 / 22.08}
& 11.63 / 12.16 \\

\texttt{SigLIP2 / Amazon}
& 2.40 / 2.18
& 10.44 / 11.36
& 8.42 / 8.77
& \textbf{17.72 / 20.35}
& 13.46 / 13.15 \\

\texttt{EBNR / PENS}
& 2.65 / 2.45
& 6.13 / 6.19
& 8.65 / 8.58
& \textbf{11.97 / 14.28}
& 11.23 / 13.17 \\

\texttt{Mamba4Rec / MovieLens}
& 29.29 / 38.50
& 31.47 / 38.85
& 32.38 / 39.26
& \textbf{34.26 / 40.17}
& 34.02 / 39.96 \\

\texttt{LSTUR / MIND}
& 51.43 / 59.63
& 52.67 / 62.14
& 53.49 / 64.13
& \textbf{55.84 / 66.92}
& 54.73 / 66.12 \\

\bottomrule
\end{tabular}
}
\end{table*}

\begin{table*}[t]
\centering
\caption{\textbf{RQ2 temporal diagnostic.} If the temporal module merely copied the controlled input timing, the copy-through prediction would be $L$, $S$, and $E$ for L-Tr, S-Tr, and E-Tr, respectively. This pattern holds in only \textbf{6/12} settings. Full \repair{}, which jointly uses $L+S+E$, outperforms the best isolated view in \textbf{11/12} settings. Brackets report MRR.}
\label{tab:rq2_temporal_summary}
\scalebox{0.78}{
\begin{tabular}{llcccccc}
\toprule
& &
\multicolumn{3}{c}{\textbf{\boldmath Controlled one-view diagnostic}}
& \multicolumn{3}{c}{\textbf{\boldmath Full \repair{} ($L+S+E$ jointly)}} \\
\cmidrule(lr){3-5}
\cmidrule(lr){6-8}
\textbf{\boldmath Dataset} & \textbf{\boldmath Host}
& \textbf{\boldmath L-Tr} & \textbf{\boldmath S-Tr} & \textbf{\boldmath E-Tr}
& \textbf{\boldmath L-Tr} & \textbf{\boldmath S-Tr} & \textbf{\boldmath E-Tr} \\
\midrule
\multicolumn{2}{l}{\emph{Copy-through prediction}}
& $L$ & $S$ & $E$
& \multicolumn{3}{c}{\emph{all three views used jointly}} \\
\midrule
\multirow{2}{*}{MovieLens}
& \texttt{Mamba4Rec}
& $S$ [33.95] & $S$ [33.92] & $S$ [33.88]
& \textbf{[34.87]} & \textbf{[35.52]} & \textbf{[34.39]} \\
& \texttt{TiSASRec}
& $L$ [36.19] & $S$ [36.20] & $S$ [36.04]
& \textbf{[37.38]} & \textbf{[36.93]} & \textbf{[36.63]} \\
\midrule
\multirow{2}{*}{MIND}
& \texttt{NAML}
& $L$ [52.12] & $S$ [52.03] & $L$ [53.70]
& \textbf{[53.31]} & \textbf{[54.55]} & \textbf{[53.80]} \\
& \texttt{EBNR}
& $S$ [55.08] & $L$ [53.64] & $E$ [53.70]
& \textbf{[56.79]} & \textbf{[55.69]} & [53.19] \\
\midrule
\multicolumn{2}{r}{\textbf{Summary}}
& \multicolumn{3}{c}{\textbf{Copy-through match: 6/12}}
& \multicolumn{3}{c}{\textbf{Full \repair{} wins: 11/12}} \\
\bottomrule
\end{tabular}
}
\end{table*}

\textbf{Does temporal repair simply copy input timing?} To test whether corrective timing simply follows input timing, we construct evaluation-only histories with target-related evidence concentrated in long-range (L-Tr), recent (S-Tr), or localized episodic (E-Tr) portions, while holding the supervised target fixed and without retraining. If the temporal module simply copied input timing, the prediction would be $L/S/E$ under L-Tr/S-Tr/E-Tr. Instead, this mapping holds in only \textbf{6/12} probes, and \texttt{Mamba4Rec} yields $S/S/S$ across all three input patterns (Table~\ref{tab:rq2_temporal_summary}). Thus, \emph{the most useful corrective temporal evidence is not determined simply by where evidence appears in the input}.

\textbf{Do multiple temporal views help?} Full \repair{} jointly uses long-, short-, and episodic-support, and $L+S+E$ outperforms every isolated-view ablation in \textbf{11/12} settings. Removing temporal resolution can eliminate the gain: \texttt{Mamba4Rec} drops to $28.43$ MRR, below the untouched frozen host at $29.29$, whereas full temporal \repair{} reaches $34.26$. Hence, useful post-compression correction can be supported by multiple temporal scales simultaneously. Theorem~\ref{thm:kernel-trace-recoverability} in Appendix~\ref{app:temporal_recoverability} formalizes temporal identifiability. The full pathway ablation is in Appendix Table~\ref{tab:rq1_timescale_ablation}, with the gradient analysis in Appendix~\ref{app:kernel-backprop}.

\textbf{Does selective write-back become more useful with longer prefixes?} Across $i\in\{10,20,30,40,50\}$, selection is host-dependent at earlier prefixes. By $i=40$, it improves both MRR and nDCG@10 for three of four hosts, and by $i=50$ for all four. Thus, selectivity becomes more consistent across hosts as corrective evidence accumulates. Selection acts at the finer timestep-pattern level, so an interaction can contribute one corrective pattern while another is suppressed (Appendix Table~\ref{tab:rq3_timestep_selection_length_ablation}, Appendix~\ref{app:length-selection-results}).

\textbf{Is task supervision alone sufficient?} $\mathcal{L}_{\mathrm{task}}$ tells \repair{} whether correction helps downstream, while $\mathcal{L}_{\mathrm{pos}}$ preserves \emph{where} evidence came from and $\mathcal{L}_{\mathrm{coh}}$ preserves \emph{which interaction} produced it. Across all four hosts in Appendix Table~\ref{tab:rq1_loss_progressive_ablation}, task supervision helps, timestep provenance adds further gains, and source coherence gives the strongest complete objective. Hence, preserving evidence provenance adds value beyond learning only that a correction helps the task. Theorem~\ref{thm:loss-design-recoverability} in Appendix~\ref{app:loss_recoverability} states the corresponding conditional local bound.

\subsection{RQ3: Where And How Is Recovered Preference Evidence Used?}\label{subsec:rq3}

\begin{table*}[t]
\centering
\caption{\textbf{RQ3: Intervention-locus comparison under the same frozen-host protocol.} Pre-head \repair{} is compared with two adapted post-head correctors using the same task-trained host and two-epoch downstream alignment protocol.}
\label{tab:rq2_corrector_comparison}
\scalebox{0.65}{
\begin{tabular}{llcccccccc}
\toprule
\multirow{2}{*}{\textbf{Tasks (Datasets)}} & \multirow{2}{*}{\textbf{Host Model}}
& \multicolumn{2}{c}{\textbf{Base}}
& \multicolumn{2}{c}{\textbf{+\texttt{SERAC-Corr}}}
& \multicolumn{2}{c}{\textbf{+\texttt{Rec-Denoiser-Corr}}}
& \multicolumn{2}{c}{\textbf{+\repair{}}} \\
\cmidrule(lr){3-4}\cmidrule(lr){5-6}\cmidrule(lr){7-8}\cmidrule(lr){9-10}
& & \textbf{MRR} & \textbf{nDCG@10} & \textbf{MRR} & \textbf{nDCG@10} & \textbf{MRR} & \textbf{nDCG@10} & \textbf{MRR} & \textbf{nDCG@10} \\
\midrule

\multirow{3}{*}{Movie Reco. (MovieLens)}
& \texttt{TiSASRec}  & 35.51 & 41.14 & 33.17 & 41.32 & 34.21 & 42.08 & 36.78 & 42.91 \\
& \texttt{Mamba4Rec} & 29.29 & 38.50 & 29.81 & 38.63 & 31.32 & 39.22 & 34.26 & 40.17 \\
& \texttt{SIGMA}     & 33.27 & 38.96 & 31.95 & 38.98 & 29.43 & 39.61 & 34.81 & 40.72 \\

\midrule

\multirow{3}{*}{News Reco. (MIND)}
& \texttt{LSTUR} & 51.43 & 59.63 & 53.66 & 62.17 & 52.45 & 61.35 & 55.84 & 66.92 \\
& \texttt{NAML}  & 45.65 & 46.56 & 46.47 & 51.33 & 48.12 & 48.94 & 53.22 & 63.18 \\
& \texttt{NRMS}  & 44.91 & 45.73 & 48.13 & 46.74 & 42.25 & 48.23 & 52.74 & 62.44 \\

\midrule

\multirow{2}{*}{Product Reco. (Amazon)}
& \texttt{SigLIP2} & 2.40 & 2.18 & 7.43 & 8.67 & 9.54 & 10.61 & 17.72 & 20.35 \\
& \texttt{OpenCLIP} & 1.59 & 1.05 & 5.32 & 7.41 & 6.37 & 8.11 & 19.49 & 22.08 \\

\bottomrule
\end{tabular}
}
\end{table*}


\textbf{Is pre-head repair more effective than post-head correction?} We compare state correction before the head with two implemented output-side alternatives. Under the same task-trained checkpoint, frozen encoder, and two-epoch alignment protocol, pre-head \repair{} achieves the best MRR and nDCG@10 in all \textbf{8/8} settings against \texttt{SERAC-Corr} and \texttt{Rec-Denoiser-Corr}. This comparison favors the evaluated pre-head configuration. The adapted mechanisms and the limits of attributing their difference to intervention location are specified in Appendix~\ref{app:corrective-baseline-details}.

\begin{table}[t]
\centering
\caption{RQ3 generation and personalization on PENS. Values are relative gains (\%). Full ranking and generation comparisons are in Appendix Table~\ref{tab:rq3_task_interface_full}. Absolute scores are in Appendix Tables~\ref{tab:rq1_pens_accuracy_main_all} and~\ref{tab:rq1_pens_pse_main_selected}.}
\label{tab:rq3_task_interface}
\scalebox{0.83}{
\begin{tabular}{lrrrr}
\toprule
\textbf{Host} & \textbf{RG-L} & \textbf{METEOR} & \textbf{PSE-W-JSD} & \textbf{PSE-W-RG-L} \\
\midrule
\texttt{Walk2Pers} & +6.58 & +13.38 & +13.25 & +15.32 \\
\texttt{IMPerSumm} & -2.69 & +4.76 & +20.76 & +25.23 \\ \hdashline
\texttt{DeepSeek-32B} & +53.21 & +45.49 & +9.32 & +5.05 \\
\texttt{Qwen2.5-32B} & +41.58 & +47.64 & +31.52 & +12.61 \\
\bottomrule
\end{tabular}}
\end{table}

\textbf{Does \repair{} improve personalization in preference-coupled decoders?} PerSEval (PSE)~\citep{perseval-tmlr-2024} measures personalization with respect to evolving user preference rather than snapshot generation accuracy alone. For the same source document, it compares differences between users' generated outputs with differences between their expected outputs and penalizes poor accuracy. If \repair{} improves preference evidence, preference-coupled generators should preserve that benefit through decoding. We observe this for two generators in Table~\ref{tab:rq3_task_interface}: \texttt{Walk2Pers} gains $+13.25\%/+15.32\%$ on PSE-W-JSD/PSE-W-RG-L, while \texttt{IMPerSumm} gains $+20.76\%/+25.23\%$, with PSE-P-RG-L improving by $39.29\%$.

\textbf{Can \repair{} improve preference conditioning in a frozen LLM?} For LLM hosts, \repair{} encodes interaction-level evidence externally and maps the repaired state to a continuous prefix for the frozen decoder (Appendix~\ref{app:llm-hosts}). The gains in Table~\ref{tab:rq3_task_interface} differ across decoders. DeepSeek improves generation quality much more than personalization, whereas Qwen improves both. This supports the usability of the induced preference pathway, with host-dependent output gains. It does not directly identify the decoder mechanism responsible.

\section{Discussion And Limitations}
\label{sec:discussion-limitations}
For encoder hosts, \repair{} uses representations from the existing frozen forward pass. This avoids re-encoding the history, while adding the cost of retaining cached evidence and running the repair module. Across the evaluated predictive hosts, parameter-derived memory overhead is $31$ to $36\%$ and latency overhead is $38$ to $44\%$ (Appendix Tables~\ref{tab:rq4_accuracy_scalability_overhead_est} and~\ref{tab:rq4_compute_footprint}). Evidence discarded before the host processes the history is outside the recovery path. On TG-ReDial, preference actions must be inferred from recommendations, acceptance or rejection, and topic shifts. Sequential hosts benefit from repair, while the mixed \texttt{TextCNN} response limits claims of uniformly beneficial correction across these inferred interaction interfaces (Appendix Table~\ref{tab:rq1_tgredial_joint}, Appendix~\ref{app:tgredial-broader-applicability}). Additional noise and serving-path analyses are in Appendix Table~\ref{tab:rq3_noisy_trajectory_stress_blank_repair} and Appendix~\ref{app:deployment-results-expanded}.

\section{Conclusion}

In this paper, we propose \repair{}, a method for correcting compressed preference states using evidence retained in cached representations. \repair{} learns compact corrective patterns, resolves their temporal support, and selectively updates the state before downstream use. Across twelve recommendation hosts, it improves MRR and nDCG@10 with both encoder and task head frozen, outperforming head-only finetuning in every evaluated setting. The analyses identify host-dependent corrective timing and show that personalization gains in generation can differ from gains in generic text quality. These findings support post-compression correction as a way to improve personalization without updating the host encoder. Future work can study how to retain useful corrective evidence under tighter cache budgets and how downstream interfaces can better use repaired states.



\section*{Reproducibility Statement}
Appendix~\ref{app:method-details} specifies the repair computation and host interfaces, and Appendix~\ref{app:theory} states the assumptions and proofs of the theoretical results. Appendices~\ref{app:evaluation-details} and~\ref{app:metrics} describe data construction, training stages, evaluation protocols, and metric definitions. Appendix~\ref{app:task-baselines} identifies the hosts, Appendix~\ref{app:corrective-baseline-details} describes the adapted correctors, and Appendix~\ref{app:additional-results} provides extended analyses. Architecture and hyperparameters are given in Table~\ref{tab:repair_hyperparameters}.

\bibliographystyle{plainnat}
\bibliography{REPAIR_merged_references}

\appendix
\addtocontents{toc}{\protect\setcounter{tocdepth}{2}}

\section{\textbf{\repair{}} Implementation Details}
\label{app:method-details}

The implementation follows one question through six operations. Given a user state that under-expresses some available history, which cached evidence should change that state? We first put the state and cached interactions in a common space, read their differences into $K$ corrective coordinates, resolve when those coordinates occur, retain selected activations, pool the retained evidence, and fuse the resulting correction with the original state. The notation is collected in Table~\ref{tab:symbols_notation}, and the dimensions and fixed settings are specified in Table~\ref{tab:repair_hyperparameters}. Here, MLP denotes a multilayer perceptron, $[x;y]$ denotes concatenation, $\odot$ denotes elementwise multiplication, and $\sigma$ denotes the sigmoid unless a different local meaning is stated.

For encoder-based hosts, the host encoder and its original task head are fixed during frozen-head repair training. The raw interaction history is processed once by the host. \repair{} receives the task-facing preference state and a sequence tensor already produced within that forward computation. Prompt-conditioned LLM hosts use the separate structured-interaction interface in Appendix~\ref{app:llm-hosts}. Unless stated otherwise, all equations below are evaluated only over valid prefix positions. Padded positions are masked from centering, temporal resolution, selection, aggregation, and auxiliary losses.

\subsection{User Interaction Graph And Eventization}
\label{app:uig-construction}

An interaction combines an action with its associated content. For example, a reader's click on an article becomes a positive interaction, a logged skip becomes a negative interaction, and a request for a headline supplies the task command. The graph records these relations, while the ordered prefix determines which interactions are available before the next prediction. This conversion from heterogeneous logs to action and content pairs is called eventization.

We use the User Interaction Graph (UIG) abstraction of \citet{peraugy-tmlr-2025} to represent heterogeneous histories as ordered interaction trajectories. For user $u$, $\mathcal{G}_u=\langle\mathcal{N}_u,\mathcal{E}_u\rangle,$ where $\mathcal{N}_u$ contains the user and task-dependent content, query, response, and supervised-target nodes, and $\mathcal{E}_u$ records observed actions. Actions are normalized as $a^{(t_i)}\in\{\texttt{pos},\texttt{neg},\texttt{cmd}\}.$ Here, \texttt{pos} denotes positive feedback such as click, consume, accept, purchase, or select. \texttt{neg} denotes negative feedback such as skip, reject, non-click, or low preference. \texttt{cmd} denotes the task request.

An observed interaction is
\begin{equation}
b^{(t_i)}
=
\langle a^{(t_i)},v^{(t_i)}\rangle,
\qquad
v^{(t_i)}\in\mathcal{N}_u ,
\label{eq:app_interaction_unit}
\end{equation}
and a prefix ending at $t_i$ is
\begin{equation}
\tau_b^{t_{1:i}}
=
\langle
b^{(t_1)},\ldots,b^{(t_i)}
\rangle .
\label{eq:app_interaction_sequence}
\end{equation}
Predictive personalization uses the next content item as ${r^\star}_u^{(t_{i+1})}$. For generative personalization, the supervised target is the reference response and the task-side source or query $q^{(t_i)}$ is stored separately from the user-preference state. Dataset-specific trajectory update rules are given with the evaluation protocol.

\subsection{Frozen Host Interface}
\label{app:host-interface}

The host produces two related objects. Its sequence tensor retains one representation per valid history position, while its task-facing state is the single representation passed to the task head. The cache boundary specifies where we read the sequence tensor before the host compresses or pools it. Consequently, the cache contains already-computed evidence and is not a second encoding of the history.

For an encoder-based host, let
\begin{equation}
E_b^{(t_i)}
=
g_{\mathrm{enc}}
\!\left(
\tau_b^{t_{1:i}}
\right)
=
\left[
e_b^{(t_1)},\ldots,e_b^{(t_i)}
\right],
\qquad
e_b^{(t_j)}\in\mathbb{R}^{d_e},
\label{eq:app_host_sequence}
\end{equation}
denote the sequence-position representations produced by the frozen history computation at the cache boundary. The host subsequently forms the task-facing preference state
\begin{equation}
z_b^{(t_i)}
=
c_{\mathrm{host}}
\!\left(
E_b^{(t_i)}
\right),
\qquad
z_b^{(t_i)}\in\mathbb{R}^{d_z}.
\label{eq:app_host_compression}
\end{equation}
Both $g_{\mathrm{enc}}$ and $c_{\mathrm{host}}$ are part of the trained host and remain unchanged in the frozen-head condition.

The cache boundary is the sequence tensor supplied to the host's terminal history-compression or state-readout operation. For recurrent hosts, this is the sequence of recurrent representations from the same recurrence used to form the terminal state. For self-attentive and news hosts, it is the contextualized sequence before terminal pooling or state extraction. For state-space hosts, it is the per-step sequence representation available before terminal state readout. For multimodal recommendation hosts, it is the fused item sequence consumed by history compression. No additional history encoder is introduced to produce $E_b^{(t_i)}$.

For batched computation with maximum history length $L$, let $m_j^{(t_i)}=\mathbf{1}[j\leq i]$ denote the validity mask. All sequence operations below are restricted to $m_j^{(t_i)}=1$. The cache requires $O(Ld_e)$ storage per active prefix and is released after the corresponding forward/evaluation step.

\subsection{State-Relative Corrective Evidence}
\label{app:vector-implementation}

The cached vector and final state may have different dimensions or meanings. Their raw coordinates therefore cannot be subtracted directly. The projections $P_e$ and $P_z$ first map them into a shared $m$-dimensional space. A residual then expresses how one projected history position differs from the projected state. The read vectors turn that difference into $K$ learned corrective activations, and the write vectors later map those activations back toward the state interface.

The task-facing preference state and cached per-timestep representations are mapped into a common repair space:
\begin{equation}
\bar{z}_b^{(t_i)}
=
P_z z_b^{(t_i)},
\qquad
\bar{e}_b^{(t_j)}
=
P_e e_b^{(t_j)},
\label{eq:app_repair_space}
\end{equation}
with $P_z\in\mathbb{R}^{m\times d_z}, \qquad P_e\in\mathbb{R}^{m\times d_e}.$ Their state-relative residual is
\begin{equation}
r^{(t_j;t_i)}
=
\bar{e}_b^{(t_j)}
-
\bar{z}_b^{(t_i)}
\in\mathbb{R}^{m}.
\label{eq:app_unreduced_residual}
\end{equation}
This quantity is defined as a representation-space difference and is not an information-theoretic measure of missing content.

The repair operator is parameterized as
\begin{equation}
J_{\mathrm{rep}}
=
\hat{U}_{\mathrm{rep}}
\Lambda_{\mathrm{rep}}
\hat{V}_{\mathrm{rep}}^{\top},
\label{eq:app_low_rank_map}
\end{equation}
where
\[
\hat{U}_{\mathrm{rep}}
=
\mathrm{colnorm}(U_{\mathrm{rep}}),
\qquad
\hat{V}_{\mathrm{rep}}
=
\mathrm{colnorm}(V_{\mathrm{rep}}),
\]
\[
U_{\mathrm{rep}},V_{\mathrm{rep}}
\in\mathbb{R}^{m\times K},
\qquad
\Lambda_{\mathrm{rep}}
=
\operatorname{diag}
\left(
\operatorname{softplus}(\lambda_{\mathrm{raw}})
\right).
\]
Column normalization gives each read or write vector unit norm. It does not make different columns orthogonal, and $\Lambda_{\mathrm{rep}}$ contains learned positive transmission scales rather than measured singular values. The read coordinates are
\begin{equation}
d^{(t_j)}
=
\hat{V}_{\mathrm{rep}}^\top
r^{(t_j;t_i)}
\in\mathbb{R}^{K}.
\label{eq:app_residual_coordinate}
\end{equation}

Centering removes the mean activation of each coordinate over the current prefix. The resulting field emphasizes variation across history positions before temporal weighting. For a valid prefix of length $i$,
\begin{equation}
\tilde{d}^{(t_j)}
=
d^{(t_j)}
-
\frac{1}{i}
\sum_{q=1}^{i}
d^{(t_q)} .
\label{eq:app_centered_residual}
\end{equation}
In padded batches, the mean is computed only over valid positions and the centered values at invalid positions are set to zero.

The state term is shared by every position in a fixed prefix and therefore cancels under centering. Substitution gives
\begin{equation}
\tilde d^{(t_j)}=\hat V_{\mathrm{rep}}^\top P_e\left(e_b^{(t_j)}-\frac{1}{i}\sum_{q=1}^{i}e_b^{(t_q)}\right).
\label{eq:app_centering_cancellation}
\end{equation}
The temporal field consequently encodes variation within the cached sequence rather than its common offset from the state. State conditioning remains in the later coordinate gate, aggregation scores, and residual fusion. The read-coordinate energy analysis in Appendix~\ref{app:k_choice} concerns the uncentered residual and should not be read as an energy guarantee for this centered field.

\begin{figure*}[t]
    \centering
    \includegraphics[width=0.85\textwidth]{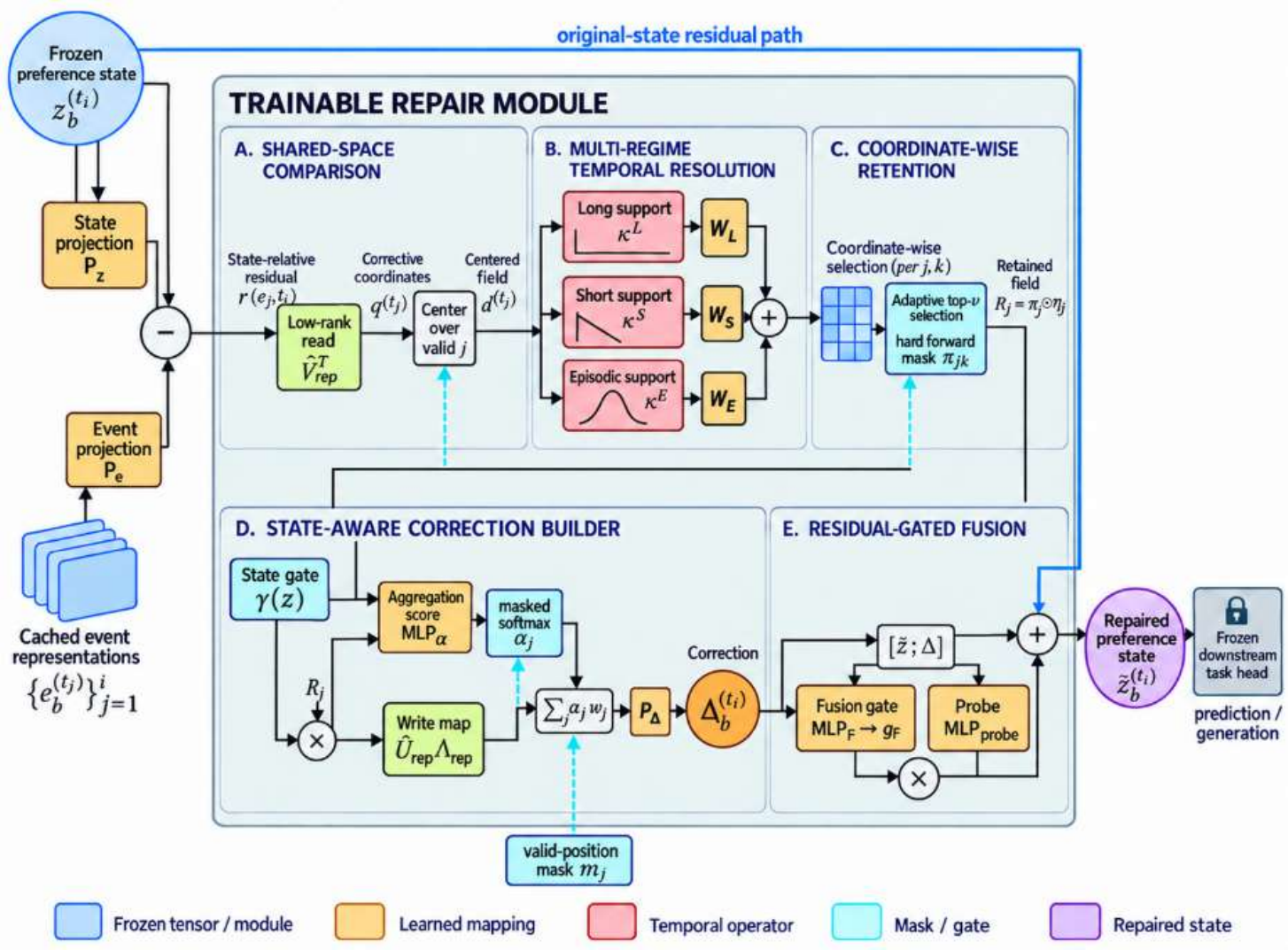}
    \caption{Encoder-host repair from cached evidence to the task head. Read from left to right. The state and cached sequence enter a shared space, the low-rank reader produces corrective coordinates, and centering precedes temporal resolution and coordinate-wise selection. State gating and weighted aggregation form a correction, which is fused with the original state. The diagram's raw-residual label denotes $r^{(t_j;t_i)}$ in Appendix~\ref{app:vector-implementation}. Dashed arrows indicate conditioning, and the original state remains the residual base.}
    \label{fig:repair_detailed}
\end{figure*}

\subsection{Temporal Resolution}
\label{app:temporal-responses}

Temporal resolution gives the same corrective activation different weight according to its lag from the current state. The exponential branch provides a learned decay, the linear branch provides a fixed recency envelope over the prefix, and the Poisson-shaped branch can emphasize a localized lag. The labels $L$, $S$, and $E$ name these functions. They do not label the user's underlying intent. The matrices $W_L$, $W_S$, and $W_E$ then mix the weighted coordinates into a shared corrective field.

For timestep $t_j$, define lag $\ell_j=i-j.$ The long-support, short-support, and episodic-support bases are coordinate-wise functions
\begin{equation}
\begin{split}
& \kappa_k^{L}(\ell_j) = \exp(-\beta_k\ell_j), \\
& \kappa_k^{S}(\ell_j) = \max\!\left(1-\frac{\ell_j}{i},0\right), \\
& \kappa_k^{E}(\ell_j) = \exp\!\left( \ell_j\log\xi_k -\xi_k -\log\Gamma(\ell_j+1) \right),
\end{split}
\label{eq:app_temporal_responses}
\end{equation}
where
\begin{equation}
\beta
=
\operatorname{softplus}(\beta_{\mathrm{raw}})
+10^{-4},
\qquad
\xi
=
\operatorname{softplus}(\xi_{\mathrm{raw}})
+10^{-4}.
\label{eq:app_positive_temporal_parameters}
\end{equation}

The resolved corrective field is
\begin{equation}
\eta^{(t_j)}
=
\sum_{\rho\in\{L,S,E\}}
W_{\rho}
\left(
\kappa^\rho(\ell_j)
\odot
\tilde{d}^{(t_j)}
\right),
\label{eq:app_regime_evidence}
\end{equation}
where each $W_{\rho}:\mathbb{R}^{K}\rightarrow\mathbb{R}^{K}$ is a trainable bias-free coordinate-mixing map. Invalid positions are zeroed after temporal resolution.

\subsection{Adaptive Top-\texorpdfstring{$\nu$}{nu} Retention}
\label{app:selector}

The selector decides which timestep activations are eligible to contribute to each corrective coordinate. For example, one interaction may contribute coordinate $k$ while its activation in another coordinate is suppressed. Selection is therefore finer than retaining or discarding an entire interaction. A hard mask determines the forward computation, while a smooth surrogate supplies gradients during training.

Selection is performed independently for each corrective coordinate. For coordinate $k$, define
\begin{equation}
\mu^{\mathrm{abs}}_k
=
\frac{1}{i}
\sum_{q=1}^{i}
|\eta_k^{(t_q)}|,
\qquad
\mu_k
=
\frac{1}{i}
\sum_{q=1}^{i}
\eta_k^{(t_q)},
\label{eq:app_selector_means}
\end{equation}
\begin{equation}
\sigma_k
=
\sqrt{
\frac{1}{i}
\sum_{q=1}^{i}
\left(
\eta_k^{(t_q)}-\mu_k
\right)^2
+10^{-8}
},
\label{eq:app_selector_std}
\end{equation}
and
\begin{equation}
s_k^{(t_i)}
=
\left[
\mu_k^{\mathrm{abs}};
\sigma_k;
\max_{1\leq q\leq i}|\eta_k^{(t_q)}|
\right].
\label{eq:app_selector_stats}
\end{equation}
The selector predicts a retained fraction
\begin{equation}
f_k^{(t_i)}
=
\sigma
\left(
\mathrm{MLP}_{\mathrm{sel}}
\left(
s_k^{(t_i)}
\right)
\right)
\in(0,1).
\label{eq:app_selector_fraction}
\end{equation}

Let $R_{j,k}^{(t_i)} \in \{0,\ldots,i-1\}$ be the zero-based rank of $|\eta_k^{(t_j)}|$ after sorting valid timesteps in descending magnitude. Its normalized rank position is
\begin{equation}
\rho_{j,k}^{(t_i)}
=
\frac{
R_{j,k}^{(t_i)}+0.5
}{i}.
\label{eq:app_rank_fraction}
\end{equation}
The hard forward mask is
\begin{equation}
h_{j,k}^{(t_i)}
=
\mathbf{1}
\left[
\rho_{j,k}^{(t_i)}
\leq
f_k^{(t_i)}
\right].
\label{eq:app_selector_hard}
\end{equation}
For gradient propagation, the implementation uses
\begin{equation}
s_{j,k}^{\mathrm{soft}}
=
\sigma
\left(
\frac{
f_k^{(t_i)}
-
\rho_{j,k}^{(t_i)}
}{
\tau_{\mathrm{sel}}
}
\right),
\qquad
\tau_{\mathrm{sel}}=0.05,
\label{eq:app_selector_soft}
\end{equation}
and the straight-through mask
\begin{equation}
\pi_{j,k}^{(t_i)}
=
h_{j,k}^{(t_i)}
+
s_{j,k}^{\mathrm{soft}}
-
\operatorname{stopgrad}
\left(
s_{j,k}^{\mathrm{soft}}
\right).
\label{eq:app_selector_st}
\end{equation}
The retained field is
\begin{equation}
\Pi_{\nu}^{(t_i)}
\left(
\eta^{(t_j)}
\right)
=
\pi_j^{(t_i)}
\odot
\eta^{(t_j)}.
\label{eq:app_selected_evidence}
\end{equation}
The rank threshold may retain zero positions for a coordinate when the predicted fraction lies below the first rank midpoint. There is no additional minimum-one rule in these equations. The implied coordinate-specific forward cardinality is
\[
\nu_k^{(t_i)}
=
\sum_{j=1}^{i}
h_{j,k}^{(t_i)}.
\]
The notation top-$\nu$ therefore refers to the hard forward support induced by the learned retained fraction. Sorting and the hard comparison are used for forward routing. Gradients to the selector are supplied through Equation~\ref{eq:app_selector_st}. No unbiased-gradient assumption is made.

\subsection{Correction Aggregation And Fusion}
\label{app:aggregation-fusion}

Three weights serve different purposes here. The selector determines eligibility for each timestep and coordinate, the state gate $\gamma$ scales corrective coordinates according to the current state, and the normalized weights $\alpha_j$ pool evidence across valid timesteps. After pooling, residual-gated fusion learns how the correction should alter each state component while keeping the original state as the additive base.

The current task-facing preference state produces a corrective-coordinate gate
\begin{equation}
\gamma^{(t_i)}
=
\sigma
\left(
\mathrm{MLP}_{\gamma}
\left(
z_b^{(t_i)}
\right)
\right)
\in(0,1)^K.
\label{eq:app_gate}
\end{equation}
For each valid timestep, an aggregation score is computed from the retained field and the state gate:
\begin{equation}
a_j^{(t_i)}
=
\mathrm{MLP}_{\alpha}
\left(
\left[
\Pi_{\nu}^{(t_i)}
(\eta^{(t_j)});
\gamma^{(t_i)}
\right]
\right),
\label{eq:app_attention_score}
\end{equation}
followed by masked softmax
\begin{equation}
\alpha_j^{(t_i)}
=
\frac{
\exp(a_j^{(t_i)})
}{
\sum_{q=1}^{i}
\exp(a_q^{(t_i)})
}.
\label{eq:app_attention}
\end{equation}

The write-space representation at timestep $t_j$ is
\begin{equation}
w^{(t_j)}
=
\hat{U}_{\mathrm{rep}}
\Lambda_{\mathrm{rep}}
\left(
\gamma^{(t_i)}
\odot
\Pi_{\nu}^{(t_i)}
(\eta^{(t_j)})
\right)
\in\mathbb{R}^{m}.
\label{eq:app_timestep_write}
\end{equation}
The task-facing correction is
\begin{equation}
\Delta_b^{(t_i)}
=
P_{\Delta}
\left(
\sum_{j=1}^{i}
\alpha_j^{(t_i)}
w^{(t_j)}
\right),
\qquad
P_{\Delta}\in\mathbb{R}^{d_z\times m}.
\label{eq:app_delta}
\end{equation}

Let
\[
p_b^{(t_i)}
=
\left[
z_b^{(t_i)};
\Delta_b^{(t_i)}
\right]
\in\mathbb{R}^{2d_z}.
\]
Residual-gated fusion is
\begin{equation}
\begin{split}
& g_F^{(t_i)} = \sigma \left( \mathrm{MLP}_{F} \left( p_b^{(t_i)} \right) \right), \\
& \tilde{z}_b^{(t_i)} = z_b^{(t_i)} + g_F^{(t_i)} \odot \mathrm{MLP}_{\mathrm{probe}} \left( p_b^{(t_i)} \right).
\end{split}
\label{eq:app_fusion}
\end{equation}
Accordingly,
\[
\mathrm{MLP}_{F}:
\mathbb{R}^{2d_z}\rightarrow\mathbb{R}^{d_z},
\qquad
\mathrm{MLP}_{\mathrm{probe}}:
\mathbb{R}^{2d_z}\rightarrow\mathbb{R}^{d_z}.
\]
The original task-facing preference state appears explicitly as the residual base of $\tilde{z}_b^{(t_i)}$. The downstream task interface receives $\tilde{z}_b^{(t_i)}$ in place of $z_b^{(t_i)}$.

\label{app:training-objective}

The three outer coefficients $\lambda_{\mathrm{task}}$, $\lambda_{\mathrm{pos}}$, and $\lambda_{\mathrm{coh}}$ below are the family weights denoted by $\omega_{\mathrm{task}}$, $\omega_{\mathrm{pos}}$, and $\omega_{\mathrm{coh}}$ in Section~\ref{sec:methodology}. The appendix reserves $\omega$ for coefficients inside each family.

The main methodology organizes repair supervision into three functional loss families:
\[
\mathcal{L}_{\mathrm{task}},
\qquad
\mathcal{L}_{\mathrm{pos}},
\qquad
\mathcal{L}_{\mathrm{coh}}.
\]
These families specify distinct supervisory roles rather than individual primitive losses. The task-utility family constrains the repaired state and aggregated correction to remain predictive of the supervised task target. The timestep-provenance family preserves the relative temporal origin of timestep-level corrective representations. The source-coherence family preserves correspondence between corrective representations and the action, interaction, and target-side evidence from which they are constructed.

For encoder hosts, the complete repair objective is
\[
\mathcal{L}_{\mathrm{repair}}^{\mathrm{enc}}
=
\lambda_{\mathrm{task}}
\mathcal{L}_{\mathrm{task}}^{\mathrm{enc}}
+
\lambda_{\mathrm{pos}}
\mathcal{L}_{\mathrm{pos}}^{\mathrm{enc}}
+
\lambda_{\mathrm{coh}}
\mathcal{L}_{\mathrm{coh}}^{\mathrm{enc}},
\]
with
\[
\left(
\lambda_{\mathrm{task}},
\lambda_{\mathrm{pos}},
\lambda_{\mathrm{coh}}
\right)
=
(0.5,0.2,0.3).
\]
The coefficients $\lambda_{\mathrm{task}}$, $\lambda_{\mathrm{pos}}$, and $\lambda_{\mathrm{coh}}$ therefore weight the three functional families. Each family is itself instantiated as a weighted combination of implementation-level supervisory terms, whose within-family coefficients are specified separately. The complete encoder-host decomposition, coefficients, training-side perturbations, and task-specific instantiations are given in Appendix~\ref{app:repair-objective-implementation}.

During frozen-head repair training, the host encoder and original downstream task interface remain fixed. Gradients update only \repair{} and its associated repair-side projections and supervision heads. Downstream-interface parameters are updated only in the separate alignment stage described in Appendix~\ref{app:evaluation-details}.

\subsection{Prompt-Conditioned LLM Hosts}
\label{app:llm-hosts}

For a large language model (LLM), tokens and user interactions do not have a one-to-one correspondence. This branch therefore constructs a separate vector for each serialized interaction and compares those vectors with the frozen LLM's terminal prompt state. The output of repair becomes a soft prefix, meaning a sequence of learned continuous vectors supplied before the decoder's ordinary token embeddings.

Prompt-conditioned LLM hosts use a separate input construction. Let
\begin{equation}
x_{\mathrm{task}}^{(t_i)}
=
\mathrm{Prompt}
\left(
\tau_b^{t_{1:i}},
q^{(t_i)}
\right),
\qquad
z_b^{(t_i)}
=
\phi_{\mathrm{LLM}}^{\mathrm{last}}
\left(
x_{\mathrm{task}}^{(t_i)}
\right),
\label{eq:llm_prompt_state}
\end{equation}
where $z_b^{(t_i)}$ is the terminal hidden representation of the frozen prompt computation.

The interaction representations in this branch are not native per-timestep hidden states from that prompt computation. Each serialized history interaction is represented separately as
\begin{equation}
e_b^{(t_j)}
=
P_c c_b^{(t_j)}
+
P_a a_b^{(t_j)},
\qquad
j\leq i,
\label{eq:llm_interaction_rep}
\end{equation}
where $c_b^{(t_j)}$ is the interaction-content representation and $a_b^{(t_j)}$ is the action embedding. The resulting pair $\left( z_b^{(t_i)}, \{e_b^{(t_j)}\}_{j=1}^{i} \right)$ is passed through the same residual, temporal-resolution, selection, aggregation, and fusion operators defined above.

The repaired state is converted to a continuous prefix
\begin{equation}
S_b^{(t_i)}
=
\mathrm{MLP}_{\mathrm{sp}}
\left(
\tilde{z}_b^{(t_i)}
\right)
\in
\mathbb{R}^{M\times d_{\mathrm{LLM}}}.
\label{eq:llm_soft_prompt}
\end{equation}
$S_b^{(t_i)}$ is prepended to the token embeddings used for prefix-conditioned decoding. The LLM backbone remains frozen. The trainable components in this branch are \repair{} and the soft-prefix projection.

Here, $M$ is the number of soft-prefix positions and $d_{\mathrm{LLM}}$ is the decoder embedding dimension. The key and value (KV) cache stores intermediate attention keys and values for a particular decoder input. Prefix insertion changes that input sequence. The decoder forward pass is therefore evaluated with the inserted prefix rather than reusing a KV cache computed for the unmodified prompt. Reported LLM latency includes this prefix-conditioned decoding step. The no-second-history-encoder condition used for encoder-based hosts does not imply that the original prompt forward cache can be reused after prefix insertion.

\begin{figure*}[t]
\centering
\begin{tikzpicture}[x=1cm,y=1cm,box/.style={draw,rounded corners,align=center,text width=3.1cm,minimum height=0.75cm,font=\small},flow/.style={->,thick}]
\node[box,fill=blue!8] (prompt) at (0,0) {Serialized history and task query};
\node[box,fill=blue!8] (llm) at (0,-1.5) {Frozen LLM prompt computation};
\node[box,fill=green!8] (events) at (4,0) {Separate content and action representations};
\node[box,fill=green!8] (eventvec) at (4,-1.5) {Interaction vectors $e_b^{(t_j)}$};
\node[box,fill=orange!10,text width=7cm] (repair) at (2,-3.1) {\repair{} reads state-relative evidence, resolves time, selects, aggregates, and fuses};
\node[box,fill=purple!10] (prefix) at (2,-4.6) {Repaired state $\tilde z_b^{(t_i)}$ mapped to soft prefix $S_b^{(t_i)}$};
\node[box,fill=blue!8] (decoder) at (2,-6.2) {Frozen decoder with soft prefix and task-token embeddings};
\node[box] (out) at (2,-7.7) {Ranking or generated output};
\draw[flow] (prompt)--(llm);
\draw[flow] (events)--(eventvec);
\draw[flow] (llm)--node[left,font=\small] {$z_b^{(t_i)}$} (repair.north west);
\draw[flow] (eventvec)--(repair.north east);
\draw[flow] (repair)--(prefix);
\draw[flow] (prefix)--(decoder);
\draw[flow] (decoder)--(out);
\end{tikzpicture}
\caption{Prompt-conditioned repair for a frozen LLM. The prompt computation supplies the terminal state, while a separate interaction construction supplies corrective evidence. Repair maps the fused state into continuous prefix vectors. The frozen decoder then evaluates the prefix-conditioned input to produce the task output. Decoder weights remain fixed, and the attention cache must correspond to this modified input.}
\label{fig:repair_llm-prompt-pipeline}
\end{figure*}
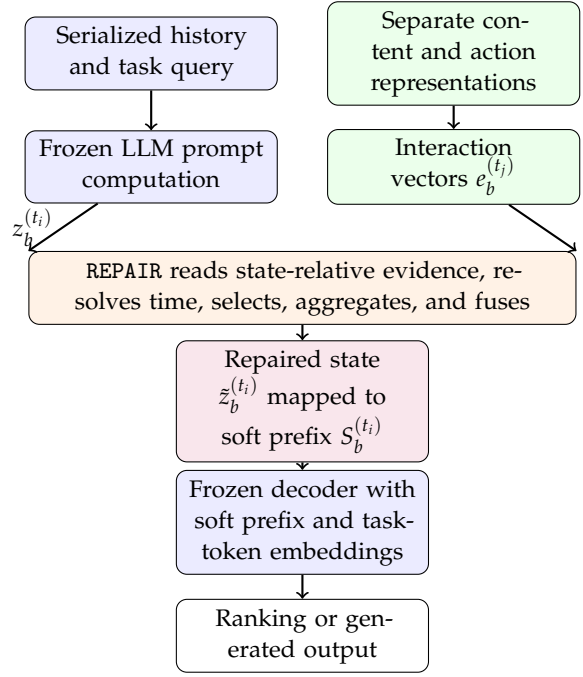

\begin{figure*}[t]
    \centering
    \includegraphics[width=0.85\textwidth]{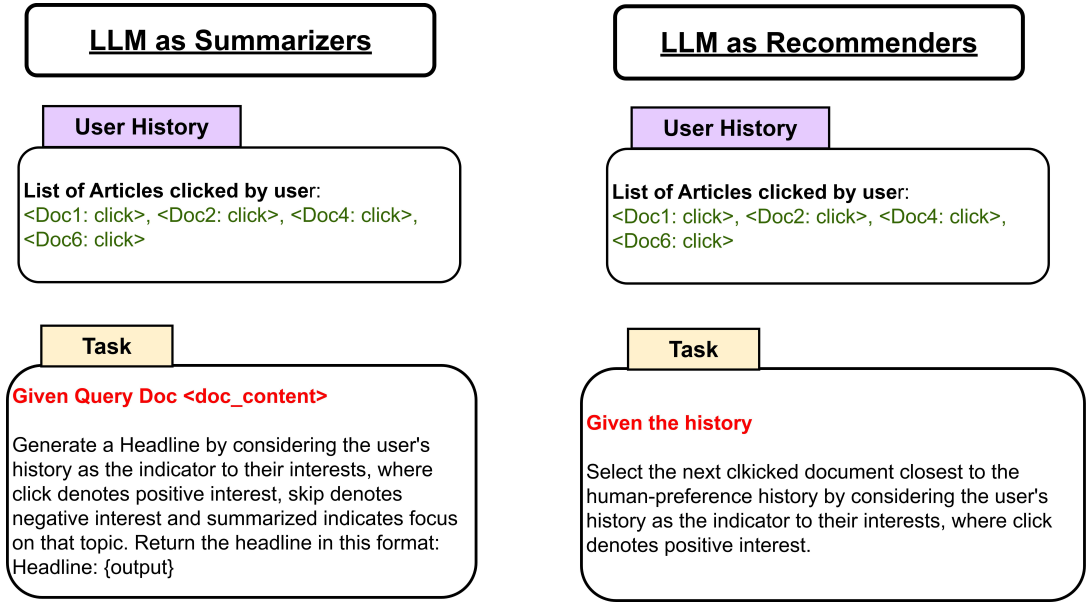}
    \caption{Task prompts used for the LLM headline-generation and recommendation settings. The history supplies preference evidence, while the task block specifies the requested output and, for generation, the current source document. The examples illustrate prompt structure rather than additional evaluation instances. The repair-specific continuous prefix is constructed separately as described in Appendix~\ref{app:llm-hosts}.}
    \label{fig:prompt}
\end{figure*}

The interaction encoder, repaired state, and soft-prefix projection are located in Figure~\ref{fig:repair_llm-prompt-pipeline}. The task prompts for serializing histories and requests are reproduced in Figure~\ref{fig:prompt}.

\begin{table}[!t]
\centering
\caption{Notation for the interaction sequence, frozen host interface, repair operators, and supervision. Hats on $U_{\mathrm{rep}}$ and $V_{\mathrm{rep}}$ denote column normalization. The family weights denoted by $\omega$ in the main text are denoted by $\lambda$ here, with $\omega$ reserved for within-family coefficients.}
\label{tab:symbols_notation}
\setlength{\tabcolsep}{3pt}
\renewcommand{\arraystretch}{0.95}
\scalebox{0.6}{
\begin{tabular}{p{0.18\textwidth}p{0.28\textwidth}p{0.18\textwidth}p{0.28\textwidth}}
\toprule
\textbf{Symbol} & \textbf{Meaning} & \textbf{Symbol} & \textbf{Meaning} \\
\midrule

\multicolumn{4}{l}{\textit{Interaction and host interface}} \\
$\mathcal{G}_u=\langle \mathcal{N}_u,\mathcal{E}_u\rangle$ & User interaction graph
& $\mathcal{N}_u$ & User-specific typed node set \\
$\mathcal{E}_u$ & Action-labeled edge set
& $u^{(t_0)}$ & Initial user node \\
$v^{(t_i)}$ & Content or response node
& $a^{(t_i)}$ & Observed action label \\
$\{\mathrm{pos},\mathrm{neg},\mathrm{cmd}\}$ & Normalized action types
& $b^{(t_i)}=\langle a^{(t_i)},v^{(t_i)}\rangle$ & Interaction unit at timestep \\
$\tau^{t_{1:i}}_{b}$ & Model-side prefix trajectory
& $t_i$ & Current prefix endpoint \\
$t_j$ & Earlier prefix timestep
& $\ell_j=i-j$ & Lag from current timestep \\
$f_{\mathrm{enc}}$ & Frozen personalization encoder
& $\mathbf{e}^{(t_j)}_{b}$ & Cached per-timestep representation \\
$\mathbf{z}^{(t_i)}_{b}$ & Task-facing preference state
& $\mathbf{q}^{(t_i)}$ & Task-side query input \\
$r^{\star(t_{i+1})}_{u}$ & Supervised next-task target
& $h_{\mathrm{task}}$ & Prediction or generation interface \\

\midrule
\multicolumn{4}{l}{\textit{Core dimensions and projections}} \\
$d_e$ & Cached-representation dimension
& $d_z$ & Preference-state dimension \\
$m$ & Shared repair-space dimension
& $K$ & Corrective coordinate dimension \\
$H$ & Maximum theoretical prefix length
& $L$ & Host truncation window \\
$P_z$ & State-to-repair projection
& $P_e$ & Event-to-repair projection \\
$\bar{\mathbf{z}}^{(t_i)}_{b}=P_z\mathbf{z}^{(t_i)}_{b}$ & Projected host state
& $\bar{\mathbf{e}}^{(t_j)}_{b}=P_e\mathbf{e}^{(t_j)}_{b}$ & Projected cached representation \\

\midrule
\multicolumn{4}{l}{\textit{State-relative corrective representation}} \\
$\mathbf{r}^{(t_j;t_i)}$ & Unreduced state-relative residual
& $D_{\mathrm{rep}}$ & Corrective evidence reader \\
$\mathbf{d}^{(t_j)}$ & Low-rank corrective coordinate
& $\tilde{\mathbf{d}}^{(t_j)}$ & Centered corrective coordinate \\
$\mathcal{D}_i=\{\mathbf{d}^{(t_j)}\}_{j=1}^{i}$ & Learned corrective coordinate field
& $J_{\mathrm{rep}}$ & Low-rank repair operator \\
$V_{\mathrm{rep}}$ & Corrective read basis
& $U_{\mathrm{rep}}$ & Corrective write basis \\
$\Lambda_{\mathrm{rep}}$ & Learned positive transmission scales
&  &  \\

\bottomrule
\end{tabular}
}
\end{table}

\begin{table}[!t]
\centering
\caption{Notation continued from Table~\ref{tab:symbols_notation}. Temporal resolution, selective retention, aggregation, and training objective.}
\label{tab:symbols_notation_repair}
\setlength{\tabcolsep}{3pt}
\renewcommand{\arraystretch}{0.95}
\scalebox{0.6}{
\begin{tabular}{p{0.18\textwidth}p{0.28\textwidth}p{0.18\textwidth}p{0.28\textwidth}}
\toprule
\textbf{Symbol} & \textbf{Meaning} & \textbf{Symbol} & \textbf{Meaning} \\
\midrule

\multicolumn{4}{l}{\textit{Temporal resolution}} \\
$\rho\in\{L,S,E\}$ & Temporal basis index
& $\kappa^L(\ell_j)$ & Long-support temporal basis \\
$\kappa^S(\ell_j)$ & Short-support temporal basis
& $\kappa^E(\ell_j)$ & Episodic-support temporal basis \\
$\beta$ & Learned long-support decay vector
& $\xi$ & Learned episodic-support rate vector \\
$W_L,W_S,W_E$ & Temporal branch-mixing matrices
& $G_{\mathrm{rep}}$ & Multi-basis resolution operator \\
$\boldsymbol{\eta}^{(t_j)}$ & Temporally resolved corrective vector
&  &  \\

\midrule
\multicolumn{4}{l}{\textit{Selective retention}} \\
$\Pi^{(t_i)}_{\nu}$ & Top-$\nu$ retention operator
& $\boldsymbol{\pi}^{(t_i)}_{j}$ & Timestep-coordinate selector mask \\
$\pi^{(t_i)}_{j,k}$ & Timestep-coordinate mask entry
& $\nu^{(t_i)}_{k}$ & Retained count per coordinate \\
$\mathbf{s}^{(t_i)}_{k}$ & Selector summary statistics
& $\mathrm{MLP}_{\mathrm{sel}}$ & Adaptive selector network \\

\midrule
\multicolumn{4}{l}{\textit{Aggregation and fusion}} \\
$\boldsymbol{\gamma}^{(t_i)}$ & State-conditioned coordinate gate
& $\alpha^{(t_i)}_{j}$ & Retained-timestep aggregation weight \\
$A_{\mathrm{rep}}$ & Correction aggregation operator
& $P_{\Delta}$ & Correction-to-state projection \\
$\boldsymbol{\Delta}^{(t_i)}_{b}$ & State-space correction vector
& $F_{\mathrm{rep}}$ & Residual repair fusion operator \\
$\mathbf{g}^{(t_i)}_{F}$ & Residual fusion gate
& $\mathrm{MLP}_{F}$ & Fusion gate network \\
$\mathrm{MLP}_{\mathrm{probe}}$ & Probe residual network
& $\tilde{\mathbf{z}}^{(t_i)}_{b}$ & Repaired preference state \\

\midrule
\multicolumn{4}{l}{\textit{Training objective}} \\
$\mathcal{L}_{\mathrm{task}}$ & Task-utility loss family
& $\mathcal{L}_{\mathrm{pos}}$ & Timestep-provenance loss family \\
$\mathcal{L}_{\mathrm{coh}}$ & Source-coherence loss family
& $\mathcal{L}_{\mathrm{repair}}$ & Complete repair objective \\
$\lambda_{\mathrm{task}}$ & Task-utility family weight
& $\lambda_{\mathrm{pos}}$ & Timestep-provenance family weight \\
$\lambda_{\mathrm{coh}}$ & Source-coherence family weight
& $\omega_{\cdot}$ & Within-family supervisory coefficient \\
$h_{\mathrm{pos}}$ & Relative-position classifier
& $P_{\mathrm{coh}}$ & Coherence projection matrix \\
$\tau$ & Contrastive temperature
&  &  \\

\bottomrule
\end{tabular}
}
\end{table}

\section{Theoretical Justification}
\label{app:theory}

The results answer five separate questions. Can a small subspace approximate the residual? Can retained evidence express a desired state correction? Can the temporal components be distinguished? What gradient reaches the temporal parameters? Under what local assumptions do the loss families control coefficient error? These questions concern representation and optimization conditions. Improved ranking or generation remains an empirical outcome rather than a conclusion of the bounds.

The analysis in this section follows the encoder-host \repair{} pathway defined in Appendix~\ref{app:method-details}. The methodology first constructs a state-relative residual between cached per-timestep representations and the task-facing preference state, reads that residual in a compact repair space, resolves its temporal support, selectively retains corrective coordinates, aggregates them, and writes the resulting correction back through residual-gated fusion. The theoretical results follow the same sequence. We characterize compact approximation of the measured residual, conditional realizability of the resulting state correction, identifiability and optimization of its temporal structure, and local control supplied by the three training-objective families. Each statement is conditional on its explicit assumptions. The experiments test observable consequences of these conditions rather than establishing the assumptions themselves. Unless stated otherwise, these results apply to the encoder-host path in which cached per-timestep representations arise from the same frozen host computation. The structured-interaction LLM interface in Appendix~\ref{app:llm-hosts} is not used to establish the same-forward-pass theoretical claim.

\subsection{Compact Approximation Of The State-Relative Residual}
\label{app:k_choice}

Compactness concerns the size of the space needed to approximate residual vectors. Imagine that most residual variation lies along a few directions. Retaining those directions can then preserve most of the approximation quality of a bounded linear write map. The analysis first considers the best spectral subspace, then an arbitrary learned subspace, and finally the distinct read-coordinate statistic measured in our experiments. Assume the residual has a finite second moment.

Appendix~\ref{app:vector-implementation} constructs the unreduced state-relative residual
\begin{equation}
r^{(t_j;t_i)}
=
P_e e_b^{(t_j)}
-
P_z z_b^{(t_i)}
\in
\mathbb{R}^{m},
\label{eq:theory_raw_residual}
\end{equation}
and reads it through a $K$-dimensional repair space. Let $r$ denote a residual drawn from the distribution over valid prefix and timestep pairs and define its second-moment matrix
\begin{equation}
C
=
\mathbb{E}
\left[
rr^\top
\right]
=
\sum_{q=1}^{m}
\lambda_q u_q u_q^\top,
\qquad
\lambda_1
\geq
\cdots
\geq
\lambda_m
\geq
0.
\label{eq:theory_residual_second_moment}
\end{equation}
We use the second moment rather than calling $C$ a covariance matrix because the unreduced residual in Equation~\ref{eq:theory_raw_residual} need not be globally mean zero.

For $K\leq m$, define the leading spectral projector and discarded residual energy as
\begin{equation}
\mathcal P_K
=
\sum_{q=1}^{K}
u_q u_q^\top,
\qquad
T_K
=
\sum_{q>K}
\lambda_q.
\label{eq:theory_spectral_tail}
\end{equation}

\begin{definition}[Admissible repair dimension]
\label{def:k_admissible}
Let $A:\mathbb{R}^{m}\rightarrow\mathbb{R}^{d_z}$ be a fixed bounded linear write map satisfying $\|A\|_2 \leq B_A.$ Let $\mathcal B(K)$ denote any monotone implementation-cost function for a $K$-dimensional repair space. For approximation tolerance $\varepsilon_{\rm app}>0$ and deployment budget $B_{\max}$, a rank $K$ is admissible when
\begin{equation}
B_A^2 T_K
\leq
\varepsilon_{\rm app}^2,
\qquad
\mathcal B(K)
\leq
B_{\max}.
\label{eq:theory_admissibility}
\end{equation}
Equivalently, define
\[
K_{\min}
=
\min
\left\{
k:
B_A^2
\sum_{q>k}
\lambda_q
\leq
\varepsilon_{\rm app}^2
\right\}
\]
and $K_{\max} = \max \left\{ k: \mathcal B(k) \leq B_{\max} \right\}.$ Any $K\in[K_{\min},K_{\max}]$ is admissible whenever this interval is nonempty.
\end{definition}

The definition does not assume that implementation cost is linear in $K$. This is useful for \repair{}, whose projections contain terms scaling linearly with $K$ together with coordinate-mixing maps whose parameter cost can scale quadratically.

\begin{theorem}[Compact residual approximation]
\label{thm:k_admissible}
For the leading rank-$K$ projector $\mathcal P_K$ and any bounded linear write map $A$ satisfying $\|A\|_2\leq B_A$,
\begin{equation}
\mathbb{E}
\left[
\left\|
Ar
-
A\mathcal P_Kr
\right\|_2^2
\right]
\leq
B_A^2 T_K.
\label{eq:theory_rank_bound}
\end{equation}
Consequently, every $K$ satisfying Equation~\ref{eq:theory_admissibility} meets the prescribed approximation tolerance. Moreover, for any $\tau>0$,
\begin{equation}
\Pr
\left(
\left\|
Ar
-
A\mathcal P_Kr
\right\|_2
\geq
\tau
\right)
\leq
\frac{
B_A^2T_K
}{
\tau^2
}.
\label{eq:theory_rank_probability}
\end{equation}
\end{theorem}

\begin{proof}
Because $Ar-A\mathcal P_Kr = A(I-\mathcal P_K)r,$ submultiplicativity gives
\[
\left\|
A(I-\mathcal P_K)r
\right\|_2^2
\leq
B_A^2
\left\|
(I-\mathcal P_K)r
\right\|_2^2.
\]
Taking expectation,
\[
\mathbb{E}
\left[
\left\|
Ar-A\mathcal P_Kr
\right\|_2^2
\right]
\leq
B_A^2
\mathbb{E}
\left[
r^\top
(I-\mathcal P_K)
r
\right].
\]
Using Equation~\ref{eq:theory_residual_second_moment},
\[
\mathbb{E}
\left[
r^\top
(I-\mathcal P_K)
r
\right]
=
{\rm tr}
\left(
(I-\mathcal P_K)C
\right)
=
\sum_{q>K}\lambda_q
=
T_K,
\]
which proves Equation~\ref{eq:theory_rank_bound}. The admissibility statement follows directly from Definition~\ref{def:k_admissible}. For Equation~\ref{eq:theory_rank_probability}, apply Markov's inequality to $X = \left\| Ar-A\mathcal P_Kr \right\|_2^2,$ using $\mathbb{E}[X]\leq B_A^2T_K$.
\end{proof}

Theorem~\ref{thm:k_admissible} characterizes the optimal rank-$K$ spectral approximation of the measured residual distribution. \repair{} instead learns its read basis $V_{\rm rep}$ end-to-end. The next corollary connects the ideal spectral approximation to the subspace induced by this learned basis.

\begin{corollary}[Excess error of a learned repair subspace]
\label{cor:learned-repair-subspace}
Let $\widehat{\mathcal P}_K$ denote the orthogonal projector onto the learned $K$-dimensional subspace ${\rm span}(V_{\rm rep})$, assuming $V_{\rm rep}$ has rank $K$. Define
\begin{equation}
\delta_K
=
{\rm tr}
\left[
(\mathcal P_K-\widehat{\mathcal P}_K)C
\right].
\label{eq:theory_subspace_excess}
\end{equation}
Then $\delta_K \geq 0$ and
\begin{equation}
\mathbb{E}
\left[
\left\|
Ar
-
A\widehat{\mathcal P}_Kr
\right\|_2^2
\right]
\leq
B_A^2
\left(
T_K+\delta_K
\right).
\label{eq:theory_learned_rank_bound}
\end{equation}
\end{corollary}

\begin{proof}
For any rank-$K$ orthogonal projector $Q$,
\[
\mathbb{E}
\left[
\|
(I-Q)r
\|_2^2
\right]
=
{\rm tr}(C)
-
{\rm tr}(QC).
\]
The leading eigenspace $\mathcal P_K$ maximizes ${\rm tr}(QC)$ over rank-$K$ projectors, so
\[
{\rm tr}
\left[
(\mathcal P_K-\widehat{\mathcal P}_K)C
\right]
\geq0.
\]
Therefore,
\[
{\rm tr}
\left[
(I-\widehat{\mathcal P}_K)C
\right]
=
T_K+\delta_K.
\]
Applying the same operator-norm argument as in Theorem~\ref{thm:k_admissible} yields Equation~\ref{eq:theory_learned_rank_bound}. Here, $\widehat{\mathcal P}_K$ is the orthogonal projector associated with the learned span $\operatorname{span}(V_{\rm rep})$ and is used only for the approximation analysis. The implemented read operation itself is $d=\widehat V_{\rm rep}^{\top}r$, as defined in Appendix~\ref{app:vector-implementation}.
\end{proof}

\paragraph{Learned Read-Space Energy.}
To characterize the representation learned by \repair{}, we additionally measure the normalized energy transferred from the state-relative residual into its learned read coordinates. For $d=\widehat V_{\rm rep}^{\top}r,$ define
\begin{equation}
R_{\rm read}(K)
=
\frac{\mathbb E\|d\|_2^2}
     {\mathbb E\|r\|_2^2}
=
\frac{
{\rm tr}\!\left(
\widehat V_{\rm rep}^{\top}
C
\widehat V_{\rm rep}
\right)
}{
{\rm tr}(C)
}.
\label{eq:learned-read-energy}
\end{equation}

The ratio is defined when $\mathbb E\|r\|_2^2>0$. Because the implemented read vectors are column-normalized and need not be orthogonal, $R_{\rm read}(K)$ is not generally bounded by one. Overlapping read directions can count the same residual component more than once. Thus, $1-R_{\rm read}(K)$ is not a spectral-tail fraction. The orthogonal projector in Corollary~\ref{cor:learned-repair-subspace} is an analytical object distinct from this implemented read map. The equality follows from $\mathbb E[r^\top\widehat V_{\rm rep}\widehat V_{\rm rep}^{\top}r] = {\rm tr}(\widehat V_{\rm rep}^{\top}C\widehat V_{\rm rep})$.

\paragraph{Connection To The Methodology And RQ2.}
At $K=128$, the measured mean learned read-coordinate energy ratio is $96.46\%$ across five encoder-host settings. This is the statistic in Equation~\ref{eq:learned-read-energy}, not the energy fraction of the leading $128$ eigendirections. Separately, $K=128$ gives the best MRR and nDCG@10 among the four tested capacities in Table~\ref{tab:rq2_rank_sweep}. These observations support the usefulness of the chosen compact representation. They do not measure $T_K$ or $\delta_K$ and therefore do not numerically instantiate the spectral approximation bound in Theorem~\ref{thm:k_admissible}.

\subsection{Conditional Recoverability Through Selective Write-Back And Fusion}
\label{app:recoverability}

A small residual representation is useful only if its retained coordinates can produce the needed state change. The next bound separates four possible sources of error. The write space may miss part of the desired correction, retained coordinates may approximate it imperfectly, pooling weights may differ from the desired weights, and fusion may be unable to express the ideal final state. An oracle quantity denotes a hypothetical target used in the analysis, not an observed training label.

Compact residual approximation does not by itself imply that the resulting coordinates can realize a useful state correction. We next connect the selective write-back and aggregation operations in Appendices~\ref{app:selector} and~\ref{app:aggregation-fusion} to the frozen-head recoverability experiment in RQ1.

For fixed host state $z=z_b^{(t_i)}$, define the residual-fusion map
\begin{equation}
\phi_z(\Delta)
=
g_F(z,\Delta)
\odot
\mathrm{MLP}_{\rm probe}
\left(
[z;\Delta]
\right),
\label{eq:theory_fusion_map}
\end{equation}
so that the final repaired state is $\tilde z = z+\phi_z(\Delta).$ Let
\begin{equation}
W_{\rm wr}
=
P_\Delta
\widehat U_{\rm rep}
\Lambda_{\rm rep}
\label{eq:theory_write_map}
\end{equation}
denote the learned coordinate-to-correction write map.

\begin{theorem}[Conditional recoverability of state-relative \repair{}]
\label{thm:repair-recoverability}
Fix a trajectory prefix and frozen host state $z_b^{(t_i)}$. Let $z_b^{\star(t_i)}$ denote an ideal task-relevant state. Assume there exists an oracle pre-fusion correction $\Delta_b^{\star(t_i)}$ such that
\begin{equation}
\left\|
z_b^{\star(t_i)}
-
\left[
z_b^{(t_i)}
+
\phi_{z_b^{(t_i)}}
\left(
\Delta_b^{\star(t_i)}
\right)
\right]
\right\|_2
\leq
\varepsilon_{\rm fuse}.
\label{eq:theory_oracle_fusion}
\end{equation}

Suppose there exist oracle timestep coordinates $c^{(t_j)} \in \mathbb{R}^{K}$ and convex aggregation weights
\[
\alpha_j^{\star(t_i)}
\geq0,
\qquad
\sum_{j=1}^{i}
\alpha_j^{\star(t_i)}
=
1,
\]
such that
\begin{equation}
\left\|
\Delta_b^{\star(t_i)}
-
W_{\rm wr}
\sum_{j=1}^{i}
\alpha_j^{\star(t_i)}
c^{(t_j)}
\right\|_2
\leq
\varepsilon_{\rm span}.
\label{eq:theory_span_error}
\end{equation}

Let the coordinates and aggregation weights produced by \repair{} be
\[
\widehat c^{(t_j)}
=
\gamma^{(t_i)}
\odot
\Pi_\nu^{(t_i)}
\left(
\eta^{(t_j)}
\right)
\]
and
\[
\widehat\alpha_j^{(t_i)}
\geq0,
\qquad
\sum_{j=1}^{i}
\widehat\alpha_j^{(t_i)}
=
1.
\]
Assume
\begin{equation}
\left\|
\widehat c^{(t_j)}
-
c^{(t_j)}
\right\|_2
\leq
\varepsilon_j,
\label{eq:theory_coordinate_error}
\end{equation}
\begin{equation}
\left\|
\widehat\alpha^{(t_i)}
-
\alpha^{\star(t_i)}
\right\|_1
\leq
\varepsilon_\alpha,
\label{eq:theory_aggregation_error}
\end{equation}
and $\| c^{(t_j)} \|_2 \leq C_c \qquad \text{for all }j.$

Then the actual pre-fusion correction
\begin{equation}
\widehat\Delta_b^{(t_i)}
=
W_{\rm wr}
\sum_{j=1}^{i}
\widehat\alpha_j^{(t_i)}
\widehat c^{(t_j)}
\label{eq:theory_actual_delta}
\end{equation}
satisfies
\begin{equation}
\left\|
\widehat\Delta_b^{(t_i)}
-
\Delta_b^{\star(t_i)}
\right\|_2
\leq
\varepsilon_{\rm span}
+
\|W_{\rm wr}\|_2
\left[
\sum_{j=1}^{i}
\widehat\alpha_j^{(t_i)}
\varepsilon_j
+
C_c\varepsilon_\alpha
\right].
\label{eq:theory_delta_bound}
\end{equation}

If, additionally, $\phi_z$ is $L_F$-Lipschitz in $\Delta$ over the relevant correction region, then the final repaired state satisfies
\begin{equation}
\begin{split}
& \left\| \tilde z_b^{(t_i)} - z_b^{\star(t_i)} \right\|_2 \leq{} \varepsilon_{\rm fuse} + L_F\varepsilon_{\rm span} \\
& + L_F \|W_{\rm wr}\|_2 \left[ \sum_{j=1}^{i} \widehat\alpha_j^{(t_i)} \varepsilon_j + C_c\varepsilon_\alpha \right].
\end{split}
\label{eq:theory_final_state_bound}
\end{equation}
If $\varepsilon_j\leq\varepsilon_{\rm coord}$ for every $j$, then the weighted coordinate term is at most $\varepsilon_{\rm coord}$.
\end{theorem}

\begin{proof}
Let
\[
\bar\Delta_b^{\star(t_i)}
=
W_{\rm wr}
\sum_{j=1}^{i}
\alpha_j^{\star(t_i)}
c^{(t_j)}.
\]
By Equation~\ref{eq:theory_span_error},
\[
\left\|
\Delta_b^{\star(t_i)}
-
\bar\Delta_b^{\star(t_i)}
\right\|_2
\leq
\varepsilon_{\rm span}.
\]
For the realizable correction,
\[
\begin{split}
& \widehat\Delta_b^{(t_i)} - \bar\Delta_b^{\star(t_i)} = W_{\rm wr} \Bigg[ \sum_{j=1}^{i} \widehat\alpha_j^{(t_i)} \left( \widehat c^{(t_j)} - c^{(t_j)} \right) \\
& + \sum_{j=1}^{i} \left( \widehat\alpha_j^{(t_i)} - \alpha_j^{\star(t_i)} \right) c^{(t_j)} \Bigg].
\end{split}
\]
Hence
\[
\begin{split}
& \left\| \widehat\Delta_b^{(t_i)} - \bar\Delta_b^{\star(t_i)} \right\|_2 \leq \|W_{\rm wr}\|_2 \Bigg[ \sum_{j=1}^{i} \widehat\alpha_j^{(t_i)} \left\| \widehat c^{(t_j)} - c^{(t_j)} \right\|_2 \\
& + \sum_{j=1}^{i} \left| \widehat\alpha_j^{(t_i)} - \alpha_j^{\star(t_i)} \right| \| c^{(t_j)} \|_2 \Bigg].
\end{split}
\]
Applying Equations~\ref{eq:theory_coordinate_error} and~\ref{eq:theory_aggregation_error} and $\|c^{(t_j)}\|_2\leq C_c$ gives
\[
\left\|
\widehat\Delta_b^{(t_i)}
-
\bar\Delta_b^{\star(t_i)}
\right\|_2
\leq
\|W_{\rm wr}\|_2
\left[
\sum_{j=1}^{i}
\widehat\alpha_j^{(t_i)}
\varepsilon_j
+
C_c\varepsilon_\alpha
\right].
\]
A triangle inequality with $\Delta_b^{\star(t_i)}$ proves Equation~\ref{eq:theory_delta_bound}.

For the final state, Equation~\ref{eq:theory_oracle_fusion} and another triangle inequality give
\[
\begin{split}
& \left\| \tilde z_b^{(t_i)} - z_b^{\star(t_i)} \right\|_2 \leq{} \left\| \phi_z \left( \widehat\Delta_b^{(t_i)} \right) - \phi_z \left( \Delta_b^{\star(t_i)} \right) \right\|_2 + \varepsilon_{\rm fuse}.
\end{split}
\]
The $L_F$-Lipschitz assumption together with Equation~\ref{eq:theory_delta_bound} yields Equation~\ref{eq:theory_final_state_bound}. Finally, because the learned aggregation weights are nonnegative and sum to one,
\[
\sum_{j=1}^{i}
\widehat\alpha_j^{(t_i)}
\varepsilon_j
\leq
\varepsilon_{\rm coord}.
\]
\end{proof}

\paragraph{Connection To The Methodology And RQ1.}
Theorem~\ref{thm:repair-recoverability} bounds the complete pre-head pathway, including the low-rank write map, aggregation, and fusion. Its four approximation terms correspond to four implementation operations: $\varepsilon_{\rm span}$ measures whether the required pre-fusion correction lies near the learned write-back span, $\varepsilon_j$ measures timestep-coordinate recovery after temporal resolution and selective retention, $\varepsilon_\alpha$ measures aggregation error, and $\varepsilon_{\rm fuse}$ accounts for whether the residual-gated fusion map can express the desired state change. The theorem is conditional on the existence of an appropriate oracle correction and does not establish that information is absent from the original state.

RQ1 tests the observable consequence of this setting. In the \repair{} (Frozen Head) condition, both the task-trained encoder and original task head remain unchanged, and only the task-facing state supplied between them is modified using cached per-timestep representations. Improvements across all twelve representative recommendation hosts therefore show that useful state corrections can be realized from the available cached host computation for the same downstream task interface. Head-only continuation is a complementary same-state consumer control: it produces substantially smaller gains while leaving the task-facing preference state unchanged. Neither experiment identifies the oracle state or separately estimates the four error terms in Equation~\ref{eq:theory_final_state_bound}. Stage-3 \repair{} + Head Alignment is also outside the attribution claim of this theorem because \repair{} and its downstream consumer are jointly adapted there.

\subsection{Identifiability Of Temporal Residual Structure}
\label{app:temporal_recoverability}

Identifiability means that different coefficient choices cannot produce the same temporal trace for a fixed dictionary. If two temporal atoms are nearly identical over a short prefix, separating their contributions becomes unstable. The Gram matrix below records their pairwise similarity, and its smallest eigenvalue quantifies this separation. Full column rank for three atoms requires at least three valid positions.

Appendix~\ref{app:temporal-responses} resolves each centered corrective coordinate through long-support, short-support, and episodic-support kernels. We next characterize when the coefficients of this finite-prefix temporal dictionary are identifiable. The identifiable objects are temporal components of the centered corrective field. They are not semantic long-term, short-term, or episodic user-preference classes.

\begin{definition}[Temporal residual trace]
\label{def:temporal-trace-classes}
For each coordinate $k\in\{1,\ldots,K\}$, define
\begin{equation}
\tilde{\mathbf d}_k
=
\left[
\tilde d_k^{(t_1)},
\ldots,
\tilde d_k^{(t_i)}
\right]^\top
\in
\mathbb{R}^{i}.
\label{eq:theory_temporal_trace}
\end{equation}
Let $\ell_j=i-j$ and define the normalized temporal atoms
\[
\psi_{L,k}
=
\frac{
\kappa_k^L
}{
\|\kappa_k^L\|_2
},
\qquad
\psi_{S,k}
=
\frac{
\kappa_k^S
}{
\|\kappa_k^S\|_2
},
\qquad
\psi_{E,k}
=
\frac{
\kappa_k^E
}{
\|\kappa_k^E\|_2
},
\]
with
\[
\kappa_{k,j}^L
=
\exp(-\beta_k\ell_j),
\qquad
\kappa_{k,j}^S
=
\max
\left(
1-\frac{\ell_j}{i},
0
\right),
\]
and
\[
\kappa_{k,j}^E
=
\exp
\left(
\ell_j\log\xi_k
-
\xi_k
-
\log\Gamma(\ell_j+1)
\right).
\]
A mixed temporal residual trace is a decomposition
\begin{equation}
\tilde{\mathbf d}_k
=
a_{L,k}\psi_{L,k}
+
a_{S,k}\psi_{S,k}
+
a_{E,k}\psi_{E,k}
+
\varepsilon_k.
\label{eq:theory_temporal_decomposition}
\end{equation}
\end{definition}

\begin{theorem}[Identifiability of temporal residual components]
\label{thm:kernel-trace-recoverability}
Let
\[
\Psi_k
=
[
\psi_{L,k},
\psi_{S,k},
\psi_{E,k}
]
\in
\mathbb{R}^{i\times3},
\qquad
G_k
=
\Psi_k^\top\Psi_k.
\]
Assume $\lambda_{\min}(G_k) > 0.$ For
\[
\tilde{\mathbf d}_k
=
\Psi_k a_k
+
\varepsilon_k,
\qquad
a_k
=
[
a_{L,k},
a_{S,k},
a_{E,k}
]^\top,
\]
let
\[
\widehat a_k
=
\Psi_k^\dagger
\tilde{\mathbf d}_k
=
G_k^{-1}
\Psi_k^\top
\tilde{\mathbf d}_k.
\]
Then
\begin{equation}
\left\|
\widehat a_k-a_k
\right\|_2
\leq
\frac{
1
}{
\sqrt{
\lambda_{\min}(G_k)
}
}
\|
\varepsilon_k
\|_2.
\label{eq:theory_temporal_identifiability}
\end{equation}
If $\varepsilon_k=0$, then $\widehat a_k=a_k.$
\end{theorem}

\begin{proof}
Substituting Equation~\ref{eq:theory_temporal_decomposition},
\[
\widehat a_k
=
\Psi_k^\dagger
\left(
\Psi_k a_k+\varepsilon_k
\right)
=
a_k
+
\Psi_k^\dagger
\varepsilon_k,
\]
because full column rank implies $\Psi_k^\dagger\Psi_k=I_3$. Hence
\[
\left\|
\widehat a_k-a_k
\right\|_2
\leq
\|
\Psi_k^\dagger
\|_2
\|
\varepsilon_k
\|_2.
\]
The singular values of $\Psi_k$ are the square roots of the eigenvalues of $G_k$, so
\[
\|
\Psi_k^\dagger
\|_2
=
\frac{
1
}{
\sigma_{\min}(\Psi_k)
}
=
\frac{
1
}{
\sqrt{
\lambda_{\min}(G_k)
}
}.
\]
This proves Equation~\ref{eq:theory_temporal_identifiability}. The zero-residual case follows immediately.
\end{proof}

\begin{corollary}[Coordinate-wise identification of the temporal field]
\label{cor:full-field-kernel-recovery}
If the assumptions of Theorem~\ref{thm:kernel-trace-recoverability} hold for every corrective coordinate, then
\begin{equation}
\sum_{k=1}^{K}
\left\|
\widehat a_k-a_k
\right\|_2^2
\leq
\sum_{k=1}^{K}
\frac{
\|
\varepsilon_k
\|_2^2
}{
\lambda_{\min}(G_k)
}.
\label{eq:theory_full_temporal_bound}
\end{equation}
\end{corollary}

\begin{proof}
Square Equation~\ref{eq:theory_temporal_identifiability} independently for each $k$ and sum over the $K$ coordinates.
\end{proof}

\begin{corollary}[Fixed short-support envelope]
\label{cor:fixed-short-envelope}
The short-support coefficient remains identifiable without a separately trainable decay parameter inside $\kappa^S$, provided $\lambda_{\min}(G_k)>0.$ Thus identifiability requires the short-support atom to remain distinguishable from the long-support and episodic-support atoms, not that every temporal atom contain its own learnable decay.
\end{corollary}

\begin{proof}
Theorem~\ref{thm:kernel-trace-recoverability} requires only full column rank of $\Psi_k$. The parameterization used to generate an individual dictionary column is irrelevant once the resulting column is fixed.
\end{proof}

\begin{corollary}[Finite-prefix expressivity of the temporal dictionary]
\label{cor:classical-timeseries-approximation}
For prefix length $i$, define the multi-coordinate dictionary
\[
\Phi
=
[
\psi_{L,1},
\ldots,
\psi_{L,K},
\psi_{S,1},
\ldots,
\psi_{S,K},
\psi_{E,1},
\ldots,
\psi_{E,K}
]
\in
\mathbb{R}^{i\times3K}.
\]
For any centered sampled trace $\mathbf y\in\mathbb{R}^{i}$, let $\widehat{\mathbf y} = \Phi\Phi^\dagger\mathbf y.$ If $\mathbf y = \Phi b+\zeta$ for some $b\in\mathbb{R}^{3K}$, then
\begin{equation}
\left\|
\mathbf y-\widehat{\mathbf y}
\right\|_2
\leq
\|\zeta\|_2.
\label{eq:theory_temporal_projection}
\end{equation}

Moreover, if at least $i$ episodic atoms have distinct positive parameters $\xi_1,\ldots,\xi_i>0, \qquad \xi_p\neq\xi_q \quad (p\neq q),$ then the corresponding episodic subdictionary spans $\mathbb{R}^{i}$.
\end{corollary}

\begin{proof}
Because $\Phi\Phi^\dagger\mathbf y$ is the orthogonal projection of $\mathbf y$ onto ${\rm span}(\Phi)$, it minimizes Euclidean projection error over that span. Taking the candidate $\Phi b$ gives Equation~\ref{eq:theory_temporal_projection}.

For the second statement, consider the unnormalized episodic matrix
\[
E_{j,q}
=
\frac{
e^{-\xi_q}
\xi_q^{\ell_j}
}{
\Gamma(\ell_j+1)
}.
\]
Define
\[
D_r
=
{\rm diag}
\left(
\Gamma(\ell_1+1),
\ldots,
\Gamma(\ell_i+1)
\right)
\]
and $D_c = {\rm diag} \left( e^{\xi_1}, \ldots, e^{\xi_i} \right).$ Then $(D_rED_c)_{j,q} = \xi_q^{\ell_j}.$ Because $\ell_j$ ranges over $0,\ldots,i-1$, this is a row-permuted Vandermonde matrix. Therefore $\det(D_rED_c) = \pm \prod_{1\leq p<q\leq i} (\xi_q-\xi_p) \neq0$ for distinct $\xi_q$. Since the diagonal scaling matrices are invertible, $E$ has rank $i$.
\end{proof}

\paragraph{Boundary Of The Finite-Prefix Result.}
Corollary~\ref{cor:classical-timeseries-approximation} is an expressivity statement, not an identifiability or generalization guarantee. It does not assert that the learned episodic parameters are distinct, that training reaches the spanning configuration, or that a finite-prefix representation extrapolates the same temporal pattern beyond the observed window. Its role is only to show that the $L/S/E$ terminology does not restrict the dictionary to three literal behavioral categories.

\paragraph{Connection To The Methodology And RQ2.}
Theorem~\ref{thm:kernel-trace-recoverability} concerns the temporal envelopes applied to the centered state-relative corrective field before the learned branch-mixing matrices $W_L,W_S,W_E$. It establishes when those basis coefficients are identifiable in principle. It does not state that end-to-end gradient training computes the least-squares estimator used in the proof. The controlled temporal experiments test the corresponding structural behavior after host compression. The preferred single pathway varies by host, while the combined $L+S+E$ pathway is strongest in eleven of twelve controlled host and trace settings. TiSASRec and NAML recover several injected long- and short-support patterns through the corresponding single branches, whereas Mamba4Rec favors short-support correction across all three injected trace families. These results are consistent with the theoretical object being the post-compression residual rather than the raw temporal label.

Selective retention is analytically separate from identifiability. Theorem~\ref{thm:kernel-trace-recoverability} does not imply that every identifiable temporal component should be written back. The RQ2 selective-retention experiment therefore tests the subsequent methodological operation: at the full history length, selective write-back improves all four representative hosts while retaining broad support across the prefix.

\subsection{Backpropagation Through Temporal Recovery Kernels}
\label{app:kernel-backprop}

Identifiability concerns uniqueness of temporal coefficients for a fixed dictionary, while the preceding expressivity result concerns which traces the dictionary can represent. We next follow the chain rule to identify the signal used to update the trainable temporal parameters. A parameter receives credit only through the residual activation and the loss gradient reaching that activation. The straight-through selector supplies the surrogate gradient described in Appendix~\ref{app:selector}.

For timestep $t_j$, define
\[
x_\rho^{(t_j)}
=
\kappa^\rho(\ell_j)
\odot
\tilde d^{(t_j)},
\qquad
\rho\in\{L,S,E\},
\]
so that
\[
\eta^{(t_j)}
=
\sum_{\rho\in\{L,S,E\}}
W_\rho
x_\rho^{(t_j)}.
\]
Let
\begin{equation}
g^{(t_j)}
=
\nabla_{\eta^{(t_j)}}
\mathcal L_{\rm repair}^{\rm enc}
\in
\mathbb{R}^{K}
\label{eq:theory_temporal_gradient}
\end{equation}
denote the gradient reaching the temporally resolved field after backpropagation through selective retention, aggregation, residual fusion, and the complete encoder-host repair objective. Thus $g^{(t_j)}$ contains credit from the task-utility, timestep-provenance, and source-coherence loss families defined in Appendix~\ref{app:training-details}. For the hard selector, this denotes the straight-through surrogate gradient used by the implementation.

The positive temporal parameters are
\[
\beta_k
=
{\rm softplus}(b_k)
+
\epsilon_0,
\qquad
\xi_k
=
{\rm softplus}(s_k)
+
\epsilon_0,
\qquad
\epsilon_0=10^{-4}.
\]

\begin{proposition}[Temporal credit assignment]
\label{prop:temporal-credit}
For coordinate $k$,
\begin{equation}
\frac{
\partial
\mathcal L_{\rm repair}^{\rm enc}
}{
\partial b_k
}
=
-
\sigma(b_k)
\sum_{j=1}^{i}
\ell_j
\kappa_k^L(\ell_j)
\tilde d_k^{(t_j)}
\left[
W_L^\top
g^{(t_j)}
\right]_k.
\label{eq:theory_long_gradient}
\end{equation}
For the episodic-support parameter,
\begin{equation}
\frac{
\partial
\mathcal L_{\rm repair}^{\rm enc}
}{
\partial s_k
}
=
\sigma(s_k)
\sum_{j=1}^{i}
\left(
\frac{\ell_j}{\xi_k}
-
1
\right)
\kappa_k^E(\ell_j)
\tilde d_k^{(t_j)}
\left[
W_E^\top
g^{(t_j)}
\right]_k.
\label{eq:theory_episodic_gradient}
\end{equation}
For each branch $\rho\in\{L,S,E\}$,
\begin{equation}
\frac{
\partial
\mathcal L_{\rm repair}^{\rm enc}
}{
\partial W_\rho
}
=
\sum_{j=1}^{i}
g^{(t_j)}
\left(
x_\rho^{(t_j)}
\right)^\top.
\label{eq:theory_branch_gradient}
\end{equation}
\end{proposition}

\begin{proof}
For $\kappa_k^L(\ell_j) = \exp(-\beta_k\ell_j),$ we have
\[
\frac{
\partial
\kappa_k^L(\ell_j)
}{
\partial\beta_k
}
=
-
\ell_j
\kappa_k^L(\ell_j)
\]
and
\[
\frac{
\partial\beta_k
}{
\partial b_k
}
=
\sigma(b_k).
\]
The additive constant $\epsilon_0$ has zero derivative. Applying the chain rule through
\[
\eta^{(t_j)}
=
\sum_\rho
W_\rho x_\rho^{(t_j)}
\]
gives Equation~\ref{eq:theory_long_gradient}.

For
\[
\kappa_k^E(\ell_j)
=
\exp
\left(
\ell_j\log\xi_k
-
\xi_k
-
\log\Gamma(\ell_j+1)
\right),
\]
\[
\frac{
\partial
\kappa_k^E(\ell_j)
}{
\partial\xi_k
}
=
\left(
\frac{\ell_j}{\xi_k}
-
1
\right)
\kappa_k^E(\ell_j)
\]
and
\[
\frac{
\partial\xi_k
}{
\partial s_k
}
=
\sigma(s_k),
\]
which yields Equation~\ref{eq:theory_episodic_gradient}. Equation~\ref{eq:theory_branch_gradient} follows from linearity of $\eta^{(t_j)}$ in each $W_\rho$.
\end{proof}

\paragraph{Temporal Credit.}
The product
\[
\tilde d_k^{(t_j)}
\left[
W_\rho^\top
g^{(t_j)}
\right]_k
\]
couples the presence of state-relative residual evidence with the gradient assigned by the complete repair objective. The temporal update is therefore conditioned on both the host-specific post-compression residual and the supervision reaching that residual through the repaired-state pathway. For $\kappa^L$, the additional lag factor $\ell_j$ modulates how distant evidence changes the learned decay. For $\kappa^E$, the factor
\[
\frac{\ell_j}{\xi_k}-1
\]
changes sign around the current episodic center. The short-support envelope contains no learned decay parameter. Its contribution remains trainable through $W_S$, the state-conditioned gate, the selector, and the aggregation operator.

\begin{theorem}[Bounded local temporal-kernel gradients]
\label{thm:kernel-gradient-stability}
Assume that in a local training region $\| \tilde d^{(t_j)} \|_2 \leq D, \qquad \| g^{(t_j)} \|_2 \leq G,$ and $\|W_L\|_2 \leq M_L, \qquad \|W_E\|_2 \leq M_E.$ Assume $0 \leq \ell_j \leq H-1$ and $\xi_k \geq \xi_{\min} > 0.$ Then
\begin{equation}
\left|
\frac{
\partial
\mathcal L_{\rm repair}^{\rm enc}
}{
\partial b_k
}
\right|
\leq
iDGM_L(H-1),
\label{eq:theory_long_gradient_bound}
\end{equation}
and, whenever $\kappa_k^E(\ell_j)\leq1$,
\begin{equation}
\left|
\frac{
\partial
\mathcal L_{\rm repair}^{\rm enc}
}{
\partial s_k
}
\right|
\leq
iDGM_E
\left(
1+
\frac{H-1}{\xi_{\min}}
\right).
\label{eq:theory_episodic_gradient_bound}
\end{equation}
If every temporal kernel is coordinate-wise bounded by one, then
\begin{equation}
\left\|
\frac{
\partial
\mathcal L_{\rm repair}^{\rm enc}
}{
\partial W_\rho
}
\right\|_F
\leq
iGD,
\qquad
\rho\in\{L,S,E\}.
\label{eq:theory_branch_gradient_bound}
\end{equation}
\end{theorem}

\begin{proof}
For the long-support branch, $0< \sigma(b_k) \leq1, \qquad 0< \kappa_k^L(\ell_j) \leq1,$ and
\[
\left|
[
W_L^\top
g^{(t_j)}
]_k
\right|
\leq
\|W_L\|_2
\|
g^{(t_j)}
\|_2
\leq
M_LG.
\]
Equation~\ref{eq:theory_long_gradient} therefore gives
\[
\left|
\frac{
\partial
\mathcal L_{\rm repair}^{\rm enc}
}{
\partial b_k
}
\right|
\leq
\sum_{j=1}^{i}
(H-1)DM_LG,
\]
which proves Equation~\ref{eq:theory_long_gradient_bound}.

For the episodic branch,
\[
\left|
\frac{\ell_j}{\xi_k}-1
\right|
\leq
1+
\frac{H-1}{\xi_{\min}}.
\]
Together with $\kappa_k^E(\ell_j)\leq1$ and the same operator-norm argument, this gives Equation~\ref{eq:theory_episodic_gradient_bound}.

Finally,
\[
\left\|
g^{(t_j)}
\left(
x_\rho^{(t_j)}
\right)^\top
\right\|_F
=
\|
g^{(t_j)}
\|_2
\|
x_\rho^{(t_j)}
\|_2.
\]
Coordinate-wise kernel boundedness gives $\| x_\rho^{(t_j)} \|_2 \leq \| \tilde d^{(t_j)} \|_2 \leq D.$ Summing over the $i$ valid timesteps proves Equation~\ref{eq:theory_branch_gradient_bound}.
\end{proof}

\begin{corollary}[Host-conditioned temporal learning]
\label{cor:host-conditioned-temporal-learning}
Two frozen hosts evaluated on the same raw interaction trajectory need not induce the same temporal-parameter gradients. Proposition~\ref{prop:temporal-credit} depends on both $\tilde d^{(t_j)}$ and $g^{(t_j)}$. The first depends on the centered host-specific sequence coordinates. The common projected-state offset cancels as shown in Equation~\ref{eq:app_centering_cancellation}. The second depends on the complete repair objective after propagation through that host's repaired-state interface and on the host representations used by the auxiliary supervision. Consequently, the learned temporal mapping is conditioned on post-compression corrective structure and need not reproduce the raw trajectory temporality.
\end{corollary}

\paragraph{Connection To The Methodology And RQ2.}
Corollary~\ref{cor:host-conditioned-temporal-learning} provides the optimization-side counterpart to Theorem~\ref{thm:kernel-trace-recoverability}. The controlled trace construction holds the injected temporal pattern fixed across the relevant comparisons, while the preferred repair branch varies with the host. Mamba4Rec favors short-support correction across L-Tr, S-Tr, and E-Tr, whereas TiSASRec, NAML, and EBNR exhibit different host and trace responses. The combined pathway being strongest in eleven of twelve settings is also compatible with Equation~\ref{eq:theory_branch_gradient}: $W_L$, $W_S$, and $W_E$ receive simultaneous gradient credit and can jointly explain a mixed post-compression residual. These results do not verify the bounded-gradient assumptions of Theorem~\ref{thm:kernel-gradient-stability}. That theorem only characterizes local update magnitudes when those assumptions hold.

\subsection{Local Control Induced By The Repair Loss Families}
\label{app:loss_recoverability}

Task accuracy alone need not distinguish every error in a corrective representation. A coordinate can predict the target while losing its timestep or source association. The following local result asks what additional control becomes available if the three supervision families bound corresponding components of coefficient error. The stable-decomposition and quadratic-control assumptions are stated explicitly in the theorem and are not established by the empirical loss ablation.

The encoder implementation groups its primitive supervisory terms into three functional families: task utility, timestep provenance, and source coherence. Their implementation-level decomposition is given in Appendix~\ref{app:training-details}. The following result operates at this family level. It does not require any individual family to consist of a single primitive loss, and it does not assign theoretical significance to the particular primitive coefficients used in the implementation.

Fix a prefix and let the oracle coefficient field $c^{(t_j)}$ be as in Theorem~\ref{thm:repair-recoverability}. Define
\[
\widehat c_\theta^{(t_j)}
=
\gamma_\theta^{(t_i)}
\odot
\Pi_{\nu,\theta}^{(t_i)}
\left(
\eta_\theta^{(t_j)}
\right)
\]
and $e_\theta^{(t_j)} = \widehat c_\theta^{(t_j)} - c^{(t_j)}.$ Using the oracle aggregation weights, form the weighted stacked error
\begin{equation}
\bar e_\theta
=
\left[
\sqrt{\alpha_1^{\star(t_i)}}e_\theta^{(t_1)};
\ldots;
\sqrt{\alpha_i^{\star(t_i)}}e_\theta^{(t_i)}
\right],
\label{eq:theory_weighted_coefficient_error}
\end{equation}
so that
\begin{equation}
\mathcal E_{\rm rec}(\theta)
=
\|
\bar e_\theta
\|_2^2
=
\sum_{j=1}^{i}
\alpha_j^{\star(t_i)}
\|
e_\theta^{(t_j)}
\|_2^2.
\label{eq:theory_recovery_error}
\end{equation}

\begin{theorem}[Local control of coefficient-level recoverability error]
\label{thm:loss-design-recoverability}

Let
\[
\mathcal L_{\rm repair}^{\rm enc}
=
\lambda_{\rm task}\mathcal L_{\rm task}^{\rm enc}
+
\lambda_{\rm pos}\mathcal L_{\rm pos}^{\rm enc}
+
\lambda_{\rm coh}\mathcal L_{\rm coh}^{\rm enc}
\]
denote the weighted combination of the three functional loss families defined by the encoder implementation, where $ \lambda_{\rm task}, \lambda_{\rm pos}, \lambda_{\rm coh}>0 $ are the family-level weights.

Assume that, in a neighborhood of a recoverable solution $\theta^\star$, the coefficient-level recovery error admits a stable decomposition
\begin{equation}
\bar e_\theta
=
e_{\rm task}(\theta)
+
e_{\rm pos}(\theta)
+
e_{\rm coh}(\theta)
\label{eq:theory_loss_decomposition}
\end{equation}
satisfying
\begin{equation}
\|
\bar e_\theta
\|_2^2
\leq
C_{\rm dec}
\left[
\|
e_{\rm task}(\theta)
\|_2^2
+
\|
e_{\rm pos}(\theta)
\|_2^2
+
\|
e_{\rm coh}(\theta)
\|_2^2
\right]
\label{eq:theory_stable_decomposition}
\end{equation}
for some $C_{\rm dec}\geq1$.

Assume further that the three family-level loss gaps locally control their corresponding error components:
\[
\mathcal L_{\rm task}^{\rm enc}(\theta)
-
\mathcal L_{\rm task}^{\rm enc}(\theta^\star)
\geq
\frac{\mu_{\rm task}}{2}
\|
e_{\rm task}(\theta)
\|_2^2,
\]
\[
\mathcal L_{\rm pos}^{\rm enc}(\theta)
-
\mathcal L_{\rm pos}^{\rm enc}(\theta^\star)
\geq
\frac{\mu_{\rm pos}}{2}
\|
e_{\rm pos}(\theta)
\|_2^2,
\]
and
\[
\mathcal L_{\rm coh}^{\rm enc}(\theta)
-
\mathcal L_{\rm coh}^{\rm enc}(\theta^\star)
\geq
\frac{\mu_{\rm coh}}{2}
\|
e_{\rm coh}(\theta)
\|_2^2,
\]
with positive local constants $\mu_{\rm task},\mu_{\rm pos},\mu_{\rm coh}$.

Define
\[
\mu_{\lambda}
=
\min
\left\{
\lambda_{\rm task}\mu_{\rm task},
\lambda_{\rm pos}\mu_{\rm pos},
\lambda_{\rm coh}\mu_{\rm coh}
\right\}.
\]
Then, for $ \mathcal E_{\rm rec}(\theta)=\|\bar e_\theta\|_2^2 $,
\begin{equation}
\mathcal E_{\rm rec}(\theta)
\leq
\frac{2C_{\rm dec}}{\mu_{\lambda}}
\left[
\mathcal L_{\rm repair}^{\rm enc}(\theta)
-
\mathcal L_{\rm repair}^{\rm enc}(\theta^\star)
\right].
\label{eq:theory_loss_recovery_bound}
\end{equation}
\end{theorem}

\begin{proof}
Multiplying the three family-level control inequalities by their corresponding positive family weights and summing gives
\[
\begin{split}
& \mathcal L_{\rm repair}^{\rm enc}(\theta) - \mathcal L_{\rm repair}^{\rm enc}(\theta^\star) \\
& \quad\geq \frac{1}{2} \Big[ \lambda_{\rm task}\mu_{\rm task} \|e_{\rm task}(\theta)\|_2^2 + \lambda_{\rm pos}\mu_{\rm pos} \|e_{\rm pos}(\theta)\|_2^2 + \lambda_{\rm coh}\mu_{\rm coh} \|e_{\rm coh}(\theta)\|_2^2 \Big].
\end{split}
\]
By the definition
\[
\mu_{\lambda}
=
\min
\left\{
\lambda_{\rm task}\mu_{\rm task},
\lambda_{\rm pos}\mu_{\rm pos},
\lambda_{\rm coh}\mu_{\rm coh}
\right\},
\]
it follows that
\[
\begin{split}
& \mathcal L_{\rm repair}^{\rm enc}(\theta) - \mathcal L_{\rm repair}^{\rm enc}(\theta^\star) \geq \frac{\mu_{\lambda}}{2} \Big[ \|e_{\rm task}(\theta)\|_2^2 + \|e_{\rm pos}(\theta)\|_2^2 \\
& + \|e_{\rm coh}(\theta)\|_2^2 \Big].
\end{split}
\]
Applying the stable-decomposition condition in Equation~\ref{eq:theory_stable_decomposition} therefore yields
\[
\|\bar e_\theta\|_2^2
\leq
\frac{2C_{\rm dec}}{\mu_{\lambda}}
\left[
\mathcal L_{\rm repair}^{\rm enc}(\theta)
-
\mathcal L_{\rm repair}^{\rm enc}(\theta^\star)
\right].
\]
Since $ \mathcal E_{\rm rec}(\theta)=\|\bar e_\theta\|_2^2 $, this is exactly Equation~\ref{eq:theory_loss_recovery_bound}.
\end{proof}

\begin{corollary}[Coverage by the three loss families]
\label{cor:loss-ordering}
Under the assumptions of Theorem~\ref{thm:loss-design-recoverability}, positive family-level weights on task utility, timestep provenance, and source coherence provide local control through the corresponding error components in the stable decomposition. If one family is removed by setting its family-level weight to zero, the theorem no longer provides finite control through that component unless the remaining families independently control it under an alternative valid decomposition. The result therefore motivates complementary supervision across the three families without implying that every family must improve empirical task performance, that their effects are independently causal, or that the chosen family-level or within-family coefficients are globally optimal.
\end{corollary}

\paragraph{Connection To The Methodology And Loss Ablation.}
Theorem~\ref{thm:loss-design-recoverability} applies to the weighted combination of the three functional loss families, not to individual primitive losses. In the encoder implementation, the task-utility family supervises the repaired state and aggregated correction, the timestep-provenance family uses relative-position supervision, and the source-coherence family uses action-, interaction-, and aggregate-level constraints. The outer family weights $\lambda_{\mathrm{task}}$, $\lambda_{\mathrm{pos}}$, and $\lambda_{\mathrm{coh}}$ determine the contribution of these families to the complete objective, while the within-family coefficients $\omega_{\cdot}$ determine their implementation-level decomposition, as specified in Appendix~\ref{app:repair-objective-implementation}.

The RQ2 objective ablation evaluates the empirical behavior of these supervision families through progressively enriched objectives. Task-utility supervision already improves the evaluated frozen hosts. Adding timestep-provenance supervision further improves the intermediate configuration, and adding source-coherence supervision gives the strongest configuration across all four reported ablation hosts. This ordering is consistent with complementary family-level supervision under the complete objective. Because the ablation is progressive rather than factorial, it does not independently identify the causal contribution of any single family. The experiment also does not verify the local decomposition or quadratic-control assumptions of Theorem~\ref{thm:loss-design-recoverability}, and the theorem provides no global convergence guarantee for \repair{}.

\paragraph{Scope And Boundaries Of The Theory.}
The learned read-space energy in Equation~\ref{eq:learned-read-energy} complements this approximation analysis by characterizing the residual energy expressed through the read coordinates learned by \repair{}. Theorem~\ref{thm:repair-recoverability} gives a conditional bound for the complete selective write-back, aggregation, and fusion pathway. Theorem~\ref{thm:kernel-trace-recoverability} characterizes identifiability of the finite-prefix temporal basis, while Proposition~\ref{prop:temporal-credit} and Theorem~\ref{thm:kernel-gradient-stability} characterize the gradient signal received by its trainable temporal parameters. Theorem~\ref{thm:loss-design-recoverability} gives a local control result for the weighted combination of the three functional training-objective families.

These statements establish neither that task-relevant information is absent from the original task-facing preference state nor that all such evidence survives the host computation. They do not prove that the learned repair basis equals the optimal spectral basis, that the temporal components correspond to semantic user-preference classes, that the stated local assumptions hold globally during optimization, or that a repaired state must improve every downstream task interface. RQ1 tests whether cached per-timestep representations can support useful correction while the encoder and original task interface remain unchanged. RQ2 separately evaluates corrective-space capacity, temporal structure, selective retention, and the empirical behavior of the supervision families. RQ3 lies outside the principal theoretical guarantee: its correction-setting comparisons and task-specific ranking and generation configurations characterize downstream utilization after state correction rather than establishing causal effects of intervention locus or task interface. The prompt-conditioned LLM experiments likewise provide a separate structured-event instantiation and are not evidence for the strict same-forward-pass claim established for encoder hosts.

\section{Evaluation Details}
\label{app:evaluation-details}

We specify what is held fixed, how histories and targets are formed, and which parameters are trained in each condition. A recommendation instance consists of one available user-history prefix and its next target. A generation instance additionally contains a task-side source or query. Keeping these objects fixed across paired conditions makes the comparison about the correction method applied to the same task.

\subsection{Shared Evaluation Protocol}
\label{app:evaluation-protocol}

For encoder hosts, \repair{} receives the task-facing preference state and cached per-timestep representations produced by the frozen host computation at the cache boundary defined in Appendix~\ref{app:host-interface}. No additional sequence encoder reprocesses the raw interaction history during repair. Prompt-conditioned LLM hosts use the separate structured-interaction interface defined in Appendix~\ref{app:llm-hosts}. These experiments are therefore reported separately from the strict same-forward-pass encoder-host setting.

Train and test users are disjoint for every dataset. Within each dataset and host setting, paired experimental conditions use the same held-out users, history prefixes, supervised targets, and evaluation candidate sets. Recommendation is evaluated by sampled-candidate ranking with one ground-truth target and 499 negatives per test instance. Recommendation metrics are computed on their native $[0,1]$ scale and multiplied by $100$ only for reporting. Generation-quality tables explicitly state their display scale. For the full PENS generation table, native values are multiplied by $100$. In the full PerSEval table, base and repaired scores are displayed in units of $10^{-3}$ to preserve their reported precision. Relative gains are percentages and are invariant to either display scaling. Representative repeated-run settings report the arithmetic mean over three independent seeds. Extended settings without repeated-run uncertainty analysis are reported as point estimates under the same held-out evaluation protocol.

\subsection{UIG Construction And Host-Side Eventization}
\label{app:evaluation-uig}

\subsubsection{Common Interaction Representation}
\label{app:uig-common}
The User Interaction Graph (UIG) of \citet{peraugy-tmlr-2025} records typed content nodes and user actions. Appendix~\ref{app:uig-construction} defines its action vocabulary and ordered prefix $\tau_b^{t_{1:i}}$. Here, eventization means mapping each dataset's records into those action and content pairs. A click and a rejection, for example, refer to content through different action labels. The resulting prefix is converted into the input format required by each host. The graph itself adds no trainable history encoder.

\subsubsection{Encoder-Host Representation Interface}
\label{app:uig-host-interface}
The frozen host processes the prefix once and supplies the task-facing state $z_b^{(t_i)}$ together with the cached sequence $E_b^{(t_i)}=[e_b^{(t_1)},\ldots,e_b^{(t_i)}]$. The cache boundary is defined in Appendix~\ref{app:host-interface}. \repair{} changes the state passed to the task head using this sequence, producing $\tilde z_b^{(t_i)}$. No additional sequence encoder reconstructs these cached representations.

\subsubsection{Dataset-Specific Eventization}
\label{app:uig-dataset-eventization}

\textbf{Movie And Product Interaction Logs.}
For explicit-feedback data, interactions are ordered chronologically within each user. Ratings or dataset-specific preference signals satisfying the positive-feedback criterion are mapped to \texttt{pos}, while interactions used as negative evidence are mapped to \texttt{neg}. The next held-out interaction defines the prediction target. Dataset-specific preprocessing required by each host is applied before UIG eventization and is kept fixed across paired experimental conditions.

\textbf{PENS And MIND Interaction Logs.}
For impression-based news data,
\begin{equation}
\mathrm{clicked\ news}
\mapsto
\texttt{pos},
\qquad
\mathrm{non\text{-}clicked\ news}
\mapsto
\texttt{neg}.
\end{equation}
Interactions retain their observed temporal order. For recommendation, the next positive news item defines the prediction target. For personalized generation, the source news item or task query is supplied separately through the task-side interface, while the generation request is represented by \texttt{cmd}.

When response nodes are required during training-history construction, only training-side references or training-side surrogate responses are used. Held-out generation references are not inserted into training trajectories.

\textbf{Dialogue-Derived Interaction Logs.}
TG-ReDial does not expose the same explicit click, rating, or purchase sequence available in the primary recommendation datasets. We therefore derive interactions from dialogue turns, mentioned recommendation entities, recommendation attempts, and response stance.

For a dialogue-derived candidate interaction,
\begin{equation}
a^{(t_i)}
=
\begin{cases}
\texttt{pos},
&
\text{agreement, acceptance, or positive stance},\\
\texttt{neg},
&
\text{rejection, disagreement, or non-positive stance}.
\end{cases}
\end{equation}
The target speaker's subsequent recommendation choice or response defines the corresponding predictive or generative target. Because these actions are induced from dialogue structure rather than directly observed behavioral logs, TG-ReDial is treated separately as a discourse-derived stress test.

\subsubsection{Supervised Target Construction}
\label{app:uig-targets}

\textbf{Predictive Targets.}
For recommendation,
\begin{equation}
{r^\star}_u^{(t_{i+1})}
\equiv
c_u^{(t_{i+1})}.
\end{equation}
The downstream recommendation interface ranks this ground-truth target against the fixed candidate set associated with the corresponding held-out instance.

\textbf{Generative Targets.}
For personalized generation,
\begin{equation}
{r^\star}_u^{(t_{i+1})}
\in
\{
\mathrm{headline},
\mathrm{summary},
\mathrm{dialogue\ response}
\}.
\end{equation}
The source article, document, query, or dialogue context is supplied through the task-side interface rather than folded into the task-facing preference state.

For autoregressive evaluation, a generated output may enter only subsequent prefixes after it has been generated. Gold future responses are not inserted into an autoregressive test trajectory.

\subsubsection{Training-Only Trajectory Augmentation}
\label{app:uig-training-augmentation}

Trajectory augmentation, when used, is restricted to the training pool. Cross-trajectory segment substitution retains an initial portion of a target trajectory and replaces selected later segments using contiguous segments drawn from training trajectories. Validation and test trajectories are never used as substitution sources.

When substitution makes a training-side surrogate response locally inconsistent with its new context, the replacement response is selected by
\begin{equation}
\hat r^{(t_i)}
=
\arg\min_{s_p}
\sum_{q=1}^{k}
\exp
\left(
-\lambda
\operatorname{pos}(c_q)
\right)
\sigma
\left(
e_{s_p},
e_{c_q}
\right),
\end{equation}
where $\{c_q\}_{q=1}^{k}$ is the preceding training-side context window, $\operatorname{pos}(c_q)$ is its recency index, and $\sigma(\cdot,\cdot)$ denotes the embedding-space distance used by the preprocessing implementation. Here, $s_p$ ranges over candidate training-side surrogate responses, $e_{s_p}$ and $e_{c_q}$ are their content embeddings, $k$ is the context-window size, and $\lambda$ controls recency weighting. The distance $\sigma(\cdot,\cdot)$ in this preprocessing rule is distinct from the sigmoid used in the repair module. This augmentation is applied only during training.

\subsubsection{Evaluation Data Path}
\label{app:uig-evaluation-path}

For encoder hosts, the repair-time data path is
\[
\tau_b^{t_{1:i}}
\rightarrow
\text{frozen host}
\rightarrow
\left(
z_b^{(t_i)},
\{e_b^{(t_j)}\}_{j=1}^{i}
\right)
\rightarrow
\repair{}
\rightarrow
\tilde z_b^{(t_i)}
\rightarrow
\text{task interface}.
\]
Within a dataset and host comparison, the host encoder checkpoint, held-out trajectory, supervised target, and evaluation candidates are fixed across the Original Host, Head-only Continuation, and \repair{} conditions. The trainable component differs according to the staged protocol specified below.

\subsection{Dataset-Specific Evaluation Tasks}
\label{app:evaluation-datasets}

\textbf{Movie Recommendation.}
MovieLens is evaluated as sequential next-item recommendation. The held-out evaluation pool contains 15{,}000 instances. Each instance contains one ground-truth next movie and 499 negative candidates.

\textbf{News Recommendation.}
MIND and PENS are evaluated as next-news recommendation. Each dataset uses 20{,}000 held-out instances under the same sampled-candidate protocol with one positive target and 499 negatives.

\textbf{Product Recommendation.}
Amazon Reviews 2023 is evaluated as multimodal product recommendation using 2{,}000 held-out instances. SigLIP2, OpenCLIP, and DINOv2-E5 provide the frozen multimodal representations used by their respective host settings. Each test instance contains one positive product and 499 negative candidates.

\textbf{Personalized Headline Generation.}
PENS is additionally evaluated for personalized headline generation using 20{,}000 held-out instances. Each instance contains a user history, task-side news content, and a user-specific target headline. Standard generation metrics and PerSEval variants are computed from the same generated outputs.

\textbf{Conversational Recommendation And Response Generation.}
TG-ReDial is evaluated separately using dialogue-derived interactions. SASRec, GRU4Rec, and TextCNN are used for recommendation, while ReDial-style and Transformer sequence-to-sequence interfaces are used for response generation.

\textbf{Controlled Temporal Diagnostics.}
The L-Tr, S-Tr, and E-Tr conditions are evaluation-only trajectory controls. The trained host and \repair{} parameters are unchanged. Each controlled trace preserves the supervised target and is constructed from a target-conditioned mixture of same-user history material and target-conditioned candidate-pool interactions. L-Tr modifies support across approximately 60\% of the trajectory, S-Tr restricts support to the final 40\%, and E-Tr places a contiguous 3 to 4 timestep segment at a randomly selected valid position, with segment content drawn from the target-conditioned construction. Because the known target is used to construct these diagnostic histories, they test responses to controlled evidence placement and are not ordinary next-item inference conditions. Their results are reported separately in Table~\ref{tab:rq2_temporal_summary}.

\subsection{Training And Optimization Protocol}
\label{app:training-details}

The distinction between stages is central to interpretation. Stage 2 tests whether a changed state helps the unchanged task head. Stage 3 additionally allows that head to adapt to the changed state. The Head-only Continuation control, called Head-only FT in the main text, instead updates the original head while leaving the state unchanged. We list the repair architecture and fixed coefficients in Table~\ref{tab:repair_hyperparameters}.

The encoder-host experiments use a staged protocol that separates task-trained host formation, frozen-head state correction, and downstream task-interface alignment. Once the task-trained host checkpoint is established, the host encoder remains frozen in all subsequent stages.

\subsubsection{Stage 1: Task-Trained Host}
\label{app:stage1-host}

Each host is first trained under its original task supervision without \repair{}. MovieLens and Amazon recommendation hosts are trained for 20 epochs, news recommendation hosts for 5 epochs, and the lightweight PENS generation setting trains the preference encoder for 5 epochs followed by 2 epochs of decoder training. Host-specific preprocessing, history encoding, state formation, and task interfaces follow the corresponding base implementations.

The resulting task-trained checkpoint defines the \textbf{Original Host} condition and provides the common host initialization for subsequent conditions within the same dataset and host setting. After this checkpoint is established, the host encoder is frozen.

\subsubsection{Head-Only Continuation Control}
\label{app:head-only-continuation}

The \textbf{Head-only Continuation} control starts from the same Stage-1 checkpoint and leaves the host encoder and resulting task-facing preference state unchanged. \repair{} is absent, and cached per-timestep representations are not supplied to an auxiliary correction module. Only the original downstream task interface receives additional optimization under the original task supervision.

The continuation is optimized for five additional epochs. The host encoder and task-facing preference state remain fixed throughout, and only the original downstream task interface is updated. Optimization follows the corresponding repair-compatible optimizer settings. Validation performance is evaluated after each epoch, and the checkpoint with the strongest validation performance among the five continuation epochs is retained for evaluation. Thus, Head-only Continuation receives a distinct five-epoch optimization opportunity on the unchanged state rather than sharing the two-epoch downstream-alignment schedule used by \repair{} + Head Alignment.

\subsubsection{Stage 2: Frozen-Head State Correction}
\label{app:stage2-frozen-repair}

Stage 2 inserts \repair{} between the frozen host computation and the original downstream task interface. Both the host encoder and the original task head or decoder remain fixed. For each training prefix, the frozen host provides $\left( z_b^{(t_i)}, \{e_b^{(t_j)}\}_{j=1}^{i} \right),$ where $z_b^{(t_i)}$ is the task-facing preference state and $\{e_b^{(t_j)}\}_{j=1}^{i}$ are the cached per-timestep representations from the same frozen host computation. \repair{} maps these quantities to $\tilde z_b^{(t_i)}$ without re-encoding the raw history.

For MovieLens and Amazon recommendation, Stage 2 runs for 10 epochs. For PENS and MIND recommendation and the encoder-based personalized-generation settings, Stage 2 runs for 2 epochs. Gradients are restricted to \repair{}, its repair-side projection layers, and repair-side supervision heads. The frozen host encoder and original downstream task interface receive no parameter updates. Validation performance selects the retained Stage-2 checkpoint, which defines the \textbf{\repair{} (Frozen Head)} condition.

\subsubsection{Stage-2 Repair Objective}
\label{app:repair-objective-implementation}

Each loss family addresses a different ambiguity. Task utility asks whether the repaired state predicts the target. Timestep provenance asks where corrective evidence originated within the prefix. Source coherence asks whether a corrective representation remains associated with its source action and interaction. We give their primitive terms below so that the family-level objective in Section~\ref{sec:methodology} can be reconstructed. Its outer weights $\omega_{\mathrm{task}},\omega_{\mathrm{pos}},\omega_{\mathrm{coh}}$ are denoted here by $\lambda_{\mathrm{task}},\lambda_{\mathrm{pos}},\lambda_{\mathrm{coh}}$ to reserve $\omega$ for within-family coefficients.

\textbf{Objective Hierarchy.}
The Stage-2 repair objective contains three functional supervision families:
\begin{equation}
\mathcal{L}_{\mathrm{task}},
\qquad
\mathcal{L}_{\mathrm{pos}},
\qquad
\mathcal{L}_{\mathrm{coh}}.
\end{equation}
These labels specify supervisory roles rather than individual primitive losses. For encoder hosts, the complete Stage-2 objective is
\begin{equation}
\mathcal{L}_{\mathrm{repair}}^{\mathrm{enc}}
=
\lambda_{\mathrm{task}}
\mathcal{L}_{\mathrm{task}}^{\mathrm{enc}}
+
\lambda_{\mathrm{pos}}
\mathcal{L}_{\mathrm{pos}}^{\mathrm{enc}}
+
\lambda_{\mathrm{coh}}
\mathcal{L}_{\mathrm{coh}}^{\mathrm{enc}},
\label{eq:app_encoder_repair_objective}
\end{equation}
with
\begin{equation}
\lambda_{\mathrm{task}}=0.5,
\qquad
\lambda_{\mathrm{pos}}=0.2,
\qquad
\lambda_{\mathrm{coh}}=0.3.
\label{eq:app_family_weights}
\end{equation}
The coefficients $\lambda_{\mathrm{task}}$, $\lambda_{\mathrm{pos}}$, and $\lambda_{\mathrm{coh}}$ are family-level weights. Separate coefficients $\omega_{\cdot}$ determine the relative contribution of primitive supervisory terms within each family.

\textbf{Timestep-Level Corrective Representations.}
The same corrective evidence appears at three dimensions. The retained coordinate $\hat{\mathbf c}^{(t_j)}$ has dimension $K$, its write-space representation $\mathbf w^{(t_j)}$ has dimension $m$, and its projected state correction $\delta_b^{(t_j)}$ has dimension $d_z$. The final correction $\Delta_b^{(t_i)}$ pools contributions across positions. For notational clarity, define the retained, gated corrective coordinate at timestep $t_j$ as
\begin{equation}
\hat{\mathbf c}^{(t_j)}
=
\boldsymbol{\gamma}^{(t_i)}
\odot
\Pi_{\nu}^{(t_i)}
\left(
\boldsymbol{\eta}^{(t_j)}
\right)
\in
\mathbb{R}^{K},
\label{eq:app_retained_corrective_coordinate}
\end{equation}
and its write-space representation as
\begin{equation}
\mathbf w^{(t_j)}
=
\hat U_{\mathrm{rep}}
\Lambda_{\mathrm{rep}}
\hat{\mathbf c}^{(t_j)}
\in
\mathbb{R}^{m}.
\label{eq:app_timestep_write_representation}
\end{equation}
Its state-space write-back is
\begin{equation}
\delta_b^{(t_j)}
=
P_{\Delta}
\mathbf w^{(t_j)}.
\label{eq:app_timestep_state_correction}
\end{equation}

The aggregated write-space correction and final state-space correction are
\begin{equation}
\bar{\mathbf w}^{(t_i)}
=
\sum_{j=1}^{i}
\alpha_j^{(t_i)}
\mathbf w^{(t_j)},
\qquad
\Delta_b^{(t_i)}
=
P_{\Delta}
\bar{\mathbf w}^{(t_i)}.
\label{eq:app_aggregate_write_representation}
\end{equation}
This notation distinguishes timestep-level corrective representations from the aggregated correction $\Delta_b^{(t_i)}$ applied at the task-facing state.

\textbf{Task-Utility Family.}
The task-utility family constrains both the repaired state and the aggregated correction to retain predictive information about the supervised target. Let $h_{\mathrm{task}}$ denote the frozen downstream task interface during Stage 2. The repaired-state task loss is
\begin{equation}
\mathcal{L}_{\mathrm{state}}
=
\ell_{\mathrm{task}}
\left(
h_{\mathrm{task}}
\left(
\tilde z_b^{(t_i)}
\right),
r_u^{\star(t_{i+1})}
\right).
\label{eq:app_state_task_loss}
\end{equation}

The aggregated correction is additionally supervised through
\begin{equation}
\mathcal{L}_{\Delta}
=
\ell_{\mathrm{task}}
\left(
h_{\mathrm{task}}
\left(
\Delta_b^{(t_i)}
\right),
r_u^{\star(t_{i+1})}
\right).
\label{eq:app_delta_task_loss}
\end{equation}

Let $\mathcal{L}_{\mathrm{base}}$ denote the corresponding task loss obtained from the frozen original state. We use the one-sided reference constraint
\begin{equation}
\mathcal{L}_{\mathrm{better}}
=
\max
\left(
0,
\mathcal{L}_{\Delta}
-
\mathcal{L}_{\mathrm{base}}
\right),
\label{eq:app_better_loss}
\end{equation}
which penalizes an aggregated corrective representation whose target loss exceeds that of the frozen-host reference.

The task-utility family is
\begin{equation}
\mathcal{L}_{\mathrm{task}}^{\mathrm{enc}}
=
\omega_{\mathrm{state}}
\mathcal{L}_{\mathrm{state}}
+
\omega_{\Delta}
\mathcal{L}_{\Delta}
+
\omega_{\mathrm{better}}
\mathcal{L}_{\mathrm{better}},
\label{eq:app_task_family}
\end{equation}
with
\begin{equation}
\omega_{\mathrm{state}}=1.0,
\qquad
\omega_{\Delta}=1.0,
\qquad
\omega_{\mathrm{better}}=0.5.
\label{eq:app_task_family_weights}
\end{equation}

\textbf{Timestep-Provenance Family.}
The timestep-provenance family supervises relative temporal origin from the retained timestep-level corrective coordinates. For a valid timestep $t_j$, its normalized position within a prefix of length $i$ is assigned to one of $B_{\mathrm{pos}}=8$ relative-position bins. Let $b_j^{(t_i)}$ denote the bin and $h_{\mathrm{pos}}$ the corresponding classifier. Then
\begin{equation}
\mathcal{L}_{\mathrm{position}}
=
-\frac{1}{i}
\sum_{j=1}^{i}
\log
p
\left(
b_j^{(t_i)}
\mid
\hat{\mathbf c}^{(t_j)}
\right).
\label{eq:app_impl_position_loss}
\end{equation}
The timestep-provenance family is
\begin{equation}
\mathcal{L}_{\mathrm{pos}}^{\mathrm{enc}}
=
\omega_{\mathrm{position}}
\mathcal{L}_{\mathrm{position}},
\qquad
\omega_{\mathrm{position}}=0.5.
\label{eq:app_position_family}
\end{equation}
This objective preserves relative timestep provenance. It does not assign semantic long-, short-, or episodic labels to individual corrective coordinates.

\textbf{Source-Coherence Family.}
The source-coherence family preserves correspondence between corrective representations and the interaction evidence from which they are constructed. Encoder-host interaction units contain both an interaction identity and an action label, so the family contains action-level, interaction-level, and aggregate-level constraints.

First, the action associated with timestep $t_j$ is predicted from its retained corrective coordinate:
\begin{equation}
\mathcal{L}_{\mathrm{action}}
=
-\frac{1}{i}
\sum_{j=1}^{i}
\log
p
\left(
a^{(t_j)}
\mid
\hat{\mathbf c}^{(t_j)}
\right).
\label{eq:app_action_coherence}
\end{equation}

Second, timestep-level interaction correspondence is retained through a sampled contrastive objective. Let $E(v)$ denote a content representation in the coherence space and let $\mathcal{N}_j$ contain the sampled interaction negatives. Define $s_j(v)=\operatorname{sim}(P_{\mathrm{coh}}\mathbf w^{(t_j)},E(v))/\tau_{\mathrm{coh}}$, the temperature-scaled similarity between timestep $t_j$'s write candidate and content $v$. Then
\begin{equation}
\mathcal{L}_{\mathrm{interaction}}=-\frac{1}{i}\sum_{j=1}^{i}\log\frac{\exp s_j(v^{(t_j)})}{\exp s_j(v^{(t_j)})+\sum_{v^-\in\mathcal N_j}\exp s_j(v^-)}.
\label{eq:app_interaction_coherence}
\end{equation}

Third, the aggregated write-space correction is constrained by the target-side contrastive objective. Define $s_\Delta(r)=\operatorname{sim}(P_{\mathrm{coh}}\bar{\mathbf w}^{(t_i)},E(r))/\tau_{\mathrm{coh}}$ for the aggregate and candidate target $r$. With sampled target negatives $\mathcal N_\Delta$, the loss is
\begin{equation}
\mathcal{L}_{\Delta\mathrm{con}}=-\log\frac{\exp s_\Delta(r_u^{\star(t_{i+1})})}{\exp s_\Delta(r_u^{\star(t_{i+1})})+\sum_{r^-\in\mathcal N_\Delta}\exp s_\Delta(r^-)}.
\label{eq:app_delta_contrastive}
\end{equation}

The source-coherence family is
\begin{equation}
\mathcal{L}_{\mathrm{coh}}^{\mathrm{enc}}
=
\omega_{\mathrm{action}}
\mathcal{L}_{\mathrm{action}}
+
\omega_{\mathrm{interaction}}
\mathcal{L}_{\mathrm{interaction}}
+
\omega_{\Delta\mathrm{con}}
\mathcal{L}_{\Delta\mathrm{con}},
\label{eq:app_coherence_family}
\end{equation}
with
\begin{equation}
\omega_{\mathrm{action}}=0.5,
\qquad
\omega_{\mathrm{interaction}}=0.5,
\qquad
\omega_{\Delta\mathrm{con}}=0.2.
\label{eq:app_coherence_family_weights}
\end{equation}

\textbf{Complete Stage-2 Objective.}
The following expansion substitutes every primitive term into the three-family objective. The better-than-base term is a training penalty for a correction that predicts the target less well than the base state. It does not guarantee non-degradation on every held-out instance. Combining Equations~\ref{eq:app_task_family}, \ref{eq:app_position_family}, and \ref{eq:app_coherence_family}, the complete encoder-host Stage-2 objective is
\begin{equation}
\begin{split}
& \mathcal{L}_{\mathrm{repair}}^{\mathrm{enc}} ={} \lambda_{\mathrm{task}} \underbrace{ \left( \omega_{\mathrm{state}}\mathcal{L}_{\mathrm{state}} + \omega_{\Delta}\mathcal{L}_{\Delta} + \omega_{\mathrm{better}}\mathcal{L}_{\mathrm{better}} \right) }_{\mathcal{L}_{\mathrm{task}}^{\mathrm{enc}}} \\
& + \lambda_{\mathrm{pos}} \underbrace{ \left( \omega_{\mathrm{position}} \mathcal{L}_{\mathrm{position}} \right) }_{\mathcal{L}_{\mathrm{pos}}^{\mathrm{enc}}} \\
& + \lambda_{\mathrm{coh}} \underbrace{ \left( \omega_{\mathrm{action}}\mathcal{L}_{\mathrm{action}} + \omega_{\mathrm{interaction}}\mathcal{L}_{\mathrm{interaction}} + \omega_{\Delta\mathrm{con}}\mathcal{L}_{\Delta\mathrm{con}} \right) }_{\mathcal{L}_{\mathrm{coh}}^{\mathrm{enc}}}.
\end{split}
\label{eq:app_full_encoder_objective}
\end{equation}

The outer coefficients $\lambda_{\mathrm{task}}$, $\lambda_{\mathrm{pos}}$, and $\lambda_{\mathrm{coh}}$ therefore control the relative contribution of the three functional families, while the within-family coefficients $\omega_{\cdot}$ specify their primitive implementation-level decomposition.

\textbf{Training-Side Perturbations And Negative Sampling.}
These perturbations supervise source correspondence during training. They are distinct from the evaluation-only temporal controls and noisy-history stress test. Likewise, contrastive negatives used by the auxiliary losses are distinct from the $499$ candidate items used to evaluate recommendation. Stage-2 encoder training forms perturbed views only for auxiliary supervision, while task-utility losses use clean training prefixes. Item identities are replaced with probability $p_{\mathrm{item}}=0.15$, action labels with probability $p_{\mathrm{act}}=0.10$, and adjacent interactions are swapped with probability $p_{\mathrm{swap}}=0.05$. These perturbations do not alter $\mathcal{L}_{\mathrm{state}}$, $\mathcal{L}_{\Delta}$, or $\mathcal{L}_{\mathrm{better}}$ and are not evaluation controls. They are applied only within auxiliary branches. Validation and test histories remain unchanged.

The timestep-interaction and aggregate-level contrastive objectives use 64 sampled negatives per positive during Stage 2. These negatives provide training-time supervision and are distinct from the 499 negatives used in the final sampled-candidate ranking protocol.

\textbf{Recommendation Instantiation.}
For recommendation, $\mathcal{L}_{\mathrm{state}}$ and $\mathcal{L}_{\Delta}$ use categorical cross-entropy over the supervised target. interaction-level coherence uses the corresponding host item-representation space, while action coherence uses the normalized interaction label associated with each timestep.

\textbf{Generation Instantiation.}
For generation, the task-family terms use the corresponding sequence-generation objective under teacher forcing. The source document, query, or dialogue context enters through the original task-side interface, while the repaired preference state supplies the user-side conditioning signal. Evaluation uses autoregressive decoding. The provenance and coherence families retain the same functional roles, with interaction/content supervision instantiated in the representation space available to the corresponding generation host.

\subsubsection{Stage 3: Downstream Head Alignment}
\label{app:stage3-alignment}

Alignment changes the consumer of the repaired state as well as the repair module. Improvements in this stage therefore quantify their joint adaptation. The host encoder remains fixed, which preserves its history representations across the Stage-2 and Stage-3 conditions.

Stage 3 starts from the validation-selected Stage-2 repair checkpoint and keeps the task-trained host encoder frozen. Unlike Stage 2, both \repair{} and the downstream task interface are trainable during this stage. For recommendation, \repair{} and the recommendation head are jointly updated for 2 epochs. For generation, \repair{} and the decoder-side interface are jointly aligned for 2 epochs using only the original sequence-generation loss. The task-trained history encoder remains frozen.

\textbf{Alignment Objective.}
The Stage-3 recommendation objective is
\begin{equation}
\mathcal{L}_{\mathrm{align}}
=
\mathcal{L}_{\mathrm{task}}
+
\lambda_{\Delta\mathrm{L2}}
\mathcal{L}_{\Delta\mathrm{L2}}
+
\lambda_{\mathrm{smooth}}
\mathcal{L}_{\mathrm{smooth}},
\label{eq:app_alignment_objective}
\end{equation}
where
\begin{equation}
\mathcal{L}_{\Delta\mathrm{L2}}
=
\frac{1}{i}
\sum_{j=1}^{i}
\left\|
\delta_b^{(t_j)}
\right\|_2^2,
\label{eq:app_delta_l2}
\end{equation}
and
\begin{equation}
\mathcal{L}_{\mathrm{smooth}}
=
\frac{1}{i-1}
\sum_{j=2}^{i}
\left\|
\delta_b^{(t_j)}
-
\delta_b^{(t_{j-1})}
\right\|_2^2.
\label{eq:app_delta_smooth}
\end{equation}
We use
\[
\lambda_{\Delta\mathrm{L2}}=10^{-4},
\qquad
\lambda_{\mathrm{smooth}}=0.05.
\]
These regularizers are applied only during downstream alignment and are not additional assumptions about the theoretical repair space. The resulting validation-selected checkpoint defines the \textbf{\repair{} + Head Alignment} condition.

\subsubsection{Optimization And Checkpoint Selection}
\label{app:optimization-checkpoints}

\textbf{Optimization.}
Stage-2 \repair{} parameters are optimized with AdamW~\citep{loshchilov2019adamw} using learning rate $2\times10^{-4},$ weight decay $10^{-4},$ and maximum gradient norm $1.0.$

During Stage 3, \repair{} uses learning rate $5\times10^{-5},$ while the downstream task interface uses $2\times10^{-4}.$ Weight decay remains $10^{-4}$, and the joint gradient norm is clipped at $1.0$. Mixed-precision execution is enabled when supported by the CUDA environment. Optimizer state from Stage-1 host training is not reused for Stage 2. Head-only Continuation is optimized for five epochs using the corresponding repair-compatible optimization settings, with only the original downstream task interface trainable.

\textbf{Checkpoint Selection.}
A checkpoint is retained after task-trained host formation, frozen-head repair, and downstream alignment. Validation performance is used to select the checkpoint within each trainable stage. The Stage-2 checkpoint is retained independently when Stage 3 begins.

The evaluated conditions therefore correspond to
\[
\begin{array}{ll}
\text{Original Host}
&
\text{validation-selected Stage-1 task-trained checkpoint},
\\
\text{Head-only Continuation}
&
\text{best validation over five additional head-only epochs},
\\
\repair{}\text{ (Frozen Head)}
&
\text{validation-selected Stage-2 repair checkpoint},
\\
\repair{}\text{ + Head Alignment}
&
\text{validation-selected Stage-3 joint-alignment checkpoint}.
\end{array}
\]

\subsubsection{Randomness, Pairing, And Uncertainty}
\label{app:statistical-protocol}

\textbf{Random Seeds And Paired Evaluation.}
All stochastic training components are seeded before optimization. OpenCLIP and SigLIP2 on Amazon and EBNR on PENS have three independent training seeds. We report their arithmetic mean, standard deviation, and held-out-user bootstrap interval in Table~\ref{tab:rq1_rebuttal_stats}. These are the settings with explicit repeated-run uncertainty evidence. This scope is not extended to every host in the main comparison. Within each run, paired conditions use the same held-out users, history prefixes, supervised targets, and recommendation candidate sets. Extended settings without repeated-run uncertainty analysis are reported as point estimates under the same held-out evaluation protocol.

\textbf{User-Clustered Bootstrap Intervals.}
Uncertainty intervals are computed at the held-out-user level so that multiple prefixes belonging to the same user are not treated as independent observations. Metric contributions are first grouped by held-out user. We then draw 10{,}000 bootstrap samples by resampling users with replacement and retaining all observations associated with each sampled user. The 2.5th and 97.5th percentiles of the resulting bootstrap distribution define the corresponding 95\% interval.

\subsubsection{Execution Environment And Task-Specific Preprocessing}
\label{app:execution-preprocessing}

\textbf{Compute Resources.}
Experiments were executed on paid Lightning.ai GPU resources using NVIDIA A100, NVIDIA H100, and NVIDIA RTX PRO 5000-class GPUs according to task scale and availability. Latency and memory comparisons are made only within the same dataset and host setting under the same execution environment. Absolute measurements obtained on different GPU types are not used for cross-architecture conclusions.

\textbf{PENS Text Preprocessing.}
For the lightweight word-level PENS generation implementation, titles are truncated to 30 tokens and news bodies to 120 tokens. Concatenated source inputs are truncated to 160 tokens, and target headlines or summaries are truncated to 60 tokens. The vocabulary is constructed from training-side text using minimum frequency 2 and maximum size 100{,}000.

\begin{table}[!t]
\centering
\small
\caption{Architecture and fixed hyperparameters for the encoder-host repair interface. The dimensions describe the implemented common interface, not every host's native hidden width. The default $K=128$ is varied only in the rank sweep. Appendix~\ref{app:repair-objective-implementation} defines the loss families and within-family coefficients.}
\label{tab:repair_hyperparameters}
\setlength{\tabcolsep}{4.0pt}
\renewcommand{\arraystretch}{1.05}

\scalebox{0.84}{
\begin{tabular}{lll}
\toprule
\textbf{Component} & 
\textbf{Symbol / Configuration} & 
\textbf{Dimension / Value} \\
\midrule

\multicolumn{3}{l}{\textit{Core repair-interface dimensions}} \\

Repair-interface state dimension
& $d_z$
& $768$ \\

Repair-interface cached-representation dimension
& $d_e$
& $768$ \\

Shared repair space
& $m$
& $192$ \\

Repair rank
& $K$
& $128$ \\

Maximum history length
& $L$
& $50$ \\

Position bins
& $B_{\mathrm{pos}}$
& $8$ \\

\midrule
\multicolumn{3}{l}{\textit{State-relative corrective representation}} \\

State projection
& $P_z$
& $\mathbb{R}^{192\times768}$ \\

Cached-representation projection
& $P_e$
& $\mathbb{R}^{192\times768}$ \\

Corrective read basis
& $V_{\mathrm{rep}}$
& $\mathbb{R}^{192\times128}$ \\

Corrective write basis
& $U_{\mathrm{rep}}$
& $\mathbb{R}^{192\times128}$ \\

Transmission scales
& $\Lambda_{\mathrm{rep}}$
& $128\times128$ diagonal \\

Basis normalization
& $U_{\mathrm{rep}},V_{\mathrm{rep}}$
& column-normalized \\

Scale constraint
& $\lambda_k$
& $\operatorname{softplus}(\lambda_{k,\mathrm{raw}})$ \\

\midrule
\multicolumn{3}{l}{\textit{Temporal resolution}} \\

Long-support parameter
& $\beta$
& $\mathbb{R}_{+}^{128}$ \\

Episodic-support parameter
& $\xi$
& $\mathbb{R}_{+}^{128}$ \\

Positive transform
& $\beta,\xi$
& $\operatorname{softplus}(\cdot)+10^{-4}$ \\

Long-support mixing map
& $W_L$
& $\mathbb{R}^{128\times128}$ \\

Short-support mixing map
& $W_S$
& $\mathbb{R}^{128\times128}$ \\

Episodic-support mixing map
& $W_E$
& $\mathbb{R}^{128\times128}$ \\

\midrule
\multicolumn{3}{l}{\textit{Selective write-back}} \\

Selector statistics
& $s_k^{(t_i)}$
& $\mathbb{R}^{3}$ \\

Selector network
& $\mathrm{MLP}_{\mathrm{sel}}$
& $3\rightarrow32\rightarrow1$ \\

Retained fraction
& $f_k^{(t_i)}$
& $(0,1)$ \\

Selector mask
& $\pi^{(t_i)}$
& $i\times128$ \\

Straight-through temperature
& $\tau_{\mathrm{sel}}$
& $0.05$ \\

\midrule
\multicolumn{3}{l}{\textit{Aggregation and state correction}} \\

State-conditioned gate
& $\mathrm{MLP}_{\gamma}$
& $768\rightarrow128$ \\

Aggregation scorer
& $\mathrm{MLP}_{\alpha}$
& $256\rightarrow64\rightarrow1$ \\

Correction projection
& $P_{\Delta}$
& $\mathbb{R}^{768\times192}$ \\

State correction
& $\Delta_b^{(t_i)}$
& $\mathbb{R}^{768}$ \\

\midrule
\multicolumn{3}{l}{\textit{Residual-gated fusion}} \\

Fusion input
& $[z_b^{(t_i)};\Delta_b^{(t_i)}]$
& $\mathbb{R}^{1536}$ \\

Fusion gate
& $\mathrm{MLP}_{F}$
& $1536\rightarrow768$ \\

Probe network
& $\mathrm{MLP}_{\mathrm{probe}}$
& $1536\rightarrow512\rightarrow768$ \\

Fusion gate output
& $g_F^{(t_i)}$
& $(0,1)^{768}$ \\

Repaired state
& $\tilde z_b^{(t_i)}$
& $\mathbb{R}^{768}$ \\

\midrule
\multicolumn{3}{l}{\textit{Supervision heads and objective weights}} \\

Relative-position classifier
& $h_{\mathrm{pos}}$
& $128\rightarrow8$ \\

Coherence projection
& $P_{\mathrm{coh}}$
& $192\rightarrow128$ \\

Contrastive temperature
& $\tau_{\mathrm{coh}}$
& $0.07$ \\

Task-utility family weight
& $\lambda_{\mathrm{task}}$
& $0.5$ \\

Timestep-provenance family weight
& $\lambda_{\mathrm{pos}}$
& $0.2$ \\

Source-coherence family weight
& $\lambda_{\mathrm{coh}}$
& $0.3$ \\

Within-family task coefficients
& $(\omega_{\mathrm{state}},
   \omega_{\Delta},
   \omega_{\mathrm{better}})$
& $(1.0,1.0,0.5)$ \\

Within-family provenance coefficient
& $\omega_{\mathrm{position}}$
& $0.5$ \\

Within-family coherence coefficients
& $(\omega_{\mathrm{action}},
   \omega_{\mathrm{interaction}},
   \omega_{\Delta\mathrm{con}})$
& $(0.5,0.5,0.2)$ \\

\bottomrule
\end{tabular}
}
\end{table}


\section{Evaluation Metrics And Reporting}
\label{app:metrics}

Recommendation scores evaluate where a single target is ranked. Generation-quality scores compare an output with a reference, while personalization-sensitive scores additionally compare outputs intended for different users of the same source. These measurements answer different questions and are reported separately. The definitions below use native metric values before any display scaling.

\subsection{Reporting Conventions And Relative Gain}
\label{app:metric-reporting}

Recommendation metrics are computed on their native $[0,1]$ scale and multiplied by $100$ for presentation in the result tables. The PENS generation-quality table displays $100$ times each native score, while the full PerSEval table displays $1000$ times each native score. These display factors apply only to base and repaired values. Columns named Gain report relative percentages.

For cross-interface comparisons, normalized relative gain for metric $M$ is
\begin{equation}
\Delta_{\mathrm{rel}}(M)
=
100
\cdot
\frac{
M_{\mathrm{repair}}
-
M_{\mathrm{base}}
}{
M_{\mathrm{base}}
}.
\label{eq:relative_gain}
\end{equation}
The ratio is computed independently for each metric with a nonzero base score. An absolute difference on the displayed $0$ to $100$ scale is a score-point gain, whereas Equation~\ref{eq:relative_gain} gives a relative percentage. For example, moving from $20$ to $25$ is a gain of $5$ points and $25\%$.

\subsection{Recommendation Metrics}
\label{app:metrics-predictive}

Higher values are better for all recommendation metrics below. The cutoff $K$ in a metric such as Hit@$K$ denotes the number of ranked candidates inspected and is unrelated to the repair dimension $K$. With one relevant target, the ideal discounted cumulative gain is one for every positive cutoff.

For test instance $n$, let $r_n^\star$ denote the rank of the single ground-truth target among the 500 evaluation candidates, and let $N$ denote the number of test instances.

\textbf{Mean Reciprocal Rank.}
\begin{equation}
\mathrm{MRR}
=
\frac{1}{N}
\sum_{n=1}^{N}
\frac{1}{r_n^\star}.
\end{equation}

\textbf{Normalized Discounted Cumulative Gain.}
For binary relevance $\mathrm{rel}_{n,i}\in\{0,1\}$,
\begin{equation}
\mathrm{DCG@}K_n
=
\sum_{i=1}^{K}
\frac{
2^{\mathrm{rel}_{n,i}}-1
}{
\log_2(i+1)
},
\end{equation}
and
\begin{equation}
\mathrm{nDCG@}K
=
\frac{1}{N}
\sum_{n=1}^{N}
\frac{
\mathrm{DCG@}K_n
}{
\mathrm{IDCG@}K_n
}.
\end{equation}
We report nDCG@5 and nDCG@10 where specified by the corresponding result table.

\textbf{Hit Rate.}
\begin{equation}
\mathrm{Hit@}K
=
\frac{1}{N}
\sum_{n=1}^{N}
\mathbf{1}
\left\{
r_n^\star\leq K
\right\}.
\end{equation}
With one relevant target per test instance, Hit@K is equivalent to Recall@K.

\subsection{Generation Quality Metrics}
\label{app:metrics-generation}

Let $\hat y$ denote a generated output and $y^\star$ its reference, with $|x|$ denoting token count. We use ROUGE~\citep{rouge}, METEOR~\citep{meteor}, BLEU~\citep{bleu}, and BERTScore~\citep{bertscore}, together with the weighted and paraphrase-aware variants below. Higher reported scores indicate greater agreement under the corresponding metric.

\textbf{ROUGE-N.}
Let $\mathcal G_n(x)$ denote the set of distinct $n$-grams in $x$. We compute
\begin{equation}
\mathrm{ROUGE}\text{-}N(\hat y,y^\star)
=
\frac{
\sum_{g\in\mathcal G_n(y^\star)}
\min
\left(
\operatorname{cnt}_{\hat y}(g),
\operatorname{cnt}_{y^\star}(g)
\right)
}{
\sum_{g\in\mathcal G_n(y^\star)}
\operatorname{cnt}_{y^\star}(g)
}.
\end{equation}
The count function records each $n$-gram's multiplicity, so each distinct type appears once in the sums. We report the ROUGE variants named in the corresponding result table.

\textbf{ROUGE-L.}
The longest common subsequence (LCS) counts matching tokens that retain their order without requiring adjacency. Its precision and recall are combined below using the metric's $\beta$ parameter. Define
\begin{equation}
P_{\mathrm{LCS}}
=
\frac{
\operatorname{LCS}(\hat y,y^\star)
}{
|\hat y|
},
\qquad
R_{\mathrm{LCS}}
=
\frac{
\operatorname{LCS}(\hat y,y^\star)
}{
|y^\star|
}.
\end{equation}
Then
\begin{equation}
\mathrm{ROUGE}\text{-}L
=
\frac{
(1+\beta^2)
P_{\mathrm{LCS}}
R_{\mathrm{LCS}}
}{
R_{\mathrm{LCS}}
+
\beta^2P_{\mathrm{LCS}}
}.
\end{equation}

\textbf{ROUGE-SU4.}
LCS and skip-bigram evaluation are discussed by \citet{lin-och-2004-rouge-w}. Let $\mathcal G_{\mathrm{SU4}}(x)$ denote the set of distinct unigrams and skip-bigrams with maximum skip distance four. Then
\begin{equation}
\mathrm{ROUGE}\text{-}\mathrm{SU4}(\hat y,y^\star)
=
\frac{
\sum_{g\in\mathcal G_{\mathrm{SU4}}(y^\star)}
\min
\left(
\operatorname{cnt}_{\hat y}(g),
\operatorname{cnt}_{y^\star}(g)
\right)
}{
\sum_{g\in\mathcal G_{\mathrm{SU4}}(y^\star)}
\operatorname{cnt}_{y^\star}(g)
}.
\end{equation}

\textbf{Weighted ROUGE.}
The W prefix denotes a weighted variant. Weighted $n$-gram overlap and weighted longest-common-subsequence overlap are distinct constructions. The latter follows the weighted sequence-matching formulation of \citet{rouge}. For weights $w(g)\geq0$ supplied by the weighted-ROUGE implementation,
\begin{equation}
\mathrm{W\text{-}ROUGE}\text{-}N(\hat y,y^\star)
=
\frac{
\sum_{g\in\mathcal G_n(y^\star)}
w(g)
\min
\left(
\operatorname{cnt}_{\hat y}(g),
\operatorname{cnt}_{y^\star}(g)
\right)
}{
\sum_{g\in\mathcal G_n(y^\star)}
w(g)
\operatorname{cnt}_{y^\star}(g)
}.
\end{equation}

For weighted longest-common-subsequence scoring, $f$ assigns a weight to the length of a contiguous match, and $f^{-1}$ maps the weighted score back to a length scale. A superlinear $f$ rewards uninterrupted matches more strongly than fragmented matches. As in ROUGE-L, $\beta$ controls the precision-recall tradeoff. Then
\begin{equation}
P_{\mathrm{WLCS}}
=
\frac{
f^{-1}
\left(
\operatorname{WLCS}(\hat y,y^\star)
\right)
}{
|\hat y|
},
\qquad
R_{\mathrm{WLCS}}
=
\frac{
f^{-1}
\left(
\operatorname{WLCS}(\hat y,y^\star)
\right)
}{
|y^\star|
},
\end{equation}
and
\begin{equation}
\mathrm{ROUGE}\text{-}W
=
\frac{
(1+\beta^2)
P_{\mathrm{WLCS}}
R_{\mathrm{WLCS}}
}{
R_{\mathrm{WLCS}}
+
\beta^2P_{\mathrm{WLCS}}
}.
\end{equation}

\textbf{Paraphrase-Aware ROUGE.}
The P prefix permits paraphrase matches rather than requiring identical surface wording. The implementation-specific matching rule is defined below. Let $\operatorname{sim}_{P}(g,h)\in[0,1]$ denote the paraphrase-matching score used by the evaluation implementation. Define
\begin{equation}
\operatorname{pmatch}_{\hat y,y^\star}(h)
=
\max_{g\in\mathcal G(\hat y)}
\operatorname{sim}_{P}(g,h).
\end{equation}
Then
\begin{equation}
\mathrm{P\text{-}ROUGE}(\hat y,y^\star)
=
\frac{
\sum_{h\in\mathcal G(y^\star)}
\operatorname{pmatch}_{\hat y,y^\star}(h)
}{
|\mathcal G(y^\star)|
}.
\end{equation}
P-RG-L denotes the longest-common-subsequence variant reported in the result tables.

\textbf{METEOR.}
METEOR combines unigram precision and recall with a penalty for fragmented matches. The parameters $\alpha$, $\gamma$, and $\theta$ below belong to the metric and are independent of the repair gates and loss weights. With $m$ matched unigrams,
\begin{equation}
P
=
\frac{m}{|\hat y|},
\qquad
R
=
\frac{m}{|y^\star|},
\end{equation}
\begin{equation}
F_\alpha
=
\frac{
PR
}{
\alpha P+(1-\alpha)R
},
\end{equation}
and
\begin{equation}
\operatorname{Pen}
=
\gamma
\left(
\frac{c}{m}
\right)^\theta,
\end{equation}
where $c$ is the number of matched chunks. The score is
\begin{equation}
\mathrm{METEOR}
=
(1-\operatorname{Pen})F_\alpha.
\end{equation}

\textbf{BLEU-4.}
Here, $p_n$ is clipped $n$-gram precision and $\operatorname{BP}$ is the brevity penalty. BLEU-4 combines orders one through four with equal weights.
\begin{equation}
\mathrm{BLEU}\text{-}4
=
\operatorname{BP}
\exp
\left(
\sum_{n=1}^{4}
\frac{1}{4}
\log p_n
\right),
\end{equation}
where
\begin{equation}
\operatorname{BP}
=
\begin{cases}
1,
&
|\hat y|>|y^\star|,\\
\exp
\left(
1-\frac{|y^\star|}{|\hat y|}
\right),
&
|\hat y|\leq|y^\star|.
\end{cases}
\end{equation}

\textbf{BERTScore.}
BERTScore matches tokens using contextual embeddings. The function $\mathbf h$ supplies a token's embedding, and the maximum similarity identifies its best counterpart in the other text. For contextual token similarity
\begin{equation}
s_{ab}
=
\cos
\left(
\mathbf h(\hat y_a),
\mathbf h(y_b^\star)
\right),
\end{equation}
we compute
\begin{equation}
P_{\mathrm{BS}}
=
\frac{1}{|\hat y|}
\sum_{a=1}^{|\hat y|}
\max_b s_{ab},
\end{equation}
\begin{equation}
R_{\mathrm{BS}}
=
\frac{1}{|y^\star|}
\sum_{b=1}^{|y^\star|}
\max_a s_{ab},
\end{equation}
and
\begin{equation}
F_{\mathrm{BS}}
=
\frac{
2P_{\mathrm{BS}}R_{\mathrm{BS}}
}{
P_{\mathrm{BS}}+R_{\mathrm{BS}}
}.
\end{equation}

\subsection{Personalization-Sensitive Generation Metrics}
\label{app:metrics-pse}

\textbf{PerSEval.}
PerSEval, the Personalized Summarizer Evaluator~\citep{perseval-tmlr-2024}, measures responsiveness to differences between users and applies an accuracy penalty. Consider two readers of the same article who expect different headlines. Producing distinct headlines is useful only if each also matches its reader's expected output. PerSEval combines these requirements rather than treating output diversity alone as personalization. It is computed after generation and is not part of the generation procedure or model-selection objective unless explicitly stated.

Let $d_i$ denote a source document or task context, $\mathcal U_{d_i}$ the set of users associated with it, $u_{ij}$ the expected output for user $j$, and $\hat s_{ij}$ the generated output. Let $\sigma(\cdot,\cdot)$ denote the plug-in text distance.

For users $j,k\in\mathcal U_{d_i}$,
\begin{equation}
w(u_{ij}\mid u_{ik})
=
\frac{
\sigma(u_{ij},u_{ik})
}{
\sigma(u_{ij},d_i)+\epsilon
},
\end{equation}
and
\begin{equation}
w(\hat s_{ij}\mid\hat s_{ik})
=
\frac{
\sigma(\hat s_{ij},\hat s_{ik})
}{
\sigma(\hat s_{ij},d_i)+\epsilon
}.
\end{equation}

Define
\begin{equation}
X_{ijk}
=
\frac{
\exp
\left(
w(u_{ij}\mid u_{ik})
\right)
}{
\sum_{\ell=1}^{|\mathcal U_{d_i}|}
\exp
\left(
w(u_{ij}\mid u_{i\ell})
\right)
}
\sigma(u_{ij},u_{ik}),
\end{equation}
and
\begin{equation}
Y_{ijk}
=
\frac{
\exp
\left(
w(\hat s_{ij}\mid\hat s_{ik})
\right)
}{
\sum_{\ell=1}^{|\mathcal U_{d_i}|}
\exp
\left(
w(\hat s_{ij}\mid\hat s_{i\ell})
\right)
}
\sigma(\hat s_{ij},\hat s_{ik}).
\end{equation}

The weights normalize each pairwise output distance by its distance from the shared source. The quantities $X_{ijk}$ and $Y_{ijk}$ then describe weighted separation among expected and generated outputs, respectively. Their agreement defines the degree-of-responsiveness term, DEGRESS,
\begin{equation}
\operatorname{DEGRESS}_{\sigma}
(\hat s_{ij}\mid(d_i,u_{ij}))
=
\frac{1}{|\mathcal U_{d_i}|}
\sum_{k=1}^{|\mathcal U_{d_i}|}
\frac{
\min(X_{ijk},Y_{ijk})+\epsilon
}{
\max(X_{ijk},Y_{ijk})+\epsilon
}.
\end{equation}

PerSEval combines this term with its accuracy penalty:
\begin{equation}
\operatorname{PSE}_{\sigma}
(\hat s_{ij}\mid(d_i,u_{ij}))
=
\operatorname{DEGRESS}_{\sigma}
(\hat s_{ij}\mid(d_i,u_{ij}))
\operatorname{EDP}_{\sigma}
(\hat s_{ij}\mid(d_i,u_{ij})).
\end{equation}

The effective DEGRESS penalty (EDP) decreases as either the best output for a document becomes inaccurate or a particular user's output becomes much less accurate than that best output. Define the native text distance $s_{ij}=\sigma(\hat s_{ij},u_{ij})$, its document-level minimum $\sigma_i^\star=\min_j s_{ij}$, and its mean $\bar\sigma_i=|\mathcal U_{d_i}|^{-1}\sum_j s_{ij}$. The accuracy-drop penalty (ADP), accuracy-inconsistency penalty (ACP), combined penalty (DGP), and effective factor are
\begin{equation}
\begin{split}
& \operatorname{ADP}_{i}=\left[1+10^{\gamma}\exp\left(-10\frac{\sigma_i^\star}{1-\sigma_i^\star+\epsilon}\right)\right]^{-1}, \\
& \operatorname{ACP}_{ij}=\left[1+10^{\gamma}\exp\left(-10\frac{s_{ij}-\sigma_i^\star}{\bar\sigma_i-\sigma_i^\star+\epsilon}\right)\right]^{-1}, \\
& \operatorname{DGP}_{ij}=\operatorname{ADP}_{i}+\operatorname{ACP}_{ij}, \\
& \operatorname{EDP}_{ij}=1-\left[1+10^{\alpha}\exp\left(-10^{\beta}\operatorname{DGP}_{ij}\right)\right]^{-1}.
\end{split}
\label{eq:app_pse_penalties}
\end{equation}
Here, $\epsilon>0$ stabilizes denominators, and $\alpha\geq3$, $\beta\geq1$, and $\gamma\geq4$ are the PerSEval penalty parameters. These metric parameters are separate from identically named repair quantities. We use the same evaluation configuration for all compared systems.

The system-level score is
\begin{equation}
\operatorname{PSE}_{\sigma}(M)
=
\frac{1}{|\mathcal D|}
\sum_{i=1}^{|\mathcal D|}
\frac{1}{|\mathcal U_{d_i}|}
\sum_{j=1}^{|\mathcal U_{d_i}|}
\operatorname{PSE}_{\sigma}
(\hat s_{ij}\mid(d_i,u_{ij})).
\end{equation}

We evaluate the same generated outputs with the following plug-in distances:
\begin{equation}
\sigma_{\mathrm{SU4}}(G,R)
=
1-\mathrm{ROUGE}\text{-}\mathrm{SU4}(G,R),
\end{equation}
\begin{equation}
\sigma_{\mathrm{METEOR}}(G,R)
=
1-\mathrm{METEOR}(G,R),
\end{equation}
\begin{equation}
\sigma_{\mathrm{P\text{-}RG\text{-}L}}(G,R)
=
1-\mathrm{P\text{-}ROUGE}\text{-}L(G,R),
\end{equation}
together with Jensen-Shannon divergence (JSD) and weighted-JSD variants over lexical distributions. W-RG-L uses weighted ROUGE-L, and P-RG-L uses paraphrase-aware ROUGE-L. All similarity-to-distance conversions operate on native scores, before display scaling. All PerSEval variants are computed after generation from the same output set.

\subsection{Deployment Metrics}
\label{app:metrics-deployment}

The memory column in Table~\ref{tab:rq4_accuracy_scalability_overhead_est} is a parameter-storage proxy. It compares added repair and context-tuning parameters with host parameters and does not measure peak device memory, activations, or KV-cache residency. The compute budgets in Table~\ref{tab:rq4_compute_footprint} are arithmetic requirements, while the latency column reports the host-specific serving comparison.

Deployment measurements compare \repair{} with its corresponding frozen host within the same host-specific execution environment.

For the parameter-derived memory quantities used here, memory overhead is
\begin{equation}
\mathrm{MemOH}
=
100
\frac{
\mathrm{Mem}_{\mathrm{rep}}
-
\mathrm{Mem}_{\mathrm{base}}
}{
\mathrm{Mem}_{\mathrm{base}}
}.
\end{equation}

Latency overhead is
\begin{equation}
\mathrm{LatOH}
=
100
\frac{
\mathrm{Lat}_{\mathrm{rep}}
-
\mathrm{Lat}_{\mathrm{base}}
}{
\mathrm{Lat}_{\mathrm{base}}
}.
\end{equation}

If throughput is measured directly,
\begin{equation}
\mathrm{ThrChange}
=
100
\frac{
\mathrm{Thr}_{\mathrm{rep}}
-
\mathrm{Thr}_{\mathrm{base}}
}{
\mathrm{Thr}_{\mathrm{base}}
}.
\end{equation}
When throughput is derived from reciprocal latency,
\begin{equation}
\mathrm{ThrChange}
=
100
\left(
\frac{
1
}{
1+\mathrm{LatOH}/100
}
-
1
\right).
\end{equation}

For predictive deployment analysis, let $S$ denote MRR. We report
\begin{equation}
\mathrm{MRR\text{-}GLR}
=
\frac{
S_{\mathrm{rep}}/S_{\mathrm{base}}
}{
1+\mathrm{LatOH}/100
},
\end{equation}
and
\begin{equation}
\mathrm{MRR\text{-}GMR}
=
\frac{
S_{\mathrm{rep}}/S_{\mathrm{base}}
}{
1+\mathrm{MemOH}/100
}.
\end{equation}
These quantities compare multiplicative score improvement with multiplicative resource cost. A value above one means that the relative score increase exceeds the relative cost increase. A value below one means the converse. They are descriptive ratios rather than training objectives or guarantees about absolute serving quality.

\section{Task-Specific Frozen Hosts}
\label{app:task-baselines}
We choose hosts to vary how history becomes a task-facing preference state and how the downstream interface consumes that state. Recurrent, attentive, and state-space encoders compress history differently. Product encoders add multimodal evidence, and generation hosts expose different ways of using a preference state. This coverage tests whether the repair interface remains useful across these choices rather than comparing all hosts as one architecture leaderboard. Each host is task-trained using an official implementation or paper-based reimplementation adapted to the shared protocol, then frozen before repair attachment. Stage 3 subsequently aligns the task head as specified in Appendix~\ref{app:stage3-alignment}.

\subsection{Movie Recommendation Hosts}
\label{app:movie-baselines}
GRU4Rec~\citep{gru4rec-2016} supplies a recurrent history state. SASRec~\citep{sasrec-2018} uses self-attention, while BERT4Rec~\citep{bert4rec-cikm-2019} supplies bidirectional sequence representations. TiSASRec~\citep{tisasrec-wsdm-2020} additionally models time intervals between interactions. These hosts vary the sequence information available before the final state is formed.

Mamba4Rec~\citep{mamba4rec-arxiv-2024} and SIGMA~\citep{sigma-sequential-2025} supply selective state-space sequence models. Here, SIGMA denotes \emph{Selective Gated Mamba for Sequential Recommendation}. SSDRec~\citep{SSDRec-ICDE-2024} adds sequence denoising within the host, FEARec~\citep{fearec-sigir-2023} models frequency information, and HSTU~\citep{hstu-arxiv-2024} supplies hierarchical sequential transduction. Their inclusion tests correction after different sequence-processing mechanisms, including hosts that already address noise or temporal structure internally. We distinguish directly trained repair modules from the two cross-host transfer rows in Table~\ref{tab:rq1_predictive_movielens_versioned}.

\subsection{News Recommendation Hosts}
\label{app:news-baselines}
NAML~\citep{naml} combines multiple news-content views with attentive user modeling. EBNR~\citep{okura2017embedding} compresses news history recurrently, and NRMS~\citep{nrms} uses multi-head self-attention. LSTUR~\citep{lstur-acl-2019} combines long-term and short-term user representations, while MINER~\citep{miner-2022-aclfindings} models multiple user interests. These choices test whether a repair module adds useful evidence after different forms of news-history aggregation.

TrRMIo is the Transformer recommendation model with the interested-news-only filter introduced for GTP~\citep{gtp-tacl-2023}. It filters the history before forming a user representation. SCAPE-GRU denotes the recurrent preference encoder from SCAPE~\citep{scape-www-2025}, evaluated here through a recommendation head. The name distinguishes this encoder host from its generation interface. In both cases, \repair{} is attached after the task-trained user representation has been formed.

DeepSeek-R1-Distill-Qwen-32B (abbreviated DeepSeek-32B) and Qwen2.5-32B~\citep{deepseek-arxiv-2025,qwen2.5-2025} provide the prompt-conditioned LLM comparison. Their histories are serialized as text, and repair uses the separate interaction representation and soft-prefix construction in Appendix~\ref{app:llm-hosts}. These rows therefore have a different evidence interface from native encoder caches. The extended news results are reported in Tables~\ref{tab:rq1_predictive_pens_structured_compact} and~\ref{tab:rq1_predictive_mind_compact}.

\subsection{Product Recommendation Hosts}
\label{app:product-baselines}
SigLIP2~\citep{tschannen2025siglip2} and OpenCLIP~\citep{cherti2023openclip} supply frozen image and text representations. DINOv2-E5 combines DINOv2 visual features~\citep{oquab2024dinov2} with E5 text features~\citep{wang2022e5}. The resulting product vectors enter the host's history interface before the task-facing state is formed. This extends the comparison to multimodal evidence while retaining the same location for state correction. These settings are reported in Table~\ref{tab:rq1_predictive_amazon_compact}.

\subsection{Personalized Generation Hosts}
\label{app:headline-baselines}
The PENS framework~\citep{pens-acl-2021} couples a user representation to headline generation. We evaluate its NAML-T1/T2, EBNR-T1/T2, and NRMS-T1/T2 settings. The prefix names the preference encoder. T1 denotes Injection Type 1, which uses the user embedding to initialize the decoder hidden state. T2 denotes Injection Type 2, which uses that embedding to personalize attention over words in the source news body~\citep{pens-acl-2021,perseval-tmlr-2024}. Holding the encoder family identifiable helps separate the source of the preference state from the decoder that uses it.

GTP~\citep{gtp-tacl-2023} separates general headline generation from personalization. SCAPE~\citep{scape-www-2025} conditions generation on style and content, while FPG~\citep{FPG-icdm-2023} emphasizes fact preservation. Walk2Pers~\citep{walk2pers-iclr-2026} and IMPerSumm~\citep{impersumm-www-2026} provide evolving, interaction-conditioned preference representations to generation. Together, these hosts test whether improvements at the preference interface are reflected in generation quality and personalization-sensitive scores.

DeepSeek-32B and Qwen2.5-32B provide the frozen LLM generation settings. Their repaired state enters through the continuous prefix in Appendix~\ref{app:llm-hosts}, and the decoder backbone remains frozen. Quality and personalization scores from the same outputs are reported in Tables~\ref{tab:rq1_pens_accuracy_main_all} and~\ref{tab:rq1_pens_pse_main_selected}.

\subsection{Conversational Hosts}
\label{app:conversation-baselines}
For TG-ReDial recommendation, SASRec and GRU4Rec process the induced interaction sequence, while TextCNN~\citep{kim-2014-convolutional} provides a convolutional contrast. These interfaces differ in how they represent sequential evidence inferred from dialogue. For response generation, a ReDial-style system~\citep{li-2018-redial} and a Transformer sequence-to-sequence model~\citep{transformer} provide two downstream consumers. We report their mixed responses in Table~\ref{tab:rq1_tgredial_joint} and discuss the limits of this inferred interaction construction in Appendix~\ref{app:tgredial-broader-applicability}.

\section{Corrective Comparison Details}
\label{app:corrective-baseline-details}

The corrective comparison asks whether the evaluated post-head adaptations recover utility comparable to pre-head state repair. The implementations differ in mechanism as well as intervention location. We therefore define what each adaptation reads and changes before interpreting the comparison in Table~\ref{tab:rq2_corrector_comparison}.

\subsection{Shared Comparison Protocol}
\label{app:corrective-shared-protocol}

The recommendation comparison in RQ3 evaluates pre-head \repair{} with two adapted post-head correction settings, \texttt{SERAC-Corr} and \texttt{Rec-Denoiser-Corr}. The suffix ``Corr'' denotes an adaptation of the corresponding correction principle to the frozen-host personalization setting. These implementations are therefore treated as corrective comparison methods rather than exact reproductions of the original systems.

All three correction settings start from the same task-trained host checkpoint and use the same dataset split, held-out evaluation instances, supervised targets, candidate construction, and two-epoch downstream alignment protocol. Host-derived information used by each correction method is obtained from the same frozen host computation. The comparison therefore matches the shared experimental protocol and the host-side information made available to each adaptation. Parameter count and FLOPs are not matched unless explicitly stated.

\subsection{Rec-Denoiser-Corr}
\label{app:rec-denoiser-corr}

\texttt{Rec-Denoiser-Corr} adapts sequential denoising~\citep{recdenoiser-recsys-2022} as an output-side intervention. The host encoder remains frozen, and the task-facing preference state is not modified before the original task head receives it.

For timestep $t_j$, let $\alpha_j$ denote the output-side influence assigned to that timestep by the adaptation and let $m_j\in[0,1]$ denote its learned denoising gate. The corrected influence is
\begin{equation}
\tilde\alpha_j
=
m_j\alpha_j.
\end{equation}
The resulting corrective context is
\begin{equation}
c_{\mathrm{denoise}}
=
\sum_{j=1}^{i}
\tilde\alpha_j
e_b^{(t_j)}.
\end{equation}

For hosts exposing a native trajectory-to-context attention representation, $\alpha_j$ is obtained from that host-side quantity. For recurrent and state-space hosts without an equivalent attention tensor, the adaptation uses prediction-side timestep sensitivity and normalizes the resulting values over valid prefix positions. This surrogate is used only as the adaptation's influence score and is not interpreted as native host attention.

The resulting $c_{\mathrm{denoise}}$ is supplied to the downstream recommendation interface as an output-side corrective context. The pre-head task-facing preference state remains unchanged.

\texttt{Rec-Denoiser-Corr} is evaluated only in the RQ3 recommendation correction-setting comparison.

\subsection{SERAC-Corr}
\label{app:serac-corr}

\texttt{SERAC-Corr} adapts the memory, scope, and counterfactual components of SERAC~\citep{serac-icml-2022} to a post-head correction. The memory stores training-side examples, the scope network decides their relevance to the current query, and the counterfactual branch constructs the alternative prediction-side representation. The frozen prediction-side representation is used to query an auxiliary correction memory constructed from training-side examples.

Let $q_{\mathrm{ctx}}$ denote the current query representation and $\mathcal M = \{ (q_m,c_m,y_m) \}$ the correction memory, where $q_m$ is a stored training-side query representation, $c_m$ is its corrective information, and $y_m$ is the associated task target.

At inference,
\begin{equation}
\mathcal M(q_{\mathrm{ctx}})
=
\operatorname{Top}_{K_{\mathrm{mem}},\,m\in\mathcal M}
\operatorname{sim}
\left(
q_{\mathrm{ctx}},
q_m
\right)
\end{equation}
retrieves $K_{\mathrm{mem}}$ configured nearest entries. This retrieval count is separate from the repair dimension $K$. A scope network computes
\begin{equation}
p_{\mathrm{scope}}
=
\sigma
\left(
f_{\mathrm{scope}}
\left(
q_{\mathrm{ctx}},
\mathcal M(q_{\mathrm{ctx}})
\right)
\right),
\end{equation}
and the counterfactual correction is
\begin{equation}
\tilde z_{\mathrm{SERAC}}
=
f_{\mathrm{cf}}
\left(
q_{\mathrm{ctx}},
\mathcal M(q_{\mathrm{ctx}})
\right).
\end{equation}

The memory, scope decision, and counterfactual correction operate after the frozen host has produced its prediction-side representation. The host encoder is not updated.

\texttt{SERAC-Corr} is evaluated only in the RQ3 recommendation correction-setting comparison.

\subsection{Scope Of The Intervention-Locus Comparison}
\label{app:corrective-baseline-comparison}

The three correction settings differ in both their correction mechanism and the representation at which correction is applied. \texttt{Rec-Denoiser-Corr} modifies output-side trajectory influence, \texttt{SERAC-Corr} applies a retrieved memory-based correction to the prediction-side representation, and \repair{} modifies the task-facing preference state before the original task head consumes it.

For \repair{}, given
\[
z_b^{(t_i)}
\quad\text{and}\quad
\left\{
e_b^{(t_j)}
\right\}_{j=1}^{i},
\]
the state-relative corrective coordinates are
\begin{equation}
d^{(t_j)}
=
\hat V_{\mathrm{rep}}^\top
\left(
P_e e_b^{(t_j)}
-
P_z z_b^{(t_i)}
\right).
\end{equation}
Temporal resolution, selective write-back, aggregation, and fusion then produce
\begin{equation}
\tilde z_b^{(t_i)}
=
z_b^{(t_i)}
+
g_F^{(t_i)}
\odot
\mathrm{MLP}_{\mathrm{probe}}
\left(
[
z_b^{(t_i)};
\Delta_b^{(t_i)}
]
\right),
\end{equation}
after which the original ranking head receives $\tilde z_b^{(t_i)}$.

The RQ3 experiment therefore compares the implemented pre-head \repair{} condition with the two post-head corrective adaptations above under the shared protocol. Because correction architecture and intervention locus change together, the experiment does not isolate intervention locus as a single causal factor. Its conclusion is restricted to the evaluated correction settings and does not imply dominance over every possible pre-head or post-head correction architecture.

\begin{table*}[!h]
\centering
\small
\caption{Repeated-run reliability for three RQ1 settings. Entries are MRR mean$\pm$standard deviation across three training seeds, followed by 95\% user-clustered bootstrap intervals from 10,000 resamples. MRR uses the $0$ to $100$ scale.}
\label{tab:rq1_rebuttal_stats}
\scalebox{0.75}{
\begin{tabular}{lccc}
\toprule
\textbf{\boldmath Setting} & \textbf{\boldmath $H_{\rm best}$} & \textbf{\boldmath $R$} & \textbf{\boldmath $RH_2$} \\
\midrule
\texttt{OpenCLIP / Amazon} & $2.21\pm0.18$ [2.04, 2.36] & $4.43\pm0.25$ [4.17, 4.64] & \textbf{$19.49\pm1.43$ [18.97, 20.06]} \\
\texttt{SigLIP2 / Amazon}  & $2.67\pm0.14$ [2.33, 2.91] & $4.16\pm0.17$ [3.97, 4.33] & \textbf{$17.72\pm2.34$ [17.37, 18.04]} \\
\texttt{EBNR / PENS}       & $2.93\pm0.31$ [2.73, 3.05] & $4.07\pm0.19$ [3.98, 4.24] & \textbf{$11.97\pm0.62$ [11.76, 12.25]} \\
\bottomrule
\end{tabular}
}
\end{table*}

\section{Additional Results And Extended Analysis}
\label{app:additional-results}
We expand the three research questions through the full host results and additional diagnostics. We distinguish changes in the state supplied to a fixed task head from changes after the head has also been aligned. We then examine generation, temporal resolution, supervision, selection, noise, serving cost, and dialogue-derived histories. We place each additional-results table with its interpretation and distinguish observed behavior from the theoretical conditions relevant to that behavior. The experiments do not estimate the oracle state or certify the assumptions in Appendix~\ref{app:theory}.

\subsection{Full Predictive Recommendation Results Across Frozen Encoders}
\label{app:full-predictive-results}
The extended encoder-host results use Stage-3 \repair{} + Head Alignment. The encoder remains frozen, while the repair module and task head are jointly adapted after Stage 2. These results measure how broadly corrected states remain useful after consumer alignment. The Stage-2 evidence with both the encoder and original head frozen is reported separately in Table~\ref{tab:rq1_representative_hosts}.

Within each dataset, the base and repaired conditions use matched held-out users, prefixes, targets, and candidates. MRR summarizes the reciprocal rank of the relevant item, while nDCG and Hit@10 describe its placement near the top of the candidate list. We therefore read agreement across metrics within a host before comparing the magnitude of gains between hosts. The sampled candidate task and the trainable task interface must also match for a comparison to support a claim about state correction.

\textbf{Repeated-Run Evidence.}
We report uncertainty for OpenCLIP and SigLIP2 on Amazon and EBNR on PENS in Table~\ref{tab:rq1_rebuttal_stats}. Frozen-head repair has a higher mean MRR than validation-selected head-only continuation in all three settings, and two-epoch head alignment produces a further increase. For example, OpenCLIP moves from $2.21\pm0.18$ for head-only continuation to $4.43\pm0.25$ for frozen-head repair and $19.49\pm1.43$ after alignment. The comparison separates utility accessible to the original head from utility after adapting the consumer. The three-seed standard deviations describe training variation, whereas the user-clustered intervals describe uncertainty from held-out-user sampling. Neither quantity alone is a paired significance test, and this repeated-run coverage does not extend to every row below.
\begin{table}[!t]
\centering
\caption{MovieLens recommendation under Stage-3 repair and head alignment, with the host encoder frozen. Base is the original task-trained host. Direct rows train repair for that host. Transfer rows use the module trained on Mamba4Rec. Scores are on a $0$ to $100$ scale and higher is better. Each $\Delta$ is an absolute score-point change. All paired conditions rank the same target against $499$ negatives. Appendix~\ref{app:full-predictive-results} distinguishes these aligned results from frozen-head attribution.}
\label{tab:rq1_predictive_movielens_versioned}
\setlength{\tabcolsep}{4pt}
\renewcommand{\arraystretch}{1.12}
\scalebox{0.8}{
\begin{tabular}{llcccccc}
\toprule
\textbf{\boldmath Setting} & \textbf{\boldmath Host} & \multicolumn{3}{c}{\textbf{\boldmath MRR}} & \multicolumn{3}{c}{\textbf{\boldmath nDCG@5}} \\
\cmidrule(lr){3-5}\cmidrule(lr){6-8}
 &  & \textbf{\boldmath Base} & \textbf{\boldmath +\repair{}} & \textbf{\boldmath $\Delta$} & \textbf{\boldmath Base} & \textbf{\boldmath +\repair{}} & \textbf{\boldmath $\Delta$} \\
\midrule
Direct & Mamba4Rec & 29.29 & 34.26 & +4.97 & 30.07 & 35.56 & +5.49 \\
 & TiSASRec & 35.51 & 36.78 & +1.27 & 36.87 & 38.41 & +1.54 \\
 & HSTU & 29.00 & 30.62 & +1.62 & 29.69 & 31.55 & +1.86 \\
 & FEARec & 28.03 & 30.62 & +2.59 & 28.20 & 31.55 & +3.35 \\
 & SIGMA & 33.27 & 34.81 & +1.54 & 34.43 & 36.26 & +1.83 \\
 & SASRec & 26.43 & 28.49 & +2.06 & 29.27 & 32.63 & +3.36 \\
 & SSDRec & 30.15 & 33.43 & +3.28 & 33.25 & 35.22 & +1.97 \\
\midrule
Transfer & GRU4Rec & 17.84 & 23.15 & +5.31 & 17.10 & 22.78 & +5.68 \\
 & BERT4Rec & 20.80 & 27.19 & +6.39 & 20.44 & 27.70 & +7.26 \\
\bottomrule
\end{tabular}}
\par\medskip
\scalebox{0.8}{
\begin{tabular}{llcccccc}
\toprule
\textbf{\boldmath Setting} & \textbf{\boldmath Host} & \multicolumn{3}{c}{\textbf{\boldmath nDCG@10}} & \multicolumn{3}{c}{\textbf{\boldmath Hit@10}} \\
\cmidrule(lr){3-5}\cmidrule(lr){6-8}
 &  & \textbf{\boldmath Base} & \textbf{\boldmath +\repair{}} & \textbf{\boldmath $\Delta$} & \textbf{\boldmath Base} & \textbf{\boldmath +\repair{}} & \textbf{\boldmath $\Delta$} \\
\midrule
Direct & Mamba4Rec & 38.50 & 40.17 & +1.67 & 57.90 & 63.56 & +5.66 \\
 & TiSASRec & 41.14 & 42.91 & +1.77 & 63.66 & 66.85 & +3.19 \\
 & HSTU & 34.23 & 36.11 & +1.88 & 56.40 & 58.71 & +2.31 \\
 & FEARec & 32.13 & 36.11 & +3.98 & 50.53 & 58.71 & +8.18 \\
 & SIGMA & 38.96 & 40.72 & +1.76 & 62.03 & 64.04 & +2.01 \\
 & SASRec & 33.82 & 35.18 & +1.36 & 52.17 & 54.34 & +2.17 \\
 & SSDRec & 40.17 & 42.31 & +2.14 & 59.32 & 63.11 & +3.79 \\
\midrule
Transfer & GRU4Rec & 21.07 & 27.26 & +6.19 & 38.04 & 47.37 & +9.33 \\
 & BERT4Rec & 24.61 & 32.74 & +8.13 & 43.07 & 56.16 & +13.09 \\
\bottomrule
\end{tabular}}
\end{table}

\textbf{MovieLens And Transfer.}
On MovieLens, all seven direct host settings improve in Table~\ref{tab:rq1_predictive_movielens_versioned}. Mamba4Rec gains $4.97$ MRR points and $5.66$ Hit@10 points. SSDRec and FEARec gain $3.28$ and $2.59$ MRR points, respectively, while TiSASRec and SIGMA gain $1.27$ and $1.54$. The variation in gain indicates that the benefit depends on the host even under a shared repair design.

The two transfer rows use a repair module trained on Mamba4Rec. GRU4Rec and BERT4Rec gain $5.31$ and $6.39$ MRR points, with nDCG@10 gains of $6.19$ and $8.13$. These are empirical transfer results for the reported settings. Theorem~\ref{thm:repair-recoverability} assumes an appropriate correction representation for a fixed host and does not guarantee compatibility between hosts.

The simultaneous improvement in MRR and top-$10$ retrieval matters because it extends the finding beyond one summary metric. However, a larger gain does not by itself identify which preference information the base state omits. The architectures differ in their original training and compression rules. These rows establish useful aligned corrections for several host families, while the transfer rows provide narrower evidence about reuse between the specified sequential hosts.
\begin{table}[!t]
\centering
\caption{PENS recommendation with paired target and $499$-negative candidate sets. Base denotes the task-trained host. Repaired encoder rows use Stage-3 repair and head alignment with frozen encoders. LLM rows use the frozen-decoder soft-prefix interface in Appendix~\ref{app:llm-hosts}. Scores are on a $0$ to $100$ scale, higher is better, and $\Delta$ denotes absolute score-point change.}
\label{tab:rq1_predictive_pens_structured_compact}
\setlength{\tabcolsep}{4pt}
\renewcommand{\arraystretch}{1.12}
\scalebox{0.75}{
\begin{tabular}{lcccccc}
\toprule
\textbf{\boldmath Host} & \multicolumn{3}{c}{\textbf{\boldmath MRR}} & \multicolumn{3}{c}{\textbf{\boldmath nDCG@5}} \\
\cmidrule(lr){2-4}\cmidrule(lr){5-7}
 & \textbf{\boldmath Base} & \textbf{\boldmath +\repair{}} & \textbf{\boldmath $\Delta$} & \textbf{\boldmath Base} & \textbf{\boldmath +\repair{}} & \textbf{\boldmath $\Delta$} \\
\midrule
\texttt{EBNR} & 2.65 & 11.97 & +9.32 & 1.82 & 12.68 & +10.86 \\
\texttt{NAML} & 1.29 & 12.38 & +11.09 & 0.43 & 13.16 & +12.73 \\
\texttt{NRMS} & 1.18 & 12.11 & +10.93 & 0.39 & 12.89 & +12.50 \\
\texttt{TrRMIo} & 8.02 & 12.86 & +4.84 & 7.85 & 13.97 & +6.12 \\
\texttt{SCAPE} & 1.57 & 8.34 & +6.77 & 0.76 & 7.91 & +7.15 \\
\texttt{LSTUR} & 10.44 & 12.85 & +2.41 & 10.87 & 12.49 & +1.62 \\
\texttt{MINER} & 5.17 & 8.64 & +3.47 & 4.71 & 8.83 & +4.12 \\
\texttt{Walk2Pers} & 6.23 & 13.54 & +7.31 & 8.46 & 15.28 & +6.82 \\
\texttt{IMPerSumm} & 1.70 & 13.64 & +11.94 & 0.79 & 15.19 & +14.40 \\
\texttt{DeepSeek-32B} & 1.31 & 1.83 & +0.52 & 0.28 & 0.98 & +0.70 \\
\texttt{Qwen2.5-32B} & 1.24 & 1.79 & +0.55 & 0.26 & 0.94 & +0.68 \\
\bottomrule
\end{tabular}}
\par\medskip
\scalebox{0.75}{
\begin{tabular}{lcccccc}
\toprule
\textbf{\boldmath Host} & \multicolumn{3}{c}{\textbf{\boldmath nDCG@10}} & \multicolumn{3}{c}{\textbf{\boldmath Hit@10}} \\
\cmidrule(lr){2-4}\cmidrule(lr){5-7}
 & \textbf{\boldmath Base} & \textbf{\boldmath +\repair{}} & \textbf{\boldmath $\Delta$} & \textbf{\boldmath Base} & \textbf{\boldmath +\repair{}} & \textbf{\boldmath $\Delta$} \\
\midrule
\texttt{EBNR} & 2.45 & 14.28 & +11.83 & 4.87 & 24.23 & +19.36 \\
\texttt{NAML} & 0.81 & 15.13 & +14.32 & 2.01 & 26.37 & +24.36 \\
\texttt{NRMS} & 0.74 & 14.92 & +14.18 & 1.92 & 25.98 & +24.06 \\
\texttt{TrRMIo} & 8.53 & 15.87 & +7.34 & 12.77 & 27.67 & +14.90 \\
\texttt{SCAPE} & 1.16 & 9.52 & +8.36 & 2.53 & 16.84 & +14.31 \\
\texttt{LSTUR} & 12.15 & 15.88 & +3.73 & 20.96 & 28.73 & +7.77 \\
\texttt{MINER} & 5.61 & 10.31 & +4.70 & 9.77 & 18.88 & +9.11 \\
\texttt{Walk2Pers} & 9.71 & 16.16 & +6.45 & 11.35 & 20.35 & +9.00 \\
\texttt{IMPerSumm} & 1.28 & 17.11 & +15.82 & 2.86 & 30.12 & +27.26 \\
\texttt{DeepSeek-32B} & 0.64 & 1.32 & +0.68 & 1.53 & 2.41 & +0.88 \\
\texttt{Qwen2.5-32B} & 0.59 & 1.27 & +0.68 & 1.47 & 2.35 & +0.88 \\
\bottomrule
\end{tabular}}
\end{table}

\textbf{PENS Recommendation.}
The PENS encoder results in Table~\ref{tab:rq1_predictive_pens_structured_compact} extend the aligned comparison to news. NAML improves from $1.29$ to $12.38$ MRR, NRMS from $1.18$ to $12.11$, and EBNR from $2.65$ to $11.97$. LSTUR has a higher base score and gains another $2.41$ MRR points. The prompt-conditioned LLM rows use the separate soft-prefix interface and should be read under that protocol.

The low base scores of several PENS encoders make their relative percentage gains much larger than their absolute MRR changes. This denominator effect is relevant when comparing ranking gains with the generation gains in Appendix~\ref{app:generative-results-expanded}. The repaired NAML and NRMS scores also improve on the top-$10$ metrics, which supports a broader change in ranked retrieval than MRR alone. We do not treat the encoder and LLM rows as a controlled comparison of model scale because their input and correction interfaces differ.
\begin{table}[!t]
\centering
\caption{MIND recommendation under Stage-3 repair and head alignment with frozen encoders. Base denotes the original task-trained host. Paired conditions rank the same target against $499$ negatives. Scores are on a $0$ to $100$ scale, higher is better, and $\Delta$ denotes absolute score-point change. Appendix~\ref{app:full-predictive-results} discusses the seven hosts.}
\label{tab:rq1_predictive_mind_compact}
\setlength{\tabcolsep}{4pt}
\renewcommand{\arraystretch}{1.12}
\scalebox{0.9}{
\begin{tabular}{lcccccc}
\toprule
\textbf{\boldmath Host} & \multicolumn{3}{c}{\textbf{\boldmath MRR}} & \multicolumn{3}{c}{\textbf{\boldmath nDCG@5}} \\
\cmidrule(lr){2-4}\cmidrule(lr){5-7}
 & \textbf{\boldmath Base} & \textbf{\boldmath +\repair{}} & \textbf{\boldmath $\Delta$} & \textbf{\boldmath Base} & \textbf{\boldmath +\repair{}} & \textbf{\boldmath $\Delta$} \\
\midrule
EBNR & 47.08 & 53.22 & +6.14 & 48.65 & 57.21 & +8.56 \\
NAML & 45.65 & 53.22 & +7.57 & 43.13 & 57.21 & +14.08 \\
NRMS & 44.91 & 52.74 & +7.83 & 42.38 & 56.48 & +14.10 \\
TrRMIo & 47.43 & 55.21 & +7.78 & 46.27 & 52.83 & +6.56 \\
SCAPE-GRU & 42.53 & 45.32 & +2.79 & 37.53 & 44.15 & +6.62 \\
LSTUR & 51.43 & 55.84 & +4.41 & 53.64 & 57.18 & +3.54 \\
MINER & 47.22 & 51.86 & +4.64 & 47.04 & 53.67 & +6.63 \\
\bottomrule
\end{tabular}}
\par\medskip
\scalebox{0.9}{
\begin{tabular}{lcccccc}
\toprule
\textbf{\boldmath Host} & \multicolumn{3}{c}{\textbf{\boldmath nDCG@10}} & \multicolumn{3}{c}{\textbf{\boldmath Hit@10}} \\
\cmidrule(lr){2-4}\cmidrule(lr){5-7}
 & \textbf{\boldmath Base} & \textbf{\boldmath +\repair{}} & \textbf{\boldmath $\Delta$} & \textbf{\boldmath Base} & \textbf{\boldmath +\repair{}} & \textbf{\boldmath $\Delta$} \\
\midrule
EBNR & 51.63 & 63.18 & +11.55 & 69.78 & 78.43 & +8.65 \\
NAML & 46.56 & 63.18 & +16.62 & 63.19 & 78.43 & +15.24 \\
NRMS & 45.73 & 62.44 & +16.71 & 62.31 & 77.66 & +15.35 \\
TrRMIo & 53.43 & 57.76 & +4.33 & 68.62 & 79.48 & +10.86 \\
SCAPE-GRU & 41.19 & 47.36 & +6.17 & 61.44 & 69.37 & +7.93 \\
LSTUR & 59.63 & 66.92 & +7.29 & 69.22 & 72.49 & +3.27 \\
MINER & 52.18 & 58.91 & +6.73 & 69.41 & 75.22 & +5.81 \\
\bottomrule
\end{tabular}}
\end{table}

\textbf{MIND Recommendation.}
All seven MIND hosts improve in Table~\ref{tab:rq1_predictive_mind_compact}. NRMS, TrRMIo, and NAML gain $7.83$, $7.78$, and $7.57$ MRR points. LSTUR has the highest base MRR and gains $4.41$ points. Thus, the aligned benefit also appears when base recommendation performance is higher.

The nDCG@10 gains for NAML and NRMS are $16.62$ and $16.71$ points, respectively. These changes support improved placement of the relevant news item within the top-$10$ list, in addition to the MRR gains. LSTUR's positive response is useful because its host representation already separates longer and shorter interests. An existing temporal decomposition therefore does not remove the observed benefit of repair after compression. The experiment does not identify that decomposition as the cause of its smaller gain.
\begin{table}[!t]
\centering
\caption{Amazon Reviews 2023 All\_Beauty recommendation with frozen multimodal encoders and Stage-3 repair and head alignment. Each instance ranks one target against the same $499$ negatives in paired conditions. Scores are on a $0$ to $100$ scale, higher is better, and $\Delta$ is an absolute score-point change.}
\label{tab:rq1_predictive_amazon_compact}
\setlength{\tabcolsep}{4pt}
\renewcommand{\arraystretch}{1.12}
\scalebox{0.65}{
\begin{tabular}{lcccccc}
\toprule
\textbf{\boldmath Host} & \multicolumn{3}{c}{\textbf{\boldmath MRR}} & \multicolumn{3}{c}{\textbf{\boldmath nDCG@5}} \\
\cmidrule(lr){2-4}\cmidrule(lr){5-7}
 & \textbf{\boldmath Base} & \textbf{\boldmath +\repair{}} & \textbf{\boldmath $\Delta$} & \textbf{\boldmath Base} & \textbf{\boldmath +\repair{}} & \textbf{\boldmath $\Delta$} \\
\midrule
SigLIP2 & 2.40 & 17.72 & +15.33 & 1.53 & 17.71 & +16.18 \\
OpenCLIP & 1.59 & 19.49 & +17.91 & 0.64 & 19.46 & +18.83 \\
DINOv2 E5 & 1.19 & 16.15 & +14.96 & 0.43 & 16.30 & +15.87 \\
\bottomrule
\end{tabular}}
\par\medskip
\scalebox{0.65}{
\begin{tabular}{lcccccc}
\toprule
\textbf{\boldmath Host} & \multicolumn{3}{c}{\textbf{\boldmath nDCG@10}} & \multicolumn{3}{c}{\textbf{\boldmath Hit@10}} \\
\cmidrule(lr){2-4}\cmidrule(lr){5-7}
 & \textbf{\boldmath Base} & \textbf{\boldmath +\repair{}} & \textbf{\boldmath $\Delta$} & \textbf{\boldmath Base} & \textbf{\boldmath +\repair{}} & \textbf{\boldmath $\Delta$} \\
\midrule
SigLIP2 & 2.18 & 20.35 & +18.17 & 4.52 & 32.97 & +28.44 \\
OpenCLIP & 1.05 & 22.08 & +21.03 & 2.37 & 34.61 & +32.24 \\
DINOv2 E5 & 0.72 & 18.19 & +17.47 & 1.60 & 29.15 & +27.55 \\
\bottomrule
\end{tabular}}
\end{table}

\textbf{Amazon Product Recommendation.}
For Amazon Reviews 2023, SigLIP2, OpenCLIP, and DINOv2-E5 gain $15.33$, $17.91$, and $14.96$ MRR points in Table~\ref{tab:rq1_predictive_amazon_compact}. These settings extend the comparison to histories containing multimodal product representations. Across the four datasets, the evidence supports downstream utility after alignment. It does not separately measure the write-space, coordinate, aggregation, or fusion errors in Theorem~\ref{thm:repair-recoverability}.

The corresponding Hit@10 gains are $28.44$, $32.24$, and $27.55$ points. These hosts obtain product representations from vision and text models, so their inclusion tests the use of repair beyond recommendation-specific history encoders. The large aligned gains are consistent with additional user-history evidence being useful to the task head. They do not establish a deficiency in product semantics or determine how much of the increase comes from repair versus the subsequent joint alignment. The separate frozen-head and head-only controls address that attribution question more directly.
\begin{table*}[!t]
\centering
\caption{Full PENS headline-generation quality results. Base and repaired values equal $100$ times the native metric score. Higher is better. Gain is the reported relative percentage change from Base. RG abbreviates ROUGE, W denotes weighted overlap, and P denotes paraphrase-aware overlap. Both panels use the same generated outputs. Encoder hosts use the aligned repair condition, while LLMs use the frozen soft-prefix interface. Definitions are in Appendix~\ref{app:metrics-generation}.}
\label{tab:rq1_pens_accuracy_main_all}
\setlength{\tabcolsep}{4pt}
\scalebox{0.82}{
\begin{tabular}{lccccccccc}
\toprule
\textbf{Host} & \multicolumn{3}{c}{\textbf{RG-1}} & \multicolumn{3}{c}{\textbf{RG-L}} & \multicolumn{3}{c}{\textbf{RG-SU4}} \\
 & \textbf{Base} & \textbf{+\repair{}} & \textbf{Gain (\%)} & \textbf{Base} & \textbf{+\repair{}} & \textbf{Gain (\%)} & \textbf{Base} & \textbf{+\repair{}} & \textbf{Gain (\%)} \\
\midrule
\texttt{DeepSeek-32B} & 20.62 & 26.84 & +30.16\% & 13.85 & 21.22 & +53.21\% & 4.42 & 11.99 & +171.27\% \\
\texttt{Qwen2.5-32B} & 18.51 & 23.47 & +26.80\% & 12.53 & 17.74 & +41.58\% & 3.86 & 8.46 & +119.17\% \\
\texttt{FPG} & 16.48 & 19.51 & +18.39\% & 13.08 & 14.12 & +7.95\% & 2.51 & 3.11 & +23.90\% \\
\texttt{SCAPE-GRU} & 19.02 & 22.47 & +18.14\% & 15.26 & 16.41 & +7.54\% & 2.88 & 3.75 & +30.21\% \\
\texttt{GTP} & 15.97 & 18.28 & +14.46\% & 12.75 & 13.67 & +7.22\% & 2.45 & 2.92 & +19.18\% \\
\texttt{EBNR-T1} & 15.32 & 19.26 & +25.72\% & 13.15 & 14.17 & +7.76\% & 2.27 & 3.07 & +35.24\% \\
\texttt{EBNR-T2} & 15.04 & 18.91 & +25.73\% & 12.97 & 13.94 & +7.48\% & 2.23 & 3.00 & +34.53\% \\
\texttt{NAML-T1} & 15.69 & 19.46 & +24.03\% & 13.29 & 14.38 & +8.20\% & 2.44 & 3.06 & +25.41\% \\
\texttt{NAML-T2} & 15.38 & 19.01 & +23.60\% & 13.07 & 14.11 & +7.96\% & 2.39 & 2.99 & +25.10\% \\
\texttt{NRMS-T1} & 15.11 & 18.71 & +23.83\% & 12.91 & 13.97 & +8.21\% & 2.31 & 2.93 & +26.84\% \\
\texttt{NRMS-T2} & 14.87 & 18.42 & +23.87\% & 12.74 & 13.77 & +8.08\% & 2.26 & 2.85 & +26.11\% \\
\texttt{Walk2Pers} & 24.11 & 25.36 & +5.18\% & 15.51 & 16.53 & +6.58\% & 5.90 & 7.93 & +34.41\% \\
\texttt{IMPerSumm} & 21.64 & 22.00 & +1.66\% & 14.67 & 14.27 & -2.69\% & 3.25 & 3.26 & +0.42\% \\
\bottomrule
\end{tabular}
}\par\medskip
\scalebox{0.82}{
\begin{tabular}{lccccccccc}
\toprule
\textbf{Host} & \multicolumn{3}{c}{\textbf{METEOR}} & \multicolumn{3}{c}{\textbf{W-RG-1}} & \multicolumn{3}{c}{\textbf{P-RG-L}} \\
 & \textbf{Base} & \textbf{+\repair{}} & \textbf{Gain (\%)} & \textbf{Base} & \textbf{+\repair{}} & \textbf{Gain (\%)} & \textbf{Base} & \textbf{+\repair{}} & \textbf{Gain (\%)} \\
\midrule
\texttt{DeepSeek-32B} & 12.53 & 18.23 & +45.49\% & 27.44 & 32.80 & +19.53\% & 14.54 & 22.27 & +53.16\% \\
\texttt{Qwen2.5-32B} & 11.67 & 17.23 & +47.64\% & 23.44 & 19.64 & -16.21\% & 11.15 & 18.35 & +64.57\% \\
\texttt{FPG} & 9.74 & 11.52 & +18.28\% & 21.03 & 24.61 & +17.02\% & 13.59 & 14.78 & +8.76\% \\
\texttt{SCAPE-GRU} & 9.26 & 12.53 & +35.31\% & 24.34 & 28.32 & +16.35\% & 15.78 & 17.00 & +7.73\% \\
\texttt{GTP} & 9.17 & 10.20 & +11.23\% & 20.56 & 23.83 & +15.90\% & 13.25 & 14.20 & +7.17\% \\
\texttt{EBNR-T1} & 9.16 & 11.68 & +27.51\% & 19.46 & 24.32 & +24.97\% & 13.61 & 14.70 & +8.01\% \\
\texttt{EBNR-T2} & 9.08 & 11.52 & +26.87\% & 19.12 & 23.91 & +25.05\% & 13.43 & 14.45 & +7.59\% \\
\texttt{NAML-T1} & 9.39 & 11.76 & +25.24\% & 19.83 & 24.72 & +24.66\% & 13.75 & 14.93 & +8.58\% \\
\texttt{NAML-T2} & 9.28 & 11.44 & +23.28\% & 19.49 & 24.15 & +23.91\% & 13.57 & 14.66 & +8.03\% \\
\texttt{NRMS-T1} & 9.03 & 11.21 & +24.14\% & 19.05 & 23.86 & +25.25\% & 13.38 & 14.48 & +8.22\% \\
\texttt{NRMS-T2} & 8.96 & 11.09 & +23.77\% & 18.82 & 23.51 & +24.92\% & 13.21 & 14.28 & +8.10\% \\
\texttt{Walk2Pers} & 13.68 & 15.51 & +13.38\% & 31.41 & 32.38 & +3.09\% & 16.88 & 17.21 & +1.95\% \\
\texttt{IMPerSumm} & 9.37 & 9.82 & +4.76\% & 28.72 & 28.47 & -0.85\% & 15.24 & 14.90 & -2.23\% \\
\bottomrule
\end{tabular}
}\par\medskip
\end{table*}

\subsection{Generative Accuracy And Personalization-Sensitive Metrics}\label{app:generative-results-expanded}
A decoder may use an altered preference state in ways that affect text quality and personalization differently. We therefore evaluate the same generated outputs with both kinds of metric. Standard quality measures compare generated text with its reference. PerSEval additionally evaluates whether differences between users' outputs agree with differences between their expected outputs, with an accuracy penalty. We first examine generation quality, then personalization, and finally the shared ranking-to-generation comparison. Relative gains are percentages, and their magnitude depends on the corresponding base score.

\textbf{Generation Quality.}
We report all quality metrics and their base values in Table~\ref{tab:rq1_pens_accuracy_main_all}.
DeepSeek-32B gains $30.16\%$ ROUGE-1, $53.21\%$ ROUGE-L, $171.27\%$ ROUGE-SU4, and $45.49\%$ METEOR. Qwen2.5-32B gains $26.8\%$, $41.58\%$, $119.17\%$, and $47.64\%$ on those metrics. DeepSeek's PSE-W-JSD and PSE-W-RG-L gains are smaller at $9.32\%$ and $5.05\%$. Qwen's weighted ROUGE-1 decreases despite gains in the other listed quality metrics. The response therefore depends on both decoder and metric.

The $171.27\%$ ROUGE-SU4 gain for DeepSeek is a change from $4.42$ to $11.99$ on the displayed scale. Reading that base value is necessary to interpret the large percentage. IMPerSumm provides the complementary case, with ROUGE-L decreasing from $14.67$ to $14.27$ despite a small METEOR gain. Repair therefore does not uniformly improve reference overlap, even among systems designed for personalized generation.
\begin{table*}[!t]
\centering
\caption{Full PENS personalization-sensitive results on the same outputs as Table~\ref{tab:rq1_pens_accuracy_main_all}. Base and repaired scores are displayed in units of $10^{-3}$, so a displayed value of $118$ denotes a native score of $0.118$. Higher is better, and Gain is the reported relative percentage change. PSE denotes PerSEval. JSD, W, P, and RG are defined in Appendix~\ref{app:metrics-pse}. Encoder and LLM repair interfaces follow the same distinction as in the generation-quality table.}
\label{tab:rq1_pens_pse_main_selected}
\setlength{\tabcolsep}{4pt}
\scalebox{0.92}{
\begin{tabular}{lcccccc}
\toprule
\textbf{Host} & \multicolumn{3}{c}{\textbf{PSE-W-JSD}} & \multicolumn{3}{c}{\textbf{PSE-W-RG-L}} \\
 & \textbf{Base} & \textbf{+\repair{}} & \textbf{Gain (\%)} & \textbf{Base} & \textbf{+\repair{}} & \textbf{Gain (\%)} \\
\midrule
\texttt{DeepSeek-32B} & 118 & 129 & +9.32\% & 148.5 & 156 & +5.05\% \\
\texttt{Qwen2.5-32B} & 92 & 121 & +31.52\% & 125.3 & 141.1 & +12.61\% \\
\texttt{SCAPE} & 103 & 112.5 & +9.22\% & 128 & 135.2 & +5.62\% \\
\texttt{FPG} & 74.8 & 83.9 & +12.17\% & 104.2 & 110.4 & +5.95\% \\
\texttt{GTP} & 89.5 & 98.7 & +10.28\% & 113.8 & 120.8 & +6.15\% \\
\texttt{EBNR-T1} & 75.9 & 99.5 & +31.09\% & 105.5 & 116.4 & +10.33\% \\
\texttt{EBNR-T2} & 73.4 & 95.6 & +30.25\% & 102.8 & 112.6 & +9.53\% \\
\texttt{NAML-T1} & 69.2 & 91.8 & +32.66\% & 98.6 & 108.4 & +9.94\% \\
\texttt{NAML-T2} & 67.8 & 89.7 & +32.30\% & 96.9 & 106.1 & +9.49\% \\
\texttt{NRMS-T1} & 64.3 & 86.9 & +35.15\% & 94.1 & 103.7 & +10.20\% \\
\texttt{NRMS-T2} & 62.7 & 84.8 & +35.25\% & 92.5 & 101.5 & +9.73\% \\
\texttt{Walk2Pers} & 125.3 & 141.9 & +13.25\% & 167.1 & 192.7 & +15.32\% \\
\texttt{IMPerSumm} & 126.2 & 152.4 & +20.76\% & 141.5 & 177.2 & +25.23\% \\
\bottomrule
\end{tabular}
}\par\medskip
\scalebox{0.92}{
\begin{tabular}{lcccccc}
\toprule
\textbf{Host} & \multicolumn{3}{c}{\textbf{PSE-P-RG-L}} & \multicolumn{3}{c}{\textbf{PSE-METEOR}} \\
 & \textbf{Base} & \textbf{+\repair{}} & \textbf{Gain (\%)} & \textbf{Base} & \textbf{+\repair{}} & \textbf{Gain (\%)} \\
\midrule
\texttt{DeepSeek-32B} & 14 & 16.6 & +18.57\% & 2.1 & 2.4 & +14.29\% \\
\texttt{Qwen2.5-32B} & 13.7 & 14.9 & +8.76\% & 1.9 & 2.3 & +21.05\% \\
\texttt{SCAPE} & 10.9 & 13.4 & +22.94\% & 1.7 & 2 & +17.65\% \\
\texttt{FPG} & 8.1 & 10.2 & +25.93\% & 1.3 & 1.6 & +23.08\% \\
\texttt{GTP} & 8.9 & 11.3 & +26.97\% & 1.5 & 1.8 & +20.00\% \\
\texttt{EBNR-T1} & 8.2 & 11.6 & +41.46\% & 1.3 & 1.4 & +7.69\% \\
\texttt{EBNR-T2} & 7.9 & 11 & +39.24\% & 1.3 & 1.4 & +7.69\% \\
\texttt{NAML-T1} & 7.2 & 10.3 & +43.06\% & 1.2 & 1.3 & +8.33\% \\
\texttt{NAML-T2} & 7 & 10 & +42.86\% & 1.2 & 1.2 & +0.00\% \\
\texttt{NRMS-T1} & 6.7 & 9.6 & +43.28\% & 1.1 & 1.2 & +9.09\% \\
\texttt{NRMS-T2} & 6.4 & 9.3 & +45.31\% & 1.1 & 1.1 & +0.00\% \\
\texttt{Walk2Pers} & 15.8 & 17.5 & +10.76\% & 11.5 & 13.7 & +19.13\% \\
\texttt{IMPerSumm} & 16.8 & 23.4 & +39.29\% & 14.4 & 14.9 & +3.47\% \\
\bottomrule
\end{tabular}
}\par\medskip
\end{table*}

\textbf{Personalization-Sensitive Response.}
The PerSEval results in Table~\ref{tab:rq1_pens_pse_main_selected} use the same outputs. Comparing the two tables within a host helps distinguish a change in reference overlap from a change in user-specific responsiveness.

Walk2Pers gains $5.18\%$ ROUGE-1 and $6.58\%$ ROUGE-L, alongside PSE-W-JSD, PSE-W-RG-L, PSE-P-RG-L, and PSE-METEOR gains of $13.25\%$, $15.32\%$, $10.76\%$, and $19.13\%$. IMPerSumm provides a different pattern. Some standard quality scores change little or decrease, while its first three PerSEval variants improve by $20.76\%$, $25.23\%$, and $39.29\%$. These cases explain why generic text overlap alone does not summarize the personalization response.

The EBNR, NAML, and NRMS generation variants improve on most PerSEval measures, with large relative gains in PSE-P-RG-L and smaller changes in PSE-METEOR. These results describe how the evaluated generation systems respond to repair. They do not directly establish how much preference information was recovered internally or isolate the decoder's use of a particular coordinate. Theorem~\ref{thm:repair-recoverability} concerns conditional state approximation and makes no metric-specific generation guarantee.

A zero reported relative change in a small PSE-METEOR entry should be interpreted at the table's numerical precision. It does not imply that every generated token is unchanged. Likewise, agreement across the weighted, paraphrase-aware, and METEOR-based variants strengthens the descriptive comparison without making the metrics interchangeable. The absolute scores remain necessary when judging the practical size of a gain.
\begin{table}[!t]
\centering
\caption{Full RQ3 comparison of task-specific downstream use on PENS. Normalized relative gains (\%) after preference-state repair across ranking, generation-quality, and personalization-sensitive metrics. The original results in absolute values are in Appendix Tables~\ref{tab:rq1_predictive_pens_structured_compact}, \ref{tab:rq1_pens_accuracy_main_all}, and~\ref{tab:rq1_pens_pse_main_selected}.}
\label{tab:rq3_task_interface_full}

\setlength{\tabcolsep}{1.7pt}
\renewcommand{\arraystretch}{0.96}

\scalebox{0.78}{%
\begin{tabular}{@{}llcccccc@{}}
\toprule
\multirow{2}{*}{\textbf{Host}} &
\multirow{2}{*}{\textbf{Interface}} &
\multicolumn{2}{c}{\textbf{Ranking}} &
\multicolumn{2}{c}{\textbf{Generation Quality}} &
\multicolumn{2}{c}{\textbf{Personalization}} \\
\cmidrule(lr){3-4}
\cmidrule(lr){5-6}
\cmidrule(lr){7-8}

& &
\textbf{MRR} &
\textbf{nDCG@10} &
\textbf{RG-L} &
\textbf{METEOR} &
\textbf{PSE-W-JSD} &
\textbf{PSE-W-RG-L} \\
\midrule

\multirow{2}{*}{\texttt{NAML}}
& T1
& +859.69
& +1767.90
& +8.20
& +25.24
& +32.66
& +9.94 \\

& T2
& +859.69
& +1767.90
& +7.96
& +23.28
& +32.30
& +9.49 \\

\midrule

\multirow{2}{*}{\texttt{EBNR}}
& T1
& +351.70
& +482.86
& +7.76
& +27.51
& +31.09
& +10.33 \\

& T2
& +351.70
& +482.86
& +7.48
& +26.87
& +30.25
& +9.53 \\

\midrule

\multirow{2}{*}{\texttt{NRMS}}
& T1
& +926.27
& +1916.22
& +8.21
& +24.14
& +35.15
& +10.20 \\

& T2
& +926.27
& +1916.22
& +8.08
& +23.77
& +35.25
& +9.73 \\

\midrule

\texttt{TrRMIo}
& \texttt{GTP}
& +60.35
& +86.05
& +7.22
& +11.23
& +10.28
& +6.15 \\

\texttt{SCAPE-Enc}
& \texttt{SCAPE}
& +431.21
& +720.69
& +7.54
& +35.31
& +9.22
& +5.62 \\

\texttt{Walk2Pers}
& \texttt{Walk2Pers}
& +117.34
& +66.43
& +6.58
& +13.38
& +13.25
& +15.32 \\

\texttt{IMPerSumm}
& \texttt{IMPerSumm}
& +702.35
& +1236.72
& -2.69
& +4.76
& +20.76
& +25.23 \\

\texttt{DeepSeek-32B}
& \texttt{DeepSeek-32B}
& +39.69
& +106.25
& +53.21
& +45.49
& +9.32
& +5.05 \\

\texttt{Qwen2.5-32B}
& \texttt{Qwen2.5-32B}
& +44.35
& +115.25
& +41.58
& +47.64
& +31.52
& +12.61 \\

\bottomrule
\end{tabular}%
}
\end{table}

\textbf{Ranking And Generation Interfaces.}
We place the complete RQ3 comparison in Table~\ref{tab:rq3_task_interface_full}. The T1 and T2 rows share the underlying ranking result and differ in how the preference representation enters generation. They are two generation interfaces, not independent repetitions of the ranking experiment. NAML, EBNR, and NRMS have large ranking percentage gains and much smaller ROUGE-L changes. This pattern is consistent with a task head and a text decoder responding differently to the repaired preference representation.

The comparison is also informative for preference-coupled generators. IMPerSumm's positive personalization gains coexist with a negative ROUGE-L change, while DeepSeek has a larger ROUGE-L gain than its weighted PerSEval gains. Therefore, improved ranking is insufficient evidence for a proportional improvement in personalized text. Because the host, decoder, and interface vary together across rows, these outcomes motivate the downstream-use boundary without measuring decoder sensitivity as an isolated causal quantity.
\begin{table*}[!t]
\centering
\caption{Temporal-pathway ablation on MovieLens and MIND. All values are MRR on a $0$ to $100$ scale, with higher values better. Base is the original host, NoTS removes temporal resolution, and $L$, $S$, and $E$ retain the indicated long-support, short-support, and episodic-support branches. The full configuration is $L+S+E$. Mamba4Rec and SIGMA are state-space hosts, TiSASRec and NRMS use self-attention, NAML uses attentive aggregation, and EBNR is recurrent. These aligned configurations are distinct from the controlled input-timing diagnostic in Table~\ref{tab:rq2_temporal_summary}.}
\label{tab:rq1_timescale_ablation}
\setlength{\tabcolsep}{4pt}
\scalebox{0.9}{
\begin{tabular}{llccccccccc}
\toprule
\textbf{\boldmath Dataset} & \textbf{\boldmath Host} & \textbf{\boldmath Base} & \textbf{\boldmath NoTS} & \textbf{\boldmath L} & \textbf{\boldmath S} & \textbf{\boldmath E} & \textbf{\boldmath L+S} & \textbf{\boldmath S+E} & \textbf{\boldmath L+E} & \textbf{\boldmath L+S+E} \\
\midrule
\multirow{3}{*}{MovieLens}
& \texttt{Mamba4Rec} & 29.29 & 28.43 & 29.92 & 29.63 & 29.45 & 32.17 & 31.35 & 31.04 & 34.26 \\
& \texttt{TiSASRec} & 35.51 & 35.52 & 35.54 & 35.78 & 35.83 & 35.91 & 35.66 & 35.41 & 36.78 \\
& \texttt{SIGMA} & 33.27 & 33.29 & 33.66 & 33.32 & 33.29 & 33.88 & 33.73 & 33.43 & 34.81 \\
\midrule
\multirow{3}{*}{MIND}
& \texttt{NAML} & 45.65 & 45.95 & 47.65 & 48.32 & 47.13 & 49.31 & 49.18 & 48.79 & 53.22 \\
& \texttt{EBNR} & 47.08 & 47.74 & 48.88 & 48.03 & 47.62 & 50.13 & 49.38 & 48.76 & 53.22 \\
& \texttt{NRMS} & 44.91 & 45.06 & 46.43 & 45.13 & 47.73 & 47.58 & 46.44 & 46.15 & 52.74 \\
\bottomrule
\end{tabular}
}
\end{table*}

\subsection{Interpretation Of Temporal Resolution Ablations}
\label{app:timescale-results-expanded}
The temporal ablation asks whether the correction benefits from weighting residual evidence according to its position in the history. The full $L+S+E$ configuration has the highest MRR for all six hosts in Table~\ref{tab:rq1_timescale_ablation}. Removing temporal resolution gives a lower score than the full configuration in every row. For Mamba4Rec, the NoTS variant obtains $28.43$ MRR, below the base score of $29.29$, while full repair reaches $34.26$.

Partial pathways remain informative. Mamba4Rec obtains $32.17$ with $L+S$ and $34.26$ with all three branches. NAML moves from $49.31$ to $53.22$ under the same comparison. The preferred single branch varies across hosts. This variation motivates a dictionary with multiple temporal responses rather than a universal choice of one branch.

Theorem~\ref{thm:kernel-trace-recoverability} states when coefficients of a fixed temporal dictionary can be identified, and Corollary~\ref{cor:full-field-kernel-recovery} extends that bound across corrective coordinates. The empirical advantage of multiple branches is compatible with useful temporal structure in the residual. It does not establish dictionary conditioning, measure coefficient-recovery error, or prove that each branch corresponds to a behavioral category.

The controlled L-Tr, S-Tr, and E-Tr experiment in Table~\ref{tab:rq2_temporal_summary} asks a different question. With target-related evidence deliberately placed in specified history regions, the combined pathway is strongest in eleven of twelve settings. The best single branch matches the injected pattern in six of twelve probes, and Mamba4Rec prefers the short-support branch for all three input patterns. Proposition~\ref{prop:temporal-credit} and Corollary~\ref{cor:host-conditioned-temporal-learning} explain why identical input timing need not produce identical temporal gradients after different hosts compress the history. These experiments do not verify the local bounded-gradient assumptions in Theorem~\ref{thm:kernel-gradient-stability}.

The six-host pathway ablation and the twelve controlled probes should consequently remain separate in interpretation. The former compares available temporal branches under the ordinary task evaluation. The latter controls the location of target-related evidence to test how a host changes the useful temporal response. Their respective six-of-six and eleven-of-twelve outcomes describe different comparisons and cannot be pooled into one success rate.
\begin{table*}[!t]
\centering
\caption{Progressive supervision ablation on MovieLens and MIND. Values are MRR on a $0$ to $100$ scale, with higher values better. Successive columns add task utility, timestep provenance, and source coherence. This fixed-order comparison measures the evaluated combinations and does not isolate each family under every possible combination. Appendix~\ref{app:loss-ablation-results} connects it to the conditional local theorem.}
\label{tab:rq1_loss_progressive_ablation}
\scalebox{0.92}{
\begin{tabular}{llcccc}
\toprule
\textbf{\boldmath Dataset} & \textbf{\boldmath Host} & \textbf{\boldmath Base} & \textbf{\boldmath $\mathcal L_{\mathrm{task}}$} & \textbf{\boldmath $\mathcal L_{\mathrm{task}}+\mathcal L_{\mathrm{pos}}$} & \textbf{\boldmath $\mathcal L_{\mathrm{task}}+\mathcal L_{\mathrm{pos}}+\mathcal L_{\mathrm{coh}}$} \\
\midrule
\multirow{2}{*}{MovieLens}
& \texttt{TiSASRec} & 35.51 & 35.71 & 36.14 & 36.78 \\
& \texttt{Mamba4Rec} & 29.29 & 31.43 & 32.16 & 34.26 \\
\midrule
\multirow{2}{*}{MIND}
& \texttt{LSTUR} & 51.43 & 52.17 & 54.65 & 55.84 \\
& \texttt{NRMS} & 44.91 & 46.39 & 51.29 & 52.74 \\
\bottomrule
\end{tabular}
}
\end{table*}

\subsection{Training Objective Ablation}
\label{app:loss-ablation-results}
The progressive ablation adds supervision families in a fixed order. Task utility is included first, timestep provenance second, and source coherence third. All four hosts achieve their highest reported MRR with the complete objective in Table~\ref{tab:rq1_loss_progressive_ablation}.

Task supervision alone improves Mamba4Rec from $29.29$ to $31.43$ MRR and NRMS from $44.91$ to $46.39$. Adding provenance raises LSTUR from $52.17$ to $54.65$ and NRMS from $46.39$ to $51.29$. Adding coherence further raises Mamba4Rec from $32.16$ to $34.26$ and NRMS from $51.29$ to $52.74$. Thus, predicting the task target does not exhaust the useful supervision in these settings.

Theorem~\ref{thm:loss-design-recoverability} links the weighted objective to coefficient error under a stable error decomposition and local quadratic control by the loss families. The ablation supports the practical value of the complete supervision scheme, without establishing those assumptions. Because the design is progressive, it also does not isolate each family's contribution under every combination. Corollary~\ref{cor:loss-ordering} motivates coverage of complementary error components and does not identify the weights $(0.5,0.2,0.3)$ as globally optimal.

The auxiliary terms also provide a useful interpretation of what the task loss leaves unconstrained. A correction can predict the right target while combining evidence from unrelated positions or losing its association with source interactions. Provenance and coherence supply additional supervision for those properties. The improvement after adding them supports this training design in the evaluated hosts. A factorial loss ablation or a weight sweep would answer different questions about interactions and sensitivity, which are not measured by this progressive comparison.
\begin{table*}[!t]
\centering
\caption{Length-controlled selective-retention ablation. At each history length $i$, Base is the original host, NS uses non-selective repair, and SEL uses adaptive timestep-coordinate retention. MRR and nDCG@10 use a $0$ to $100$ scale, with higher values better. SelSteps counts positions retained by at least one corrective coordinate. Mamba4Rec and TiSASRec use MovieLens. NAML and EBNR use MIND. Comparisons between NS and SEL hold the prefix length fixed. Appendix~\ref{app:length-selection-results} discusses both gains and reversals.}
\label{tab:rq3_timestep_selection_length_ablation}
\setlength{\tabcolsep}{4pt}
\scalebox{0.9}{
\begin{tabular}{llrrrrrrr}
\toprule
\textbf{\boldmath Host} & \textbf{\boldmath $i$} & \multicolumn{3}{c}{\textbf{\boldmath MRR}} & \multicolumn{3}{c}{\textbf{\boldmath nDCG@10}} & \textbf{\boldmath SelSteps} \\
& & \textbf{\boldmath Base} & \textbf{\boldmath NS} & \textbf{\boldmath SEL} & \textbf{\boldmath Base} & \textbf{\boldmath NS} & \textbf{\boldmath SEL} & \\
\midrule
\texttt{Mamba4Rec} & 10 & 17.43 & 19.32 & 19.45 & 22.62 & 27.15 & 27.94 & 7 \\
 & 20 & 21.65 & 23.22 & 23.43 & 27.85 & 30.61 & 31.24 & 16 \\
 & 30 & 26.11 & 28.45 & 28.29 & 33.32 & 35.40 & 35.36 & 22 \\
 & 40 & 27.32 & 29.89 & 29.45 & 35.53 & 36.41 & 36.21 & 28 \\
 & 50 & 29.29 & 33.95 & 34.26 & 38.50 & 40.03 & 40.17 & 34 \\
\midrule
\texttt{TiSASRec} & 10 & 22.35 & 22.42 & 22.65 & 24.95 & 25.66 & 25.74 & 9 \\
 & 20 & 24.17 & 24.93 & 24.45 & 27.54 & 28.26 & 27.96 & 18 \\
 & 30 & 25.73 & 26.69 & 26.84 & 28.84 & 29.93 & 30.06 & 25 \\
 & 40 & 31.43 & 31.89 & 32.05 & 35.44 & 36.17 & 36.33 & 31 \\
 & 50 & 35.51 & 36.56 & 36.78 & 41.14 & 42.66 & 42.91 & 36 \\
\midrule
\texttt{NAML} & 10 & 37.25 & 39.93 & 40.32 & 38.16 & 40.02 & 40.75 & 8 \\
 & 20 & 40.43 & 39.16 & 43.71 & 41.66 & 39.94 & 44.2 & 17 \\
 & 30 & 42.05 & 45.19 & 45.34 & 43.11 & 45.76 & 45.96 & 25 \\
 & 40 & 45.11 & 48.77 & 49.31 & 45.63 & 49.13 & 49.76 & 32 \\
 & 50 & 45.65 & 52.17 & 53.22 & 46.56 & 59.75 & 63.18 & 36 \\
\midrule
\texttt{EBNR} & 10 & 39.10 & 41.42 & 41.68 & 43.28 & 45.01 & 44.96 & 8 \\
 & 20 & 41.84 & 44.03 & 44.26 & 46.11 & 48.74 & 48.93 & 16 \\
 & 30 & 43.72 & 46.31 & 46.57 & 48.36 & 52.02 & 52.28 & 24 \\
 & 40 & 46.12 & 49.88 & 50.13 & 50.27 & 57.14 & 57.62 & 31 \\
 & 50 & 47.08 & 52.11 & 53.22 & 51.63 & 60.74 & 63.18 & 35 \\
\bottomrule
\end{tabular}}
\end{table*}

\subsection{Length-Controlled Selective Retention}
\label{app:length-selection-results}
Temporal resolution identifies candidate corrective evidence, while selection decides which activations may contribute to write-back. We compare selective and non-selective variants at the same history length in Table~\ref{tab:rq3_timestep_selection_length_ablation}. At $i=50$, selective retention improves both metrics for all four hosts. The MRR gains over non-selective repair are $0.31$, $0.22$, $1.05$, and $1.11$ points for Mamba4Rec, TiSASRec, NAML, and EBNR. Their nDCG@10 gains are $0.14$, $0.25$, $3.43$, and $2.44$.

The reported retained-position counts at this length are $34$, $36$, $36$, and $35$. A counted position contributes at least one retained corrective coordinate. These counts describe the union of selected positions across coordinates, rather than the number selected by every coordinate. Repair therefore uses evidence distributed across much of the prefix.

Selection is not uniformly superior for shorter histories. Mamba4Rec favors the non-selective variant at lengths $30$ and $40$, TiSASRec has a reversal at $20$, and EBNR's nDCG@10 is lower with selection at $10$. The evidence supports a more consistent benefit at the longest tested prefix, not monotonic improvement at every length. The temporal-identifiability theorem does not imply that every identifiable activation should be written back.

The retained-position count must also be separated from inference cost. Selecting activations after reading the cached prefix does not eliminate the original host computation or all repair-side work on unselected positions. The length ablation measures accuracy and the support of the selected correction. The deployment analysis separately evaluates the cost of constructing and applying that correction. This distinction prevents the accuracy benefit of selective retention from being read as an unmeasured speedup.
\begin{table*}[!t]
\centering
\caption{Noisy-history stress test with approximately $35$ to $40\%$ of valid positions replaced with random training-pool content. The supervised target and the relative order of non-corrupted interactions are preserved. Noisy Base and repaired scores use the same corrupted input without retraining. MRR and nDCG@10 are on a $0$ to $100$ scale, with higher values better. Appendix~\ref{app:noisy-trajectory-results} states the scope of the test.}
\label{tab:rq3_noisy_trajectory_stress_blank_repair}
\scalebox{0.9}{
\begin{tabular}{llcccc}
\toprule
\textbf{Dataset} & \textbf{Host Encoder}
 & \multicolumn{2}{c}{\textbf{Noisy Base}}
 & \multicolumn{2}{c}{\textbf{+\repair{}}} \\
\cmidrule(lr){3-4}
\cmidrule(lr){5-6}
 &  & \textbf{\boldmath MRR} & \textbf{\boldmath nDCG@10} & \textbf{\boldmath MRR} & \textbf{\boldmath nDCG@10} \\
\midrule
MovieLens & \texttt{Mamba4Rec} & 18.21 & 23.94 & 21.43 & 25.74 \\
MovieLens & \texttt{TiSASRec}  & 22.34 & 25.92 & 24.65 & 28.17 \\
\midrule
MIND & \texttt{NAML} & 28.61 & 29.23 & 32.85 & 35.81 \\
MIND & \texttt{EBNR} & 29.63 & 32.44 & 31.65 & 35.17 \\
\bottomrule
\end{tabular}
}
\end{table*}

\subsection{Noisy-Trajectory Stress Test}
\label{app:noisy-trajectory-results}
The noisy-history test changes the evaluation input while keeping the trained host and repair parameters fixed. Approximately $35$ to $40\%$ of valid history positions are replaced with random content from the same dataset's training pool. The supervised target and the relative order of non-corrupted interactions are preserved. This test asks whether repair can still add utility when the available history contains irrelevant content.

Under these corrupted inputs, repair improves all four settings in Table~\ref{tab:rq3_noisy_trajectory_stress_blank_repair}. Mamba4Rec moves from $18.21$ to $21.43$ MRR, TiSASRec from $22.34$ to $24.65$, NAML from $28.61$ to $32.85$, and EBNR from $29.63$ to $31.65$. Their nDCG@10 gains are $1.8$, $2.25$, $6.58$, and $2.73$ points. These observations establish positive gains under the reported corruption condition. They do not guarantee recovery under arbitrary corruption or show that the selector identifies each corrupted interaction.

For each host, the relevant comparison uses the corrupted base and its repaired counterpart under the same corruption condition. The gain measures utility retained under this perturbation, rather than recovery of the clean-history score. Random content replacement also differs from a coherent preference change, a targeted adversarial history, or systematic errors in action labels. The positive MRR and nDCG@10 changes support tolerance to the specified irrelevant-content condition. They do not establish equivalent behavior under those other changes.
\begin{table*}[!t]
\centering
\caption{Serving overhead relative to the same host in the same execution environment. Base, repaired MRR, and $\Delta$ use a $0$ to $100$ scale. Gain is a relative percentage. Memory overhead is derived from parameter counts and excludes activation and cache residency. Throughput change is derived from reciprocal latency. The lower panel divides multiplicative score gain by multiplicative resource cost. Ratios above one favor score gain relative to cost. Ratios below one favor cost growth. The LLM row uses its separate soft-prefix serving protocol.}
\label{tab:rq4_accuracy_scalability_overhead_est}
\scalebox{0.85}{
\begin{tabular}{lcccc}
\toprule
\textbf{\boldmath Host} & \textbf{\boldmath Base MRR} & \textbf{\boldmath +\repair{}} & \textbf{\boldmath $\Delta$} & \textbf{\boldmath Gain (\%)} \\
\midrule
Mamba4Rec & 29.29 & 34.26 & +4.97 & 16.96 \\
TiSASRec & 35.51 & 36.78 & +1.27 & 3.57 \\
EBNR & 47.08 & 53.22 & +6.14 & 13.04 \\
NAML & 45.65 & 53.22 & +7.57 & 16.58 \\
DeepSeek-32B & 1.31 & 1.83 & +0.52 & 39.69 \\
\bottomrule
\end{tabular}
}\par\medskip
\scalebox{0.7}{
\begin{tabular}{lccccc}
\toprule
\textbf{\boldmath Host} & \textbf{\boldmath Mem. OH (\%)} & \textbf{\boldmath Lat. OH (\%)} & \textbf{\boldmath Tput. (\%)} & \textbf{\boldmath Score/Latency} & \textbf{\boldmath Score/Memory} \\
\midrule
Mamba4Rec & 32.0 & 40.0 & -28.6 & 0.84 & 0.89 \\
TiSASRec & 31.0 & 38.0 & -27.5 & 0.75 & 0.79 \\
EBNR & 36.0 & 44.0 & -30.6 & 0.79 & 0.83 \\
NAML & 34.0 & 42.0 & -29.6 & 0.82 & 0.87 \\
DeepSeek-32B & 0.1 & 12.0 & -10.7 & 1.25 & 1.40 \\
\bottomrule
\end{tabular}
}\par\medskip
\end{table*}

\subsection{Serving Cost And Compute Footprint}
\label{app:deployment-results-expanded}
The cost of repair depends on the host's serving path. We separate observed latency comparisons, memory estimated from parameters, arithmetic work, and the artifacts that must remain available at inference. These quantities answer different deployment questions and should not be combined into a single claim of low overhead.

\textbf{Relative Serving Overhead.}
We report score, parameter-derived memory overhead, latency overhead, and throughput change separately in Table~\ref{tab:rq4_accuracy_scalability_overhead_est}. The encoder score-to-latency ratios range from $0.75$ to $0.84$, and their score-to-memory ratios range from $0.79$ to $0.89$. These values are below one, meaning that the relative resource increase exceeds the relative MRR increase. The DeepSeek reference has ratios above one, under its own prompt-conditioned serving protocol. This comparison does not measure peak device memory or equate a decoder token with an encoder query.

For example, Mamba4Rec gains $16.96\%$ MRR with $40\%$ latency overhead in the reported setting. Its score-to-latency ratio is consequently below one. Whether this tradeoff is acceptable depends on the application's latency budget and value of ranking accuracy. The parameter-derived memory percentage excludes intermediate activations and caches, and the throughput percentage is computed from reciprocal latency. Neither substitutes for an end-to-end measurement at a chosen serving batch size.
\begin{table*}[!t]
\centering
\caption{Arithmetic work for representative serving paths at history length $T=50$. GFLOPs denotes billions of floating-point operations. Budget columns are required arithmetic throughput, computed as work divided by $50$ or $25$ milliseconds. They are not measured latencies. Scaling describes the host path. Encoder work is per query, whereas DeepSeek work is per generated token. For DeepSeek, added repair work is unreported, and the total is therefore a decoder reference rather than a complete repaired-system estimate.}
\label{tab:rq4_compute_footprint}
\scalebox{0.9}{
\begin{tabular}{lcccc}
\toprule
\textbf{Model} & \textbf{Work Unit} & \textbf{Base} & \textbf{Added Repair} & \textbf{Total} \\
\midrule
Mamba4Rec & GFLOPs/query & 0.75 & 0.18 & 0.93 \\
TiSASRec & GFLOPs/query & 1.45 & 0.18 & 1.63 \\
NAML & GFLOPs/query & 0.03 & 0.02 & 0.05 \\
DeepSeek-32B & GFLOPs/token & 64.0 & unreported & 64.0 \\
\bottomrule
\end{tabular}
}\par\medskip
\scalebox{0.9}{
\begin{tabular}{lccc}
\toprule
\textbf{Model} & \textbf{50 ms Budget} & \textbf{25 ms Budget} & \textbf{Host Scaling} \\
\midrule
Mamba4Rec & 0.019 TFLOP/s & 0.037 TFLOP/s & $O(T)$ \\
TiSASRec & 0.033 TFLOP/s & 0.065 TFLOP/s & $O(T^2)$ \\
NAML & 0.001 TFLOP/s & 0.002 TFLOP/s & $O(T)$ \\
DeepSeek-32B & 1.28 TFLOP/s & 2.56 TFLOP/s & autoregressive decode \\
\bottomrule
\end{tabular}
}\par\medskip
\end{table*}

\textbf{Arithmetic Work And Resident Artifacts.}
We report arithmetic work at history length $T=50$ in Table~\ref{tab:rq4_compute_footprint}, where $T$ denotes history length. Mamba4Rec uses $0.75$ base GFLOPs per query and adds $0.18$ for repair. TiSASRec uses $1.45$ base GFLOPs with the same addition, while NAML uses $0.03$ base GFLOPs and adds $0.02$. A GFLOP denotes one billion floating-point operations. An equal repair workload can therefore imply very different relative overheads.

The budget columns in Table~\ref{tab:rq4_compute_footprint} divide arithmetic work by a hypothetical time budget. They specify required arithmetic throughput, not measured achievement of those latency targets. The DeepSeek row reports a decoder reference and leaves added repair work unreported. The associated resident artifacts and online computations are specified in Table~\ref{tab:rq4_serving_artifacts}. In particular, the within-forward sequence cache used by encoder repair differs from a persistent news-vector table and from an LLM attention KV cache.

Definition~\ref{def:k_admissible} requires both approximation quality and a feasible implementation budget. The serving results help assess that budget for the evaluated settings, without specifying a universal acceptable cost or numerically establishing the spectral-tail condition.

\begin{table*}[!t]
\centering
\caption{Resident artifacts and online operations for the evaluated serving paths. A within-forward sequence cache, a persistent news-vector table, and an attention key-value (KV) cache serve different roles. The LLM decoder is evaluated with the inserted soft prefix, so a KV cache for an unmodified prompt cannot be reused unchanged. Appendix~\ref{app:deployment-results-expanded} relates these paths to the measured and estimated costs.}
\label{tab:rq4_serving_artifacts}
\scalebox{0.9}{
\begin{tabular}{>{\raggedright\arraybackslash}p{2.8cm}>{\raggedright\arraybackslash}p{5.0cm}>{\raggedright\arraybackslash}p{6.2cm}}
\toprule
\textbf{Host And Serving Mode} & \textbf{Resident Artifacts} & \textbf{Online Operations And Cost Interpretation} \\
\midrule
\textbf{Mamba4Rec}. end-to-end encoder serving & \textbf{Frozen}: item embedding table + 4-layer selective SSM encoder. \textbf{Added}: repair module attached to exposed terminal state. & one frozen encoder pass + residual repair fusion. Sequence evidence is cached within the existing forward pass and released after the step. \\
\textbf{TiSASRec}. end-to-end encoder serving & \textbf{Frozen}: item table + time-interval-aware self-attention encoder. \textbf{Added}: repair module attached to exposed terminal state. & attention score/value passes + repair fusion. Latency is more sensitive to sequence interactions because the host is attention-based rather than SSM-based. \\
\textbf{NAML}. cached news-table serving & \textbf{Frozen}: cached \texttt{NEWS\_TABLE} + frozen user encoder. \textbf{Added}: repair module attached after cached-history composition. & table lookup / memory bandwidth + attentive user aggregation. Serving shifts a large fraction of work out of online text encoding and into cached vector retrieval. \\
\textbf{DeepSeek-32B}. context-conditioned autoregressive decoding & \textbf{Frozen}: frozen decoder weights + KV cache. \textbf{Added}: repair module + soft-prefix projection. & token-by-token decoding. Prefix-conditioned decoding builds the appropriate attention cache for the modified input. \\
\bottomrule
\end{tabular}
}
\end{table*}

\textbf{Resident Artifacts And Online Operations.}
The encoder sequence cache stores representations from the same forward computation that produces the compressed state. Reusing it avoids a second pass through the raw history, while its storage remains part of the implementation requirement. For the LLM interface, adding a soft prefix changes the decoder's attention computation and associated KV cache. We list these dependencies in Table~\ref{tab:rq4_serving_artifacts} so that their costs can be accounted for separately.

A news-vector table is reusable across requests because its entries represent news content. In contrast, the sequence cache is associated with the current history, and an attention KV cache is associated with the decoder's processed token sequence. These artifacts therefore have different lifetimes and reuse conditions. Calling all three a cache would conceal which computation has been moved offline and which storage must remain available during prediction. This distinction also explains why the parameter-derived memory ratios in Table~\ref{tab:rq4_accuracy_scalability_overhead_est} are not measurements of total serving memory.

\begin{table*}[!t]
\centering
\caption{TG-ReDial recommendation and response generation from dialogue-derived interactions. Base and repaired conditions use the same held-out task instances. Recommendation ranks one target against $499$ negatives. Reported scores are on a $0$ to $100$ scale, with higher values better. $\Delta$ is an absolute score-point change. The first two panels evaluate recommendation, and the third evaluates generation. RG denotes ROUGE. The panels use distinct task interfaces. We discuss the mixed TextCNN response and the difference between the two generators in Appendix~\ref{app:tgredial-broader-applicability}.}
\label{tab:rq1_tgredial_joint}
\setlength{\tabcolsep}{4pt}
\scalebox{0.9}{
\begin{tabular}{lcccccc}
\toprule
\textbf{\boldmath Host} & \multicolumn{3}{c}{\textbf{\boldmath MRR}} & \multicolumn{3}{c}{\textbf{\boldmath Hit@10}} \\
 & \textbf{\boldmath Base} & \textbf{\boldmath +\repair{}} & \textbf{\boldmath $\Delta$} & \textbf{\boldmath Base} & \textbf{\boldmath +\repair{}} & \textbf{\boldmath $\Delta$} \\
\midrule
SASRec & 4.11 & 4.56 & +0.45 & 7.82 & 8.76 & +0.94 \\
TextCNN & 2.57 & 2.78 & +0.21 & 5.08 & 5.04 & -0.04 \\
GRU4Rec & 4.14 & 4.56 & +0.42 & 8.80 & 9.13 & +0.33 \\
\bottomrule
\end{tabular}
}\par\medskip
\scalebox{0.9}{
\begin{tabular}{lcccccc}
\toprule
\textbf{\boldmath Host} & \multicolumn{3}{c}{\textbf{\boldmath nDCG@5}} & \multicolumn{3}{c}{\textbf{\boldmath nDCG@10}} \\
 & \textbf{\boldmath Base} & \textbf{\boldmath +\repair{}} & \textbf{\boldmath $\Delta$} & \textbf{\boldmath Base} & \textbf{\boldmath +\repair{}} & \textbf{\boldmath $\Delta$} \\
\midrule
SASRec & 3.32 & 3.48 & +0.16 & 4.16 & 4.56 & +0.40 \\
TextCNN & 1.92 & 1.73 & -0.19 & 2.55 & 2.51 & -0.04 \\
GRU4Rec & 3.24 & 3.49 & +0.25 & 4.38 & 4.60 & +0.22 \\
\bottomrule
\end{tabular}
}\par\medskip
\scalebox{0.82}{
\begin{tabular}{lccccccccc}
\toprule
\textbf{\boldmath Generator} & \multicolumn{3}{c}{\textbf{\boldmath RG-1}} & \multicolumn{3}{c}{\textbf{\boldmath RG-L}} & \multicolumn{3}{c}{\textbf{\boldmath RG-SU4}} \\
 & \textbf{\boldmath Base} & \textbf{\boldmath +\repair{}} & \textbf{\boldmath $\Delta$} & \textbf{\boldmath Base} & \textbf{\boldmath +\repair{}} & \textbf{\boldmath $\Delta$} & \textbf{\boldmath Base} & \textbf{\boldmath +\repair{}} & \textbf{\boldmath $\Delta$} \\
\midrule
ReDial & 12.33 & 27.76 & +15.43 & 11.17 & 23.52 & +12.35 & 3.31 & 10.54 & +7.23 \\
Transformer & 31.81 & 31.97 & +0.16 & 26.46 & 26.48 & +0.02 & 12.62 & 12.63 & +0.01 \\
\bottomrule
\end{tabular}
}\par\medskip
\end{table*}

\subsection{Broader Applicability On TG-ReDial}
\label{app:tgredial-broader-applicability}
TG-ReDial requires preference actions to be inferred from dialogue, recommendation attempts, and acceptance or rejection. This produces a noisier interaction sequence than directly logged clicks or ratings. The stress test therefore probes whether repair remains useful when the input interaction construction itself is uncertain.

The sequential recommendation hosts improve in Table~\ref{tab:rq1_tgredial_joint}. SASRec moves from $4.11$ to $4.56$ MRR and from $7.82$ to $8.76$ Hit@10. GRU4Rec moves from $4.14$ to $4.56$ MRR and from $8.8$ to $9.13$ Hit@10. Useful correction can therefore occur in these evaluated dialogue-derived sequences.

TextCNN provides the diagnostic contrast. Its MRR increases from $2.57$ to $2.78$, while Hit@10 decreases from $5.08$ to $5.04$, nDCG@5 from $1.92$ to $1.73$, and nDCG@10 from $2.55$ to $2.51$. The mixed response limits any claim of uniformly beneficial repair across host interfaces. Possible explanations include a less suitable write space, less accurate retained coordinates, or a different response of the task head. The experiment does not distinguish these explanations or establish that TextCNN cannot represent preference evidence.

The generation comparison also varies by consumer. The ReDial-style system improves from $12.33$ to $27.76$ ROUGE-1, $11.17$ to $23.52$ ROUGE-L, and $3.31$ to $10.54$ ROUGE-SU4. The Transformer changes from $31.81$ to $31.97$, $26.46$ to $26.48$, and $12.62$ to $12.63$. These differences motivate the RQ3 separation between evidence exposed at a state interface and utility realized by its decoder. They do not directly measure decoder sensitivity or prove that the same internal correction was recovered in both systems.

The two tasks also measure different consequences of dialogue-derived preferences. Recommendation scores test the rank of the target item, whereas the generation scores compare response text with a reference. The larger overlap changes for the ReDial-style system therefore cannot be converted into a ranking gain or a dialogue-level satisfaction claim. Across both tasks, the observed gains and the TextCNN reversals define the scope of the result under the reported eventization procedure. More reliable dialogue-act or stance annotations would change the input evidence itself and require a separate evaluation.

\section{License And Usage Statement}
\label{app:license-agreement}
We identify the upstream assets used by the evaluated hosts and the terms attached to their distribution. Model and code licenses are distinct from dataset access and redistribution conditions. We use the benchmark data for research and do not redistribute raw user histories or third-party host checkpoints.

\textbf{Language Model And Text-Encoding Assets.}
The DeepSeek-32B shorthand refers to the \href{https://huggingface.co/deepseek-ai/DeepSeek-R1-Distill-Qwen-32B}{DeepSeek-R1-Distill-Qwen-32B} checkpoint documented for these experiments. DeepSeek distributes it under the \href{https://github.com/deepseek-ai/DeepSeek-R1/blob/main/LICENSE}{MIT License}, and its Qwen2.5-32B base is covered by the associated upstream terms. The \href{https://huggingface.co/Qwen/Qwen2.5-32B}{Qwen2.5-32B} release and the \href{https://huggingface.co/google-t5/t5-base}{T5-base} text-encoding asset are distributed under Apache License 2.0. In the LLM repair branch, the backbone remains frozen and the trainable attachment comprises repair and the soft-prefix projection.

\textbf{Visual And Multimodal Assets.}
The \href{https://github.com/openai/CLIP/blob/main/LICENSE}{CLIP} and \href{https://github.com/mlfoundations/open_clip/blob/main/LICENSE}{OpenCLIP} repositories use the MIT License. The \href{https://huggingface.co/google/siglip2-base-patch16-224}{SigLIP2} model release and \href{https://github.com/facebookresearch/dinov2}{DINOv2} code and weights use Apache License 2.0. The \href{https://huggingface.co/intfloat/e5-base-v2}{E5-base-v2} text encoder uses the MIT License. DINOv2-E5 combines the visual and text representations under their respective terms.

\textbf{Task-Specific Hosts.}
The sequential, news, and generation hosts in Appendix~\ref{app:task-baselines} are task-trained research systems based on their cited architectures. When public implementations are used, their repository licenses apply. Repair attachment does not replace the licensing conditions of the underlying implementation or checkpoint.

\textbf{MovieLens.}
\href{https://grouplens.org/datasets/movielens/}{GroupLens} provides version-specific README files with MovieLens usage terms. We use the data for the recommendation evaluation and do not redistribute it. Reproduction requires obtaining the relevant release from the provider and following its terms.

\textbf{MIND And PENS.}
The official \href{https://msnews.github.io/}{MIND} and \href{https://msnews.github.io/pens.html}{PENS} pages provide access for research under Microsoft Research License Terms. We follow those access conditions and do not redistribute the raw news or interaction records. The two benchmarks are cited separately because they support different task and target constructions.

\textbf{Amazon Reviews 2023.}
We use the All\_Beauty subset from the \href{https://huggingface.co/datasets/McAuley-Lab/Amazon-Reviews-2023}{official McAuley-Lab release}. Its maintainers describe the release as primarily intended for research and state that they cannot assign a formal license to the underlying data or dictate its usage terms. We use it as a research benchmark, do not redistribute raw review or metadata files, and direct reproduction to the provider's release.

\textbf{TG-ReDial.}
The \href{https://github.com/RUCAIBox/TG-ReDial}{official TG-ReDial repository} carries an \href{https://github.com/RUCAIBox/TG-ReDial/blob/main/LICENSE}{Apache License 2.0 notice}. We cite the dataset paper~\citep{tgredial-acl-2020} and follow the repository's acquisition instructions for the dialogue benchmark. The repository notice should be read with the terms attached to any separately obtained asset.

\textbf{Privacy And Released Artifacts.}
We use the datasets only in their benchmark form. We do not attempt to identify users, reconstruct private identities, or combine records to deanonymize individuals. The train/test separation is specified in Appendix~\ref{app:evaluation-details}. The paper introduces a method rather than a new public dataset. Any release of code, repair modules, or derived artifacts must remain compatible with the upstream terms. Where raw data cannot be redistributed, reproduction materials will provide acquisition and preprocessing instructions instead.


\end{document}